\documentclass[10pt,journal,letterpaper]{IEEEtran}

\usepackage[T1]{fontenc}
\usepackage[utf8]{inputenc}
\usepackage{amsmath,amssymb}
\usepackage{graphicx}
\usepackage{iftex}
\ifPDFTeX\else
  \usepackage{newunicodechar}
  \newunicodechar{–}{\textendash}
  \newunicodechar{—}{\textemdash}
  \newunicodechar{’}{\textquoteright}
\fi
\usepackage{booktabs,multirow}
\usepackage{moreverb,url}
\usepackage{bm,bbm}
\usepackage{algpseudocode}
\usepackage{cite}
\usepackage[caption=false,font=footnotesize]{subfig}

\usepackage{xcolor}
\usepackage[hidelinks,bookmarksopen,bookmarksnumbered]{hyperref}

\newtheorem{theorem}{Theorem}

\newtheorem{lemma}{Lemma}
\newtheorem{corollary}{Corollary}
\newtheorem{remark}{Remark}

\newtheorem{assumption}{Assumption}

\newcommand\BibTeX{{\rmfamily B\kern-.05em \textsc{i\kern-.025em b}\kern-.08em
T\kern-.1667em\lower.7ex\hbox{E}\kern-.125emX}}
\newcommand{\overbar}[1]{\mkern 2.5mu\overline{\mkern-2.5mu#1\mkern-2.5mu}\mkern 2.5mu}

\newcommand{\citealp}{\cite}

\begin{document}

\title{FUSE-PnPL: Flexible, Unified, Statistically Consistent, and Asymptotically Efficient Camera Pose Estimation from Points and Lines}

\author{Guangyang Zeng, Yuan Shen, Yiguo Peng, Guodong Shi, Biqiang Mu, and Junfeng Wu%
\thanks{Guangyang Zeng and Yiguo Peng are with School of Intelligence Science and Engineering, Harbin Institute of Technology, Shenzhen, Shenzhen, China.}%
\thanks{Yuan Shen and Junfeng Wu are with School of Data Science, Chinese University of Hong Kong, Shenzhen, Shenzhen, China.}%
\thanks{Guodong Shi is with Australian Center for Field Robotics, School of Aerospace, Mechanical and Mechatronic Engineering, The University of Sydney, Sydney, Australia.}%
\thanks{Biqiang Mu is with Key Laboratory of Systems and Control, Institute of Systems Science, Academy of Mathematics and Systems Science, Chinese Academy of Sciences, Beijing, China.}%
\thanks{Corresponding author: Biqiang Mu, Key Laboratory of Systems and Control, Institute of Systems Science, Academy of Mathematics and Systems Science, Chinese Academy of Sciences, Beijing 100190, China. (e-mail: \texttt{bqmu@amss.ac.cn}).}}

\maketitle

\begin{abstract}
Camera pose estimation from 3D-2D correspondences is fundamental to robotics. Three practical considerations motivate this work. First, point and line features provide complementary geometric information, motivating a unified estimator. Second, the number and composition of available features vary across scenarios, requiring support for minimal and arbitrary non-minimal configurations. Third, modern visual front ends extract abundant features from individual images, enabling highly accurate pose estimation. We propose FUSE-PnPL, a flexible and unified camera pose estimator that is statistically consistent and asymptotically efficient. Specifically, we represent each line using two distinct points and adopt a unified image-space noise model for point and line measurements. By employing Pl\"{u}cker coordinates for line parameterization, we explicitly characterize the propagation of image endpoint noise into line residuals and formulate a joint maximum likelihood (ML) estimation problem for combined point and line measurements. To efficiently solve this problem, FUSE-PnPL employs a two-stage estimation scheme. First, noise variance estimation and bias elimination yield a consistent initial pose estimate. A single Gauss-Newton iteration then refines this estimate, achieving asymptotic efficiency as its mean squared error asymptotically attains the Cram\'{e}r-Rao lower bound. Moreover, the proposed estimator has linear computational complexity with respect to the number of correspondences, making it suitable for high-precision and real-time robot pose estimation. Extensive experiments on synthetic data and real-world benchmarks validate the theoretical properties of FUSE-PnPL and demonstrate its advantages in both static global localization and dynamic SLAM systems.
\end{abstract}


\section{Introduction} \label{introduction}
\subsection{Background} \label{background}
Robot pose estimation refers to inferring the position and orientation of a robot in a specific coordinate system by using sensor measurements. As one of the most fundamental techniques in many robot applications, such as robot navigation, robotic grasping, and simultaneous localization and mapping (SLAM), pose estimation is an important topic in robotics and has been extensively investigated~\cite{taddese2018enhanced,zhang2021pose,lilge2022continuum}.
In virtue of low cost, high energy efficiency, and the ability to generate high-resolution images with detailed visual context, cameras have been a commonly used sensor for robot pose estimation.

A major class of vision-based methods estimates camera pose from a calibrated projection model and correspondences between image features and 3D landmarks in the global frame. Point features are the most widely used primitives, leading to the perspective-\(n\)-point (PnP) problem~\cite{lepetit2009epnp,urban2016mlpnp,terzakis2020consistently}. Distinctive and well-distributed point correspondences generally impose strong geometric constraints and enable accurate pose estimation. However, reliable point correspondences may be scarce in weakly textured regions, ambiguous under repetitive patterns, or poorly distributed across the image.
Line features provide complementary constraints and lead to the perspective-\(n\)-line (PnL) problem~\cite{xu2016pose,wang2019camera,liu2020globally}. Supported by gradients over extended regions, lines encode both orientation and scene structure and are often detectable along object boundaries and architectural edges, especially in structured, human-made environments~\cite{lu2015robust,gomez2019pl,pautrat2021sold2}. Nevertheless, line correspondences are harder to localize consistently because a supporting line does not determine a unique position along its direction, while detected endpoints may vary with viewpoint, fragmentation, and occlusion~\cite{gomez2019pl,pautrat2021sold2,pautrat2023gluestick}. Consequently, when abundant high-quality point correspondences are available, PnL methods are typically less accurate than PnP methods~\cite{vakhitov2016accurate}. These complementary properties motivate the combination of both point and lines features, which yields the perspective-\(n\)-point-and-line (PnPL) problem~\cite{vakhitov2016accurate,vakhitov2021uncertainty}.

\begin{figure*}[!htbp]
	\centering
	\subfloat[{Point features in the ETH3D dataset}\label{courtyard_big}]{%
\begin{minipage}[b]{0.5\textwidth}
		\centering
		\includegraphics[width=\textwidth]{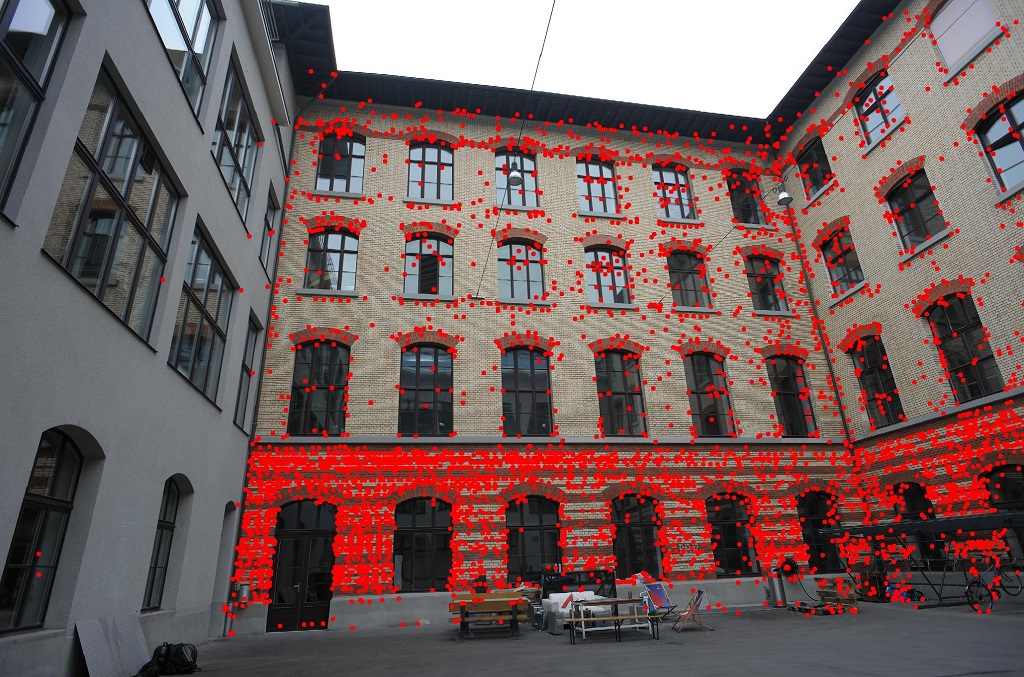}

	\end{minipage}}
\hspace{1mm}
	\subfloat[{Line features in the VGG dataset}\label{Merton_College_lines}]{%
\begin{minipage}[b]{0.44\textwidth}
		\centering
		\includegraphics[width=\textwidth]{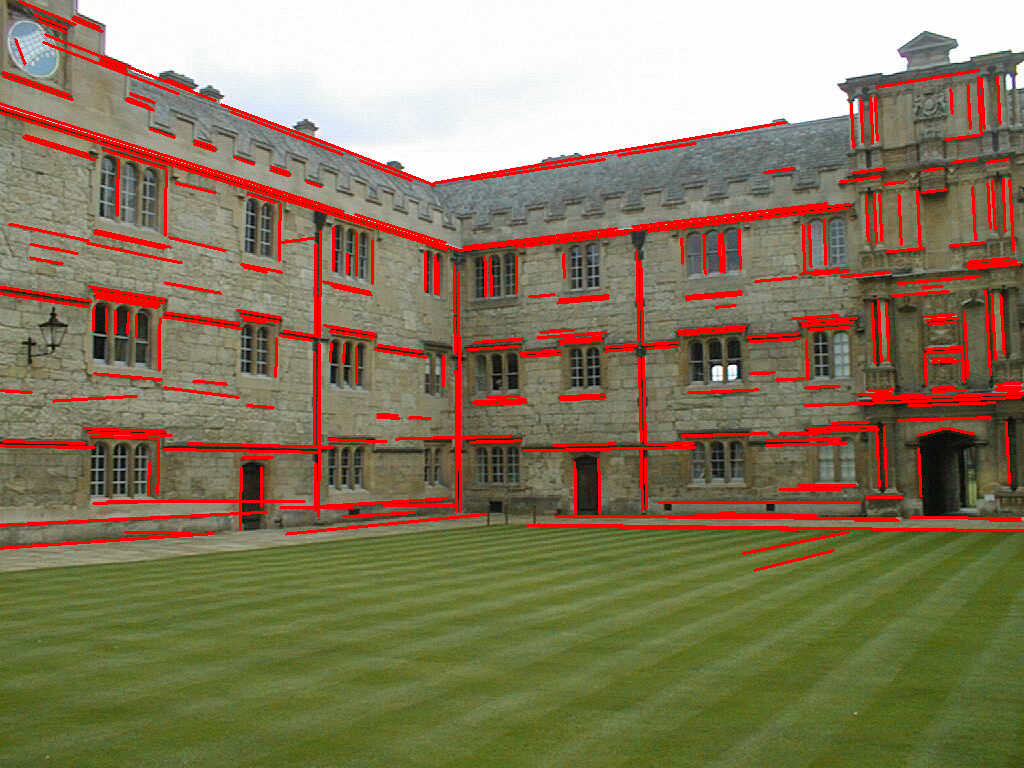}

	\end{minipage}}
	\caption{Large sample features extracted from a single image.}
	\label{large_sample_examples}
\end{figure*}

Accurate pose estimation is essential for many robotic applications, including surgical robotics~\cite{haidegger2019autonomy}, neural radiance field-based mapping~\cite{mildenhall2021nerf}, and large-scale SLAM~\cite{lynen2020large}. Advances in feature detection and description enable the extraction of numerous feature measurements from a single image, offering substantial potential for improving pose accuracy. For example, as shown in Figure~\ref{large_sample_examples}, a single image in the ETH3D dataset~\cite{schops2017multi} may contain thousands of point features, while one in the VGG dataset~\cite{werner2002new} may contain hundreds of line features.
A principled approach to exploiting these abundant measurements is to formulate and solve the corresponding maximum likelihood (ML) estimation problem. Under some mild regularity conditions, an ML estimator is \emph{consistent} and \emph{asymptotically efficient}~\cite{barfoot2024state}: as the number of features increases, the estimated pose converges in probability to the true pose, and its covariance asymptotically attains the Cram\'er--Rao bound (CRB).
ML estimation also provides a principled framework for finite-sample estimation by explicitly accounting for measurement uncertainty. Under Gaussian noise, it reduces to covariance-weighted least squares. For correctly specified, non-degenerate geometric estimation problems, the resulting estimator's covariance can match the theoretical accuracy bound to leading order in the small-noise regime, although finite-sample unbiasedness or minimum variance is not guaranteed in general~\cite{kanatani2013hyperaccurate}. 

For PnP problems, the relatively simple point parameterization and projection model make it straightforward to formulate an ML estimation problem under a specified measurement-noise model. However, many existing methods lack explicit measurement-noise modeling and rigorous statistical design. Due to the nonlinearity of the pinhole projection model, they may therefore fail to retain the consistency and asymptotic efficiency guarantees of ML estimation. 
The challenges are more pronounced for PnL and PnPL problems. The greater complexity of line parameterizations and their projection models necessitates careful treatment of noise propagation. Establishing statistically principled ML formulations for PnL and PnPL requires that the residuals and their covariances accurately reflect the underlying measurement process. This is particularly challenging because measurement-noise covariances are often unknown and must be estimated, while the resulting ML problems can be difficult to solve reliably. Beyond their role in pose estimation, accurate characterization of measurement uncertainty is also important for uncertainty quantification in localization and mapping, particularly in safety-critical robotic applications~\cite{gao2024closure,zeng2026provably}.

\subsection{Related work} \label{related_work}
\textbf{PnP problem:} Early studies of PnP problem focused on a fixed number of points. Since the degree of a pose is six and each point correspondence produces two independent equations, the minimal calibrated case is $n=3$, which corresponds to the P3P problem~\cite{dementhon1992exact,kneip2011novel}. There are also P4P~\cite{bujnak2008general} and P5P~\cite{triggs1999camera} solvers for $n=4$ and $n=5$, respectively. These fix-point solutions are sensitive to noises and generally embedded in a random sample consensus (RANSAC) framework for robust estimation. 
Most PnP solvers, however, can cope with a flexible number of points by formulating a least-squares problem. The optimization methods can be categorized into iterative 
and non-iterative approaches. Iterative methods start with an initial value and solve the problem using an iterative algorithm, e.g., the GN and Levenberg-Marquardt (LM) methods~\cite{lu2000fast,pavlakos20176,terzakis2020consistently}. Due to the nonconvexity of the problem, without a good initialization, iterative methods can only result in a local minimum. 
Early non-iterative methods usually have a high computational complexity. The first well-known $O(n)$ solver is EPnP~\cite{lepetit2009epnp}, which utilizes linearization techniques. Direct linear transformation (DLT) is another linear formulation~\cite{hartley2003multiple}. Apart from linearization techniques, polynomial solvers that do not ignore the constraints on rotation have become mainstream~\cite{hesch2011direct,zheng2013revisiting,zhou2019efficient}. They estimate the pose by solving a system of polynomial equations and can generally achieve a better performance than linear methods. 
Non-iterative solutions are sometimes used as the initial value for iterative algorithms.

\textbf{PnL problem:}
Each line correspondence provides two independent constraints; hence, the minimal calibrated case requires three line correspondences, leading to the P3L problem~\cite{dhome1989determination}. A central issue in PnL is how to represent and project line features. Common representations
include two points on the line, a point together with its direction,
the intersection of two planes, and homogeneous Pl\"ucker
coordinates~\cite{hartley2003multiple,bartoli2005structure}.
Regardless of the chosen representation, the underlying projective geometry remains the same: an observed image line back-projects to an interpretation plane passing through the camera center, within which the transformed 3D line must lie.
The two-point representation leads directly to incidence constraints
requiring the projections of both points to lie on the observed image
line~\cite{hartley2003multiple}. By contrast, Pl\"ucker coordinates
provide a homogeneous representation that is independent of the particular points chosen on the line and admits compact linear transformations and projection models~\cite{bartoli20043d,pvribyl2015camera,pvribyl2017absolute}. 
Similar to PnP, PnL methods can be broadly divided into iterative and non-iterative approaches. Iterative methods usually minimize geometric or reprojection-based errors~\cite{liu1990determination,christy1999iterative}, but may converge to local minima because of the non-convex objective. Non-iterative approaches include DLT and other linear formulations~\cite{ansar2003linear,pvribyl2017absolute}, as well as algebraic or polynomial solvers that explicitly enforce the rotation constraints~\cite{mirzaei2011globally,xu2016pose,zhou2019robust,wang2019camera}. These constrained methods generally improve accuracy over unconstrained linear formulations, particularly when only a limited number of lines is available, but may incur higher computational costs and numerical instability~\cite{zhou2020complete}. 
Finally, we remark that owing to the more involved representation and projection of line
features, many existing PnL methods employ heuristic or approximate
residual weightings rather than weights derived from the exact
propagation of measurement noise. Consequently, their least-squares
objectives generally do not constitute statistically rigorous
ML formulations.

\textbf{PnPL problem:}
Compared with PnP and PnL, relatively fewer methods have been
developed for pose estimation from mixed point and line
correspondences. In the calibrated minimal case, the two mixed configurations, P2P1L and P1P2L, were first addressed in~\cite{ramalingam2011pose} and later improved through more stable and efficient algebraic solvers~\cite{zhou2019stable,hruby2024efficient}.
For non-minimal configurations, a unified DLT formulation can be
constructed by stacking the homogeneous point-collinearity and
line-incidence constraints~\cite{hartley2003multiple}.
Another practical strategy is to incorporate line residuals into
established PnP frameworks. Vakhitov \emph{et al.} extended EPnP
and OPnP to the corresponding EPnPL and OPnPL
solvers~\cite{vakhitov2016accurate}. Li \emph{et al.} adopted the Cayley parametrization for the rotation matrix and constructed a polynomial system by the first-order optimality condition~\cite{li2017combining}. Certifiably globally
optimal approaches have also been investigated through
sum-of-squares optimization~\cite{sun2020certifiably} and
semidefinite relaxation~\cite{agostinho2023cvxpnpl}.

\textbf{Uncertainty-aware PnP(L) problem:}
Most classical PnP, PnL, and PnPL solvers optimize algebraic or
geometric objectives without explicitly modeling feature-dependent
measurement uncertainty. Since image observations are inevitably noisy, several methods incorporate 2D
feature covariances through weighted least-squares objectives based on
Mahalanobis distances~\cite{mirzaei2011globally,ferraz2014leveraging,
urban2016mlpnp,li2017combining,vakhitov2021uncertainty}.
Vakhitov \emph{et al.} further accounted for uncertainties in both 2D
observations and 3D point and line landmarks~\cite{vakhitov2021uncertainty}; however, their formulations require prior covariance information and,
in some variants, an initial pose or an approximation of the average
scene depth. More recently, GMLPnP considered anisotropic observation
noise and jointly estimated the camera pose and noise covariance within
a generalized ML framework~\cite{zhan2025generalized}.
Nevertheless, these methods mainly focus on uncertainty-based residual
weighting, while the statistical consistency and asymptotic efficiency
of the resulting estimators are generally not established. In
particular, non-linear transformations of noisy image measurements may
produce errors-in-variables formulations and introduce asymptotic bias.
CPnP~\cite{zeng2023cpnp} addressed this issue through bias elimination
and provided a consistent estimator for point correspondences; however,
it cannot handle the minimal case and is not
applicable to line or mixed point-and-line features.

\subsection{Main results} \label{main_results}
This paper revisits camera pose estimation from 3D--2D feature correspondences and proposes an estimator with the following four properties: 
1) \emph{unified}: it handles point-only, line-only, and mixed point-and-line correspondences within a single framework; 
2) \emph{flexible}: it accommodates arbitrary configurations with $n+m\geq3$, where $n$ and $m$ denote the numbers of point and line correspondences, respectively; 
3) \emph{consistent and asymptotically efficient}: as the number of features increases, the estimate converges to the true pose and its covariance asymptotically attains the CRB; and 
4) \emph{linear-time}: it has $O(n+m)$ computational complexity.
The consistency and asymptotic optimality enable highly accurate estimation when abundant features are available, while the compact pose parameterization allows the estimator to handle the minimal case $n+m=3$ and maintain stable performance with few features. Moreover, its linear time complexity makes it suitable for integration into real-time localization and SLAM systems. 
The main contributions of this paper are summarized as follows:
\begin{enumerate}
    \item [$(i).$] We adopt a two-point representation for 2D line segments and model their measurement noise through the detected endpoints. By representing 3D lines with Pl\"ucker coordinates, we explicitly analyze the propagation of endpoint noise into the line residuals and derive the corresponding statistically principled weights. This enables the construction of the likelihood function for line measurements. 
    
    \item [$(ii).$] We develop consistent estimators of the measurement noise variances for both point and line features by solving simple eigenvalue problems. These consistent noise estimates enable bias elimination in the initial pose estimation and consistent approximation of the residual weights in the ML objective, which are essential for asymptotically optimal estimation.
    
    \item [$(iii).$] By adopting the Cayley--Gibbs--Rodrigues (CGR) parameterization of rotation and a sequential strategy for estimating rotation and translation, we develop a unified and flexible pose estimation framework applicable to arbitrary point-and-line configurations with $n+m\geq3$. We design a bias-elimination scheme to construct a consistent initial estimator, from which a single GN iteration achieves asymptotic optimality and attains the CRB in covariance. The overall algorithm has $O(n+m)$ time complexity.
    
    \item [$(iv).$] We conduct extensive global localization and SLAM experiments on synthetic data, public benchmarks, and self-collected underwater datasets. Simulation results verify the consistency and asymptotic efficiency of the proposed estimator. On the VGG dataset, our PnL estimator achieves errors up to nearly three orders of magnitude smaller than those of competing methods. Integrating our method into existing SLAM pipelines reduces the mean translation and rotation relative pose errors (RPEs) on EuRoC MAV by $14.75$\% and $25$\%, respectively, compared with the original PL-SLAM, and the mean absolute trajectory error (ATE) on PoolStereo by $23.48$\% compared with the original AirSLAM.
\end{enumerate}

\section{Preliminaries} \label{preliminaries}

To facilitate the readability of the subsequent technical part, in this section, we give some preliminaries on notations, probability and statistics, rigid transformation, and GN iterations on the manifold of ${\rm SO}(3)$. 

\subsection{Notations} \label{notations}
For a vector $\bf v$, $[{\bf v}]_i$ denotes its $i$-th element, $\|{\bf v}\|_{\bm \Sigma}^2= {\bf v}^\top {\bm \Sigma}^{-1} {\bf v}$, and ${\rm diag}(\bf v)$ returns a diagonal matrix with $\bf v$ being its diagonal elements. Given a matrix ${\bf M}$, $\|{\bf M}\|_{\rm F}$ returns its Frobenius norm, and ${\rm vec}({\bf M})$ yields a vector by concatenating the columns of ${\bf M}$. If ${\bf M}$ is a square matrix, then ${\rm tr}({\bf M})$ represents its trace, and ${\rm det}({\bf M})$ denotes its determinant. If ${\bf M}$ has real eigenvalues, then $\lambda_{\rm min}({\bf M})$ and $\lambda_{\rm max}({\bf M})$ denote the minimum and maximum eigenvalues of ${\bf M}$, respectively. We use $\times $ to denote the cross product and $\otimes$ to denote the Kronecker product. We denote the identity matrix of size $n$ as ${\bf I}_n$ and the zero matrix of size $n \times m$ as ${\bf 0}_{n \times m}$. 
For any quantity ${\bf x}$, we denote its estimate by $\hat {\bf x}$ and its true or noise-free value by ${\bf x}^o$.

\subsection{Some notions in probability and statistics} \label{preliminaries_probability_statistics}

\hspace{1.2em} \textbf{Convergence in probability.} The small $o_p$ notation $X_n=o_p(a_n)$ means that the sequence $(X_n/a_n)$ converges to zero in probability. That is, for any $\varepsilon>0$, 
\begin{equation*}
    \lim_{n \rightarrow \infty} \mathbb{P} (|X_n/a_n| > \varepsilon)=0.
\end{equation*} 

\textbf{Stochastic boundedness.} The big $O_p$ notation $X_n=O_p(a_n)$ means that the sequence $(X_n/a_n)$ is stochastically bounded. That is, for any $\varepsilon >0$, there exists a finite $M$ and a finite $N$ such that for any $n>N$,
\begin{equation*}
    \mathbb{P} (|X_n/a_n|>M )<\varepsilon.
\end{equation*}

The following notions are established on the premise that given $n$ measurements $\mathcal Y_n \triangleq \{{\bf y}_1,\ldots,{\bf y}_n\}$, $\hat {\bm \theta}_n$ is an estimate of some unknown parameter ${\bm \theta}^o$.

\textbf{Consistent and $\sqrt{n}$-consistent estimate.} The estimate $\hat {\bm \theta}_n$ is called a (weakly) consistent estimate of ${\bm \theta}^o$ if
\begin{equation*}
    \hat {\bm \theta}_n-{\bm \theta}^o=o_p(1).
\end{equation*}
Furthermore, $\hat {\bm \theta}_n$ is called a \emph{$\sqrt{n}$-consistent} estimate of ${\bm \theta}^o$ if
\begin{equation*}
    \hat {\bm \theta}_n-{\bm \theta}^o=O_p(1/\sqrt{n}).
\end{equation*}
It is noteworthy that $\sqrt{n}$-consistent includes two implications: The estimate $\hat {\bm \theta}_n$ is consistent, i.e., $\hat {\bm \theta}_n$ converges to ${\bm \theta}^o$ in probability; The convergence rate is $1/\sqrt{n}$. 

\textbf{(Asymptotically) unbiased estimate.} The bias of an estimate $ \hat {\bm \theta}_n$ is equal to its expectation minus the true value, i.e., ${\rm Bias}(\hat {\bm \theta}_n)=\mathbb E[\hat {\bm \theta}_n]-{\bm \theta}^o$. If ${\rm Bias}(\hat {\bm \theta}_n)=0$, we call $\hat {\bm \theta}_n$ an unbiased estimate of ${\bm \theta}^o$. In particular, if $\lim_{n \rightarrow \infty}{\rm Bias}(\hat {\bm \theta}_n)=0$, we call $\hat {\bm \theta}_n$ an asymptotically unbiased estimate. It is noteworthy that an asymptotically unbiased estimate $\hat {\bm \theta}_n$ may not necessarily be unbiased given a finite $n$. 

\textbf{(Asymptotically) efficient estimate.} An unbiased estimate $\hat {\bm \theta}_n$ is said to be efficient if its covariance ${\rm cov}(\hat {\bm \theta}_n)$ equals the theoretical CRB. In particular, $\hat {\bm \theta}_n$ is called asymptotically efficient if it is asymptotically unbiased, and $\lim_{n \rightarrow \infty} {\rm cov}(\sqrt{n}\hat {\bm \theta}_n)=n{\rm CRB}$. CRB can be obtained from the Fisher information matrix $\bf F$ as ${\rm CRB}={\bf F}^{-1}$, and $\bf F$ is the covariance of the derivative of the log-likelihood function $\ell({\bm \theta};\mathcal Y_n)$ evaluated at the true parameter ${\bm \theta}^o$:
\begin{equation*}
    {\bf F}=\mathbb E \left[\frac{\partial \ell({\bm \theta};\mathcal Y_n)}{\partial {\bm \theta}} \frac{\partial \ell({\bm \theta};\mathcal Y_n)}{\partial {\bm \theta}^\top} \bigg| {\bm \theta}^o\right].
\end{equation*}
% An asymptotically efficient estimate has an asymptotic covariance no worse than any other estimate~\cite{greene2003econometric}. In this sense, an asymptotically efficient estimate is also called an \emph{asymptotically optimal} estimate. 

\subsection{Rigid transformation} \label{rigid_transformation}
The rigid transformations, or said relative poses, include rotations and translations. In the 3D Euclidean space, the translation is depicted by a vector ${\bf t} \in \mathbb R^3$. The rotation can be characterized by a rotation matrix $\bf R$. Specifically, rotation matrices belong to the special orthogonal group
\begin{equation*}
{\rm SO}(3)=\{{\bf R}\in \mathbb R^{3 \times 3} \mid {\bf R}^\top {\bf R}={\bf I}_3, {\rm det}({\bf R})=1\}.
\end{equation*}
The group operation is the usual matrix multiplication, and the inverse is the matrix transpose. The group ${\rm SO}(3)$ forms a smooth manifold, and the tangent space to the manifold (at the identity) is denoted as the Lie algebra $\mathfrak{so}(3)$ which is comprised of all $3 \times 3$ skew-symmetric matrices. Every $3 \times 3$ skew-symmetric matrix can be generated from a vector ${\bf s} \in \mathbb R^3$ via the \emph{hat} operator:
\begin{equation} \label{hat_operator}
    {\bf s}^\wedge = \begin{bmatrix}
        s_1 \\
        s_2 \\
        s_3
    \end{bmatrix}^\wedge = \begin{bmatrix}
	0 & -s_3 & s_2 \\
	s_3 & 0 & -s_1 \\
	-s_2 & s_1 & 0 
\end{bmatrix} \in \mathfrak {so} (3).
\end{equation}
% Conversely, the \emph{vee} operator retrieves the vector from a skew-symmetric matrix:
% \begin{equation*}
%     \left({\bf s}^\wedge \right)^\vee = {\bf s} \in  \mathbb R^3.
% \end{equation*}

The exponential map ${\rm exp}:\mathfrak{so}(3) \rightarrow {\rm SO}(3)$, which coincides with standard matrix exponential, associates an element in the Lie algebra to a rotation matrix. Specifically, Rodrigues' rotation formula reads that
\begin{equation*}
    {\rm exp}\left({\bf s}^\wedge \right)= {\bf I}_3 + \frac{\sin(\|{\bf s}\|)}{\|{\bf s}\|} {\bf s}^\wedge + \frac{1-\cos(\|{\bf s}\|)}{\|{\bf s}\|^2} \left({\bf s}^\wedge \right)^2.
\end{equation*}

% Conversely, the logarithm map ${\rm log}:{\rm SO}(3) \rightarrow \mathfrak{so}(3)$ associates a rotation matrix ${\bf R} \neq {\bf I}_3$ to a skew-symmetric matrix:
% \begin{equation*}
%     {\rm log}({\bf R})=\frac{\phi({\bf R}-{\bf R}^\top)}{2 \sin(\phi)} ~~ \text{with}~~ \phi = \cos^{-1} \left( \frac{{\rm tr}(\bf R)-1}{2}\right).
% \end{equation*}
% One can further write ${\rm log}({\bf R})^\vee=\phi {\bf n}$, where the unit vector $\bf n$ denotes the rotation axis, and $\phi$ is the rotation angle. It is noteworthy that the exponential map is bijective if we confine its domain to the open ball $\{{\bf s} \mid \|{\bf s}\|<\pi\}$, and its inverse is the logarithm map.  

Suppose the relative pose of the second frame with respect to (w.r.t.) the first frame is $({\bf R},{\bf t})$, and the coordinates of a 3D point in the second frame is ${\bf X}$. Then the coordinates of the point in the first frame is ${\bf R}{\bf X}+{\bf t}$. 

\subsection{Gauss-Newton method on ${\rm SO}(3)$} \label{GN_manifold}
The GN method is commonly used to solve non-linear least-squares problems iteratively. As an approximation to Newton’s method, it avoids computing second-order derivatives. At each iteration, GN linearizes the residuals using a first-order Taylor expansion and solves the resulting quadratic subproblem for the optimal increment.
When it comes to an optimization problem over the manifold of ${\rm SO}(3)$, the standard GN method needs to be modified with some techniques to ensure that the updated value after each iteration also belongs to  ${\rm SO}(3)$. 

Consider the following optimization problem:
\begin{equation} \label{general_so3_optimization}
    \mathop{\rm minimize}\limits_{{\bf R} \in {\rm SO}(3),{\bf t} \in \mathbb R^3} ~ \|{\bf r}({\bf R},{\bf t})\|^2. 
\end{equation}
Note that a rotation in the 3D space has three degrees of freedom, whereas a rotation matrix $\bf R$ contains nine entries and is therefore an overparameterized representation. Directly optimizing these entries may result in a rank-deficient Jacobian, and an unconstrained update generally does not preserve the constraint ${\bf R} \in {\rm SO}(3)$. To overcome this issue, we can introduce a \emph{retraction} $\mathcal {R}_{\bf R}$, which locally maps a tangent-space perturbation ${\bf s} \in \mathbb R^3$ to a neighborhood of $\bf R$ on ${\rm SO}(3)$. In this paper, we use the exponential-map retraction
\begin{equation} \label{retraction}
    \mathcal {R}_{\bf R} ({\bf s}) = {\bf R} ~{\rm exp} ({\bf s}^\wedge), ~~ {\bf s} \in \mathbb R^3.
\end{equation}
Given an initial estimate $\hat {\bf R} \in {\rm SO}(3)$, using this retraction, the original problem~\eqref{general_so3_optimization} is converted into 
\begin{equation} \label{general_so3_optimization2}
    \mathop{\rm minimize}\limits_{{\bf s,t} \in \mathbb R^3} ~ \|{\bf r}\left(\mathcal {R}_{\hat {\bf R}} ({\bf s}),{\bf t} \right)\|^2.
\end{equation}

The above reparameterization is usually called \emph{lifting}. Its key idea is to perform optimization in the tangent space at the current estimate $\hat {\bf R}$, which is locally Euclidean and has three dimensions, matching the three degrees of freedom of a 3D rotation.
Problem~\eqref{general_so3_optimization2} is therefore an optimization problem in Euclidean space and can be solved using the standard GN method. Once an updated $\hat {\bf s}$ is obtained, the rotation estimate is updated as $\mathcal {R}_{\hat {\bf R}} (\hat {\bf s})$ which naturally belongs to ${\rm SO}(3)$.

\section{Problem formulation} \label{problem_formulation}
\subsection{Point and line projection models} \label{point_line_model}

\begin{figure}[!b]
	\centering
	\includegraphics[width=0.9\columnwidth]{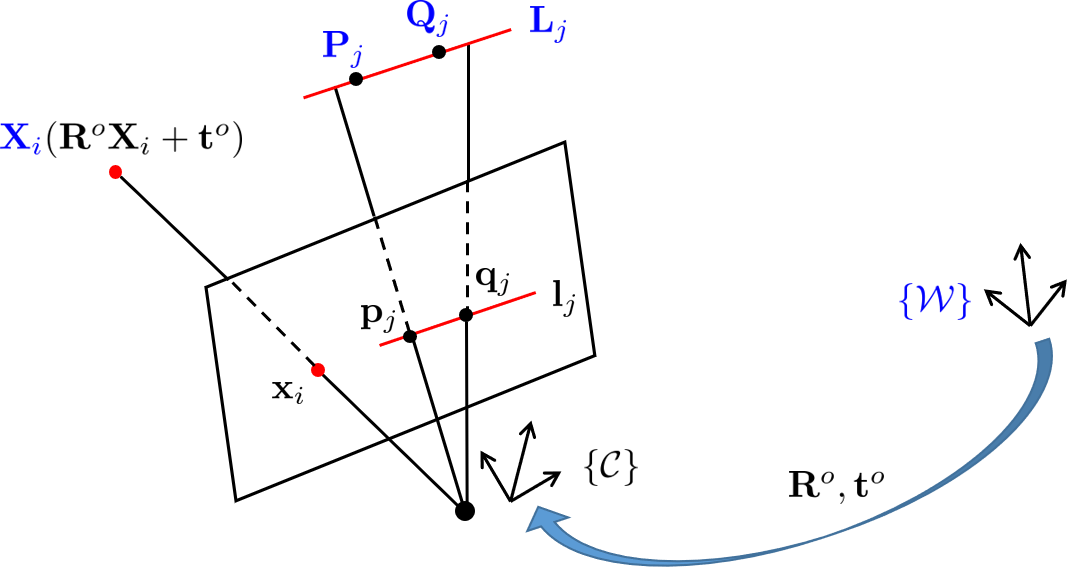}
	\caption{Illustration of point and line correspondences.}
	\label{geometry_illustration}
\end{figure}

We adopt a pinhole camera model and assume that the camera intrinsics are known, for example, from prior calibration. We therefore use normalized image coordinates, under which the intrinsic matrix is the identity and the focal length is $f=1$. 
Let $\{ \mathcal W\}$ and $\{ \mathcal C\}$ denote the world frame and camera frame, respectively. 
To estimate the relative pose of $\{ \mathcal W\}$ w.r.t. $\{ \mathcal C\}$, represented by a rotation matrix ${\bf R}^o$ and a translation vector ${\bf t}^o$, we assume that $n$ point correspondences $\{({\bf X}_i,{\bf x}_i) \}_{i=1}^{n}$ and $m$ line correspondences $\{({\bf L}_j,{\bf l}_j)\}_{j=1}^{m}$ are available. 
Here, ${\bf X}_i \in \mathbb R^3$ denotes the coordinates of the $i$-th 3D point in the world frame, and ${\bf x}_i \in \mathbb R^2$ represents its projection onto the image plane in normalized image coordinates. Let ${\bf X}_i^h \triangleq [{\bf X}_i^\top~1]^\top$ and ${\bf x}_i^h \triangleq [{\bf x}_i^\top~1]^\top$ be the corresponding homogeneous coordinates. As illustrated in Figure~\ref{geometry_illustration}, we have the following relationship:
\begin{equation} \label{linear_point_projection}
    {\bf x}_i^h \approx {\bf R}^o {\bf X}_i+{\bf t}^o = [{\bf R}^o~~{\bf t}^o] {\bf X}_i^h,
\end{equation}
where $\approx$ denotes equality up to a non-zero scale factor. 
More specifically, the exact point projection function is given by~\cite{zeng2023cpnp}
\begin{equation} \label{point_projection}
    {\bf x}_i = h_{{\bf R}^o,{\bf t}^o}({\bf X}_i) \triangleq \frac{{\bf E}({\bf R}^o {\bf X}_i+{\bf t}^o)}{{\bf e}_3^\top({\bf R}^o {\bf X}_i+{\bf t}^o)},
\end{equation}
where ${\bf E}=[{\bf e}_1~{\bf e}_2]^\top$, and ${\bf e}_i$ denotes the $i$-th standard basis vector. 

For the $j$-th line ${\bf L}_j$ in the world frame, we denote it by two distinct points ${\bf P}_j$ and ${\bf Q}_j$ on it and parameterize it with Pl\"{u}cker coordinates~\cite{hartley2003multiple} as follows:
\begin{equation*}
    {\bf L}_j \triangleq \left[({\bf P}_j \times  {\bf Q}_j)^\top~~({\bf Q}_j-{\bf P}_j)^\top \right]^\top \in \mathbb R^6. 
\end{equation*}
Note that a line in the 3D space has four degrees of freedom; thus, the Pl\"{u}cker-coordinate representation is overparameterized. Nevertheless, it is widely used in geometric computer vision because it yields a linear line projection model~\cite{hartley2003multiple,pvribyl2017absolute,bartoli20043d}. We represent the image line ${\bf l}_j$, which is the projection of ${\bf L}_j$, by two distinct points ${\bf p}_j$ and ${\bf q}_j$ on it as ${\bf l}_j \triangleq [{\bf p}_j^\top~{\bf q}_j^\top]^\top$. 
Its homogeneous line coordinates are defined as
\begin{equation*}
    \tilde {\bf l}_j \triangleq {\bf p}_j^h \times {\bf q}_j^h \in \mathbb R^3,
\end{equation*}
where ${\bf p}_j^h \triangleq [{\bf p}_j^\top~1]^\top$ and ${\bf q}_j^h \triangleq [{\bf q}_j^\top~1]^\top$. 
Then, the linear line projection is expressed as~\cite{hartley2003multiple} 
\begin{equation} \label{linear_line_projection}
    \tilde {\bf l}_j \approx \left[{\bf R}^o~~{\bf t}^{o\wedge} {\bf R}^o \right] {\bf L}_j,
\end{equation}
where the \emph{hat} operator $\wedge$ is defined in~\eqref{hat_operator}.
It is worth mentioning that, for the $j$-th line correspondence $({\bf L}_j,{\bf l}_j)$, their endpoints do not generally form point correspondences, i.e., ${\bf p}_j$ and ${\bf q}_j$ need not be the projections of ${\bf P}_j$ and ${\bf Q}_j$, respectively, as illustrated in Figure~\ref{geometry_illustration}. 
% Second, the matrix ${\bf t}^{o\wedge} {\bf R}^o$ is the essential matrix that is widely used in epipolar geometry. One can recover ${\bf R}^o$ and ${\bf t}^o$ based on the singular value decomposition (SVD) from the essential matrix and the geometric constraint that the observed 3D points and lines should be in front of the camera~\cite{hartley2003multiple}. 

\subsection{ML problem} \label{ML_problem_construction}
In the previous subsection, we introduced the point and line models under ideal projection geometry. In this subsection, we explicitly account for measurement (or projection) noise and formulate a ML estimation problem using both point and line features.
Specifically, we assume that the image measurements are corrupted by additive noise:
\begin{align} 
    {\bf x}_i & = h_{{\bf R}^o,{\bf t}^o}({\bf X}_i) + \bm{\epsilon}_{{\bf x}_i} ={\bf x}_i^o + \bm{\epsilon}_{{\bf x}_i}, \label{noisy_pt_measurement_model} \\
    {\bf p}_j & ={\bf p}_j^o + \bm{\epsilon}_{{\bf p}_j},~~~~~~{\bf q}_j ={\bf q}_j^o + \bm{\epsilon}_{{\bf q}_j}, \label{noisy_line_measurement_model}
\end{align}
where the noise vectors $\bm{\epsilon}_{{\bf x}_i}$, $\bm{\epsilon}_{{\bf p}_j}$, and $\bm{\epsilon}_{{\bf q}_j}$ are assumed to be mutually independent and distributed as $\bm{\epsilon}_{{\bf x}_i} \sim \mathcal N \left(0,\sigma^2_p {\bf I}_2 \right)$ and $\bm{\epsilon}_{{\bf p}_j},\bm{\epsilon}_{{\bf q}_j} \sim \mathcal N \left(0,\sigma^2_l {\bf I}_2 \right)$. 

\begin{remark}
\label{remark_on_line_noise}
For the image line ${\bf l}_j$, we do not directly add a 
3D noise vector to its noise-free homogeneous coordinates
$\tilde{\bf l}_j^o$. Instead, we model its two observed endpoints as being
corrupted by 2D additive noise. The induced perturbation in the homogeneous line coordinates is
\[
\tilde{\bf l}_j-\tilde{\bf l}_j^o
=
({\bf p}_j^h)^o \times \bm{\epsilon}_{{\bf q}_j}^h
+
\bm{\epsilon}_{{\bf p}_j}^h \times ({\bf q}_j^h)^o
+
\bm{\epsilon}_{{\bf p}_j}^h \times
\bm{\epsilon}_{{\bf q}_j}^h,
\]
where $\bm{\epsilon}_{{\bf p}_j}^h
\triangleq
[\bm{\epsilon}_{{\bf p}_j}^{\top}~0]^{\top}$ and $\bm{\epsilon}_{{\bf q}_j}^h
\triangleq
[\bm{\epsilon}_{{\bf q}_j}^{\top}~0]^{\top}$.
Due to the bilinear term
$\bm{\epsilon}_{{\bf p}_j}^h \times
\bm{\epsilon}_{{\bf q}_j}^h$, the induced noise in
$\tilde{\bf l}_j$ is generally non-Gaussian.
Despite the resulting analytical complexity, this endpoint-based noise
model is more consistent with practical line-segment measurements, which
are extracted from image data and commonly represented by their endpoints.
It also naturally captures the dependence of line uncertainty on segment
length. For a fixed endpoint-noise level, the orientation
uncertainty generally decreases as the distance between the two endpoints
increases. Consequently, longer line segments typically provide more
reliable geometric constraints. Endpoint-based line-noise models have also
been widely adopted in previous studies
~\cite{pvribyl2017absolute,zhou2020complete,mirzaei2011globally,xu2016pose}.
\end{remark}

% \begin{remark} \label{remark_on_line_noise}
%     Note that for the line ${\bf l}_j$, we do not directly add a 3D noise vector to $\tilde {\bf l}_j^o$. Instead, we add a 2D noise to each of its endpoints. Although this modeling is more complicated and harder to handle\footnote{With this modeling, the noise for $\tilde {\bf l}_j$ becomes ${\bf p}_j^{oh} \times \bm{\epsilon}_{{\bf q}_j}^h +\bm{\epsilon}_{{\bf p}_j}^h  \times {\bf q}_j^{oh} + \bm{\epsilon}_{{\bf p}_j}^h  \times \bm{\epsilon}_{{\bf q}_j}^h $, which is no longer a Gaussian noise. Here, $\bm{\epsilon}_{{\bf p}_j}^h \triangleq [\bm{\epsilon}_{{\bf p}_j}^\top~0]^\top$ and $\bm{\epsilon}_{{\bf q}_j}^h \triangleq[\bm{\epsilon}_{{\bf q}_j}^\top~0]^\top$.}, it is more in line with the fact that the original measurements of a camera are a series of points. This modeling also satisfies the intuition that longer line segments are less affected by the noises on their endpoints, i.e., the longer a line segment is, the more confident we are with the observation. Actually, adding noises to the endpoints of a line is widely used in the simulation of previous works~\cite{pvribyl2017absolute,zhou2020complete,mirzaei2011globally,xu2016pose}. 
% \end{remark}

% \begin{remark} \label{remark_on_Gaussian_noise}
%     The i.i.d. Gaussian noise assumption has been widely adopted in estimations in computer vision, e.g.,~\cite{lepetit2009epnp,hesch2011direct,urban2016mlpnp}.
% \end{remark}

Based on the point measurement model~\eqref{noisy_pt_measurement_model}, the residual for the $i$-th point is 
\begin{equation} \label{point_residual}
    {\bf r}_{pi}({\bf R},{\bf t})={\bf x}_i-h_{{\bf R},{\bf t}}({\bf X}_i),
\end{equation}
and the weighting matrix is $\bm{\Sigma}^{-1}_p=\frac{1}{\sigma^2_p} {\bf I}_2$. 
Defining a line residual is less straightforward because the 3D points corresponding to the observed endpoints ${\bf p}_j$ and ${\bf q}_j$ are unknown. Let
\begin{equation*}
    \bar {\bf l}_j \triangleq \left[{\bf R}~~{\bf t}^{\wedge} {\bf R} \right] {\bf L}_j.
\end{equation*}
Note that ${\bf p}_j^o$ and ${\bf q}_j^o$ lie on ${\bf l}_j^o$. With $\bar {\bf l}_j^o\triangleq \left[{\bf R}^o~~{\bf t}^{o \wedge} {\bf R}^o \right] {\bf L}_j$, we have $\bar {\bf l}_j^{o \top} ({\bf p}_j^{h})^o =0$ and $\bar {\bf l}_j^{o \top} ({\bf q}_j^{h})^o =0$, which give 
\begin{equation} \label{modified_line_measurement}
    0=\bar {\bf l}_j^{o \top} {\bf p}_j^{h} - \bar {\bf l}_j^{o \top} \bm{\epsilon}_{{\bf p}_j}^{h},~~0=\bar {\bf l}_j^{o \top} {\bf q}_j^{h} - \bar {\bf l}_j^{o \top} \bm{\epsilon}_{{\bf q}_j}^{h},
\end{equation}
where $\bm{\epsilon}_{{\bf p}_j}^h$ and $\bm{\epsilon}_{{\bf q}_j}^h$ are the homogeneous coordinates of $\bm{\epsilon}_{{\bf p}_j}$ and $\bm{\epsilon}_{{\bf q}_j}$, respectively\footnote{For a point, we add $1$ to denote its homogeneous coordinates, e.g., ${\bf p}_j^{h}=[{\bf p}_j^{\top}~1]^\top$, while for a noise vector, we add 0 to represent its homogeneous coordinates, e.g., ${\bm \epsilon}_{pj}^{h}=[{\bm \epsilon}_{pj}^{\top}~0]^\top$.}. 
The noise terms $\bar {\bf l}_j^{o \top} \bm{\epsilon}_{{\bf p}_j}^{h}$ and $\bar {\bf l}_j^{o \top} \bm{\epsilon}_{{\bf q}_j}^{h}$ in~\eqref{modified_line_measurement} follow the Gaussian distribution $\mathcal N \left(0,\sigma_{lj}^2 ({\bf R}^o,{\bf t}^o)\right)$, where we use $\sigma_{lj}^2 ({\bf R}^o,{\bf t}^o)$ to indicate that $\sigma_{lj}^2$ is a function of ${\bf R}^o$ and ${\bf t}^o$. 
The explicit expression of $\sigma_{lj}^2 ({\bf R}^o,{\bf t}^o)$ is $\sigma_{lj}^2 ({\bf R}^o,{\bf t}^o)=[\bar {\bf l}_j^{o}]_{1:2}^\top {\bm \Sigma}_l [\bar {\bf l}_j^{o}]_{1:2}$, where $[\bar {\bf l}_j^{o}]_{1:2}$ denotes the first two elements of $\bar {\bf l}_j^{o}$, and ${\bm \Sigma}_l=\sigma^2_l {\bf I}_2$.
As a result, the residual for the $j$-th line is formulated as 
\begin{equation} \label{line_residual}
    {\bf r}_{lj}({\bf R},{\bf t}) = \begin{bmatrix}
        {\bf p}_j^{h \top} \\
        {\bf q}_j^{h \top}
    \end{bmatrix} \bar {\bf l}_j = \begin{bmatrix}
        {\bf p}_j^{h \top} \\
        {\bf q}_j^{h \top}
    \end{bmatrix} \left[{\bf R}~~{\bf t}^{\wedge} {\bf R} \right] {\bf L}_j,
\end{equation}
and the weighting matrix is $\bm{\Sigma}_{lj}^{-1}=\frac{1}{\sigma_{lj}^2 ({\bf R}^o,{\bf t}^o)} {\bf I}_2$.

Finally, by combining the point residual~\eqref{point_residual} and line residual~\eqref{line_residual}, the ML problem is formulated as 
\begin{equation} \label{ML_problem}
    \begin{split}
        \mathop{\rm minimize}\limits_{{\bf R},{\bf t} } ~ & \frac{1}{n+m} \left( \sum_{i=1}^{n} \left\|{\bf r}_{pi}\right\|_{\bm {\Sigma}_p}^2 + \sum_{j=1}^{m} \left\|{\bf r}_{lj}\right\|_{\bm {\Sigma}_{lj}}^2 \right) \\
        \mathop{\rm subject~ to} ~ & {\bf R} \in {\rm SO}(3), {\bf t} \in \mathbb R^3,
    \end{split}
\end{equation}
where we have abbreviated ${\bf r}_{pi}({\bf R},{\bf t})$ and ${\bf r}_{lj}({\bf R},{\bf t})$ to ${\bf r}_{pi}$ and ${\bf r}_{lj}$, respectively. A global solution to~\eqref{ML_problem} is called an ML estimate, denoted as $\hat {\bf R}^{\rm ML}$ and $\hat {\bf t}^{\rm ML}$. A 3D pose has six degrees of freedom, and each point or line correspondence provides two independent constraints. Therefore, at least three point/line correspondences, i.e., $n+m \geq 3$, are necessary to estimate the camera pose. In what follows, we assume that the points and lines are in a generic configuration such that the camera pose is uniquely identifiable.
\begin{assumption}[(Pose identifiability)] \label{assump:identifiability}
    In the noise-free case, the true pose $({\bf R}^o,{\bf t}^o)$ is the unique pose consistent with all point and line correspondences.
\end{assumption}

Assumption~\ref{assump:identifiability} excludes both degenerate
configurations, which may admit a continuum of solutions, and ambiguous
configurations, which admit multiple isolated solutions. 
Examples of degeneracy include collinear 3D points and geometrically dependent line constraints~\cite{wu2006pnp,xu2016pose}. Moreover, even non-degenerate minimal configurations, such as P3P, P2P1L, P1P2L, and P3L, may yield multiple pose candidates~\cite{zhou2019stable,xu2016pose}. Assumption~\ref{assump:identifiability} is therefore typically satisfied when the correspondences contain a non-degenerate minimal subset and the remaining correspondences eliminate all spurious candidates. This condition is generally expected to hold when sufficiently many features are in generic positions.
Under Assumption~\ref{assump:identifiability} and other mild regularity conditions, classical statistical theory establishes that the ML estimate is consistent and asymptotically efficient as the number of measurements increases. Modern feature-extraction techniques can provide a large number of feature correspondences from a single image. For example, Figure~\ref{large_sample_examples} shows images from the ETH3D~\cite{schops2017multi} and VGG~\cite{werner2002new} datasets, containing thousands of point features and hundreds of line features, respectively. The prevalence of such large-scale feature sets motivates us to develop an efficient method for solving the ML problem~\eqref{ML_problem}.  
\begin{remark}
Assumption~\ref{assump:identifiability} is required only to establish the consistency and asymptotic efficiency of the proposed estimator; it does not preclude the application of our algorithm to minimal configurations. In fact, the algorithm can also handle the minimal configurations studied in~\cite{zhou2019stable,xu2016pose}. However, because such configurations may admit multiple poses consistent with the given correspondences, the algorithm may return several discrete pose solutions rather than a unique one, as is the case with other minimal solvers.
\end{remark}

\section{Algorithm framework}
The ML problem~\eqref{ML_problem} is non-convex since the objective is non-convex, and the ${\rm SO}(3)$ group is not a convex set. In addition, the covariance matrices ${\bm \Sigma}_p$ and $\bm{\Sigma}_{lj},j=1,\ldots,m$ are unknown. As a result, a global solution to the ML problem is hard to obtain. When using a locally iterative method to solve the ML problem, e.g., the GN algorithm, it is required that the initial value falls into the attraction neighborhood of a global minimum, otherwise, we can only obtain a locally optimal solution. 

\begin{figure*}[!htbp]
	\centering
	\includegraphics[width=0.9\textwidth]{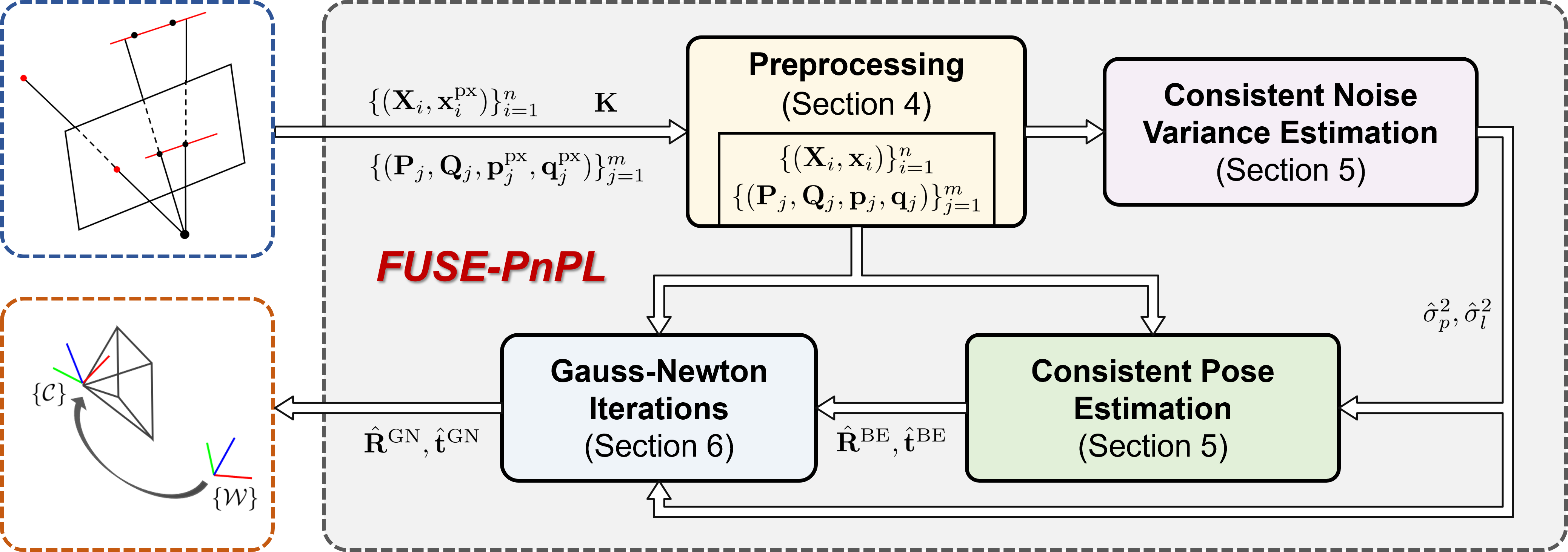}
	\caption{Overview of algorithm framework.}
	\label{algorithm_framework}
\end{figure*}

In this paper, our main focus is to devise an algorithm that can yield an estimate, which owns the same asymptotic properties as the ML estimate. In other words, we are going to propose a consistent and asymptotically efficient estimator. The algorithm framework is shown in Figure~\ref{algorithm_framework}. The input to our algorithm includes the camera intrinsic matrix $\bf K$, $n$ point correspondences, and $m$ line correspondences, with all 2D points represented in pixel coordinates, denoted by the superscript $(\cdot)^{\rm px}$. The output is the estimate of the camera pose. Our algorithm consists of four components.
In what follows, we will outline their functions respectively. 

\emph{(1) Preprocessing.} 
% The preprocessing module includes two steps. First, as illustrated in Remark~\ref{remark_on_Gaussian_noise}, the directional uncertainty of each detected 2D point (including the endpoints of detected lines) is calculated by using its neighbor pixel information. Then, points are weighted accordingly so that they approximately have the same covariance matrix $\sigma^2 {\bf I}_2$. 
Following standard practice in the literature~\cite{hartley2003multiple,pvribyl2017absolute}, we normalize the point and line measurements to improve the numerical stability of the algorithm. The underlying principle is to rescale quantities with large magnitudes to a numerically well-conditioned range. 
For each detected 2D point, we convert its pixel coordinates into normalized image coordinates using the camera intrinsic matrix. For example, the normalized coordinates corresponding to ${\bf x}_i^{\rm px}$ are computed as 
\begin{equation*}
    {\bf x}_i = [{\bf K}^{-1}]_{1:2} \begin{bmatrix}
        {\bf x}_i^{\rm px} \\
        1
    \end{bmatrix},
\end{equation*}
where $[{\bf K}^{-1}]_{1:2}$ denotes the first two rows of ${\bf K}^{-1}$. The normalized endpoints ${\bf p}_j$ and ${\bf q}_j$ of each 2D line segment are obtained in the same manner.
For each 3D line ${\bf L}_j$, we move its endpoints ${\bf P}_j$ and ${\bf Q}_j$ along the line direction such that $\|{\bf P}_j-{\bf Q}_j\|=\sqrt{3}$. 

\emph{(2) Noise variance estimation.} For non-linear and non-convex optimization problems, solutions obtained through relaxation-based methods may be biased, with the bias generally depending on the measurement-noise variances. Accurately estimating the noise variances is therefore essential for correcting this bias and obtaining an asymptotically consistent estimator. The technical details of noise-variance estimation are presented in Section~\ref{consistent_estimator_design}. The resulting consistent estimates of the noise variances are used both to construct the consistent pose estimate and to perform the subsequent GN refinement.

\emph{(3) Consistent estimate generation.} We parameterize the rotation matrix using the CGR representation and eliminate the translation parameters via ordinary least squares. The rotation is then recovered by solving a system of three cubic polynomial equations. Using the estimated noise variances, we further correct the relaxation-induced bias to obtain a consistent pose estimate. This estimate provides a reliable initialization for the subsequent GN iterations, which refine the solution toward the ML estimate. Further details on the construction of the consistent estimator are provided in Section~\ref{consistent_estimator_design}.

\emph{(4) Gauss-Newton iterations.}  Using the consistent estimate as an initialization, we apply the GN algorithm to find the ML estimate. As the number of measurements increases, the consistent estimate converges to the true pose and, with probability approaching one, lies within the  attraction basin of the global minimizer of the ML problem~\eqref{ML_problem}. Consequently, the ML estimate can be attained in the large-sample regime. The proposed GN iterations and the asymptotic efficiency of the resulting estimator are presented in Section~\ref{asymptotic_property_analysis}.

\section{Consistent estimator design} \label{consistent_estimator_design}
In this section, we will detail the initial consistent estimator. We parametrize the rotation matrix by CGR representation 
\begin{equation} \label{eqn:CGR}
    {\bf R}=\frac{\bar {\bf R}}{1+{\bf g}^\top{\bf g}}, ~\bar {\bf R}=(1-{\bf g}^\top{\bf g}) {\bf I}_3+2{\bf g}^{\wedge}+2{\bf g} {\bf g}^\top,
\end{equation}
where ${\bf g} =[g_1~g_2~g_3]^\top \in \mathbb R^3$. Then, the residual for the $i$-th point can be rewritten as 
\begin{equation*} 
    {\bf r}_{pi}={\bf x}_i-\frac{{\bf E}(\bar {\bf R} {\bf X}_i+{\bm \tau})}{{\bf e}_3^\top(\bar {\bf R} {\bf X}_i+{\bm \tau})},
\end{equation*}
where ${\bm \tau}=(1+{\bf g}^\top{\bf g}) {\bf t}$. By multiplying the residual by ${\bf e}_3^\top(\bar {\bf R} {\bf X}_i+{\bm \tau})$, we further obtain a modified linear (w.r.t. $\bar {\bf R}$ and ${\bm \tau}$) point residual:
\begin{equation} \label{point_residual2}
   \bar {\bf r}_{pi} \triangleq {\bf e}_3^\top(\bar {\bf R} {\bf X}_i+{\bm \tau}){\bf x}_i-{\bf E}(\bar {\bf R} {\bf X}_i+{\bm \tau}).
\end{equation}
Let ${\bm \varphi} \triangleq {\rm vec}(\bar {\bf R})$. Then,~\eqref{point_residual2} can be rearranged as
\begin{equation} \label{point_residual3}
   \bar {\bf r}_{pi} = {\bf A}_{pi} {\bm \varphi}+{\bf B}_{pi} {\bm \tau},
\end{equation}
where ${\bf A}_{pi}=({\bf x}_i {\bf e}_3^\top-{\bf E}) \bar {\bf X}_i$, ${\bf B}_{pi}={\bf x}_i {\bf e}_3^\top-{\bf E}$, and $\bar{\bf X}_i={\bf X}_i^\top \otimes {\bf I}_3$.

For lines, in the section of ML problem formulation, we utilize the geometry relationships $\bar {\bf l}_j^{o \top} ({\bf p}_j^{h})^o =0$ and $\bar {\bf l}_j^{o \top} ({\bf q}_j^{h})^o =0$. Note also that the transformed endpoints of a 3D line lie on the corresponding interpretation plane, i.e., $\tilde {\bf l}_j^{o \top} ({\bf R}^o{\bf P}_i+{\bf t}^o)=0$ and $\tilde {\bf l}_j^{o \top} ({\bf R}^o{\bf Q}_i+{\bf t}^o)=0$. Thus, we obtain a modified linear line residual:
\begin{equation} \label{line_residual2}
   \bar {\bf r}_{lj} \triangleq 
      \left[ \tilde {\bf l}_j^{\top} (\bar {\bf R} {\bf P}_i+{\bm \tau}) ~~
       \tilde {\bf l}_j^{\top} (\bar {\bf R} {\bf Q}_i+{\bm \tau}) \right]^\top,
\end{equation}
which can be rearranged as 
\begin{equation} \label{line_residual3}
   \bar {\bf r}_{lj} = {\bf A}_{lj} {\bm \varphi}+{\bf B}_{lj} {\bm \tau},
\end{equation}
where ${\bf A}_{lj}=[\bar {\bf P}_j^\top \tilde {\bf l}_j~\bar {\bf Q}_j^\top \tilde {\bf l}_j]^\top$, ${\bf B}_{lj}=[\tilde {\bf l}_j~\tilde {\bf l}_j]^\top$, $\bar{\bf P}_j={\bf P}_j^\top \otimes {\bf I}_3$, and $\bar{\bf Q}_j={\bf Q}_j^\top \otimes {\bf I}_3$.
\begin{remark} \label{remark:different line residuals}
    Here, we adopt the residual in~\eqref{line_residual3}, rather than that in~\eqref{line_residual} used in the ML formulation. This choice yields an expression that is linear in both the rotation and translation parameters, thereby facilitating parameter decoupling and bias elimination. In contrast, the term ${\bf t}^{\wedge}{\bf R}$ in~\eqref{line_residual} couples the rotation and translation, making them difficult to decouple. Nevertheless, the residual in~\eqref{line_residual} is associated with Gaussian noise and is therefore well suited to the ML formulation, whereas the residual in~\eqref{line_residual3} induces non-Gaussian noise, as discussed in Remark~\ref{remark_on_line_noise}.
\end{remark}
By stacking~\eqref{point_residual3} and~\eqref{line_residual3} for all points and lines, we construct the following problem:
\begin{equation} \label{linear_problem}
    \mathop{\rm minimize}\limits_{{\bf g} ,{\bm \tau} } ~ \frac{\|{\bf A} {\bm \varphi}+{\bf B} {\bm \tau}\|^2}{n+m}, 
\end{equation}
where ${\bf A}=[{\bf A}_{p1}^\top \; \cdots \; {\bf A}_{pn}^\top \; {\bf A}_{l1}^\top \; \cdots \; {\bf A}_{lm}^\top]^\top$ and ${\bf B}=[{\bf B}_{p1}^\top \; \cdots \; {\bf B}_{pn}^\top \; {\bf B}_{l1}^\top \; \cdots \; {\bf B}_{lm}^\top]^\top$.
Via ordinary least squares, the optimal ${\bm \tau}$ is given by $-({\bf B}^\top {\bf B})^{-1}{\bf B}^\top{\bf A}{\bm \varphi}$. Let ${\bf C} \triangleq -({\bf B}^\top {\bf B})^{-1}{\bf B}^\top{\bf A}$. By taking the optimal ${\bm \tau}$ into~\eqref{linear_problem} yields an optimization problem only for $\bf g$:
\begin{equation} \label{linear_problem2}
    \mathop{\rm minimize}\limits_{{\bf g} } ~ \frac{\|({\bf A}+{\bf B}{\bf C}) {\bm \varphi}\|^2}{n+m}. 
\end{equation}
Let $\bar {\bf g} \triangleq[g_1^2 \; g_2^2 \; g_3^2 \; g_1g_2 \; g_1g_3 \; g_2g_3 \; g_1 \; g_2 \; g_3 \; 1]^\top$. It holds that ${\bm \varphi}={\bm \Phi}\bar {\bf g}$, where 
\begin{equation*}
    {\bm \Phi}=\begin{bmatrix}
        1 & -1 & -1 & 0 & 0 & 0 & 0 & 0 & 0 & 1 \\
        0 & 0 & 0 & 2 & 0 & 0 & 0 & 0 & 2 & 0 \\
        0 & 0 & 0 & 0 & 2 & 0 & 0 & -2 & 0 & 0 \\
        0 & 0 & 0 & 2 & 0 & 0 & 0 & 0 & -2 & 0 \\
        -1 & 1 & -1 & 0 & 0 & 0 & 0 & 0 & 0 & 1 \\
        0 & 0 & 0 & 0 & 0 & 2 & 2 & 0 & 0 & 0 \\
        0 & 0 & 0 & 0 & 2 & 0 & 0 & 2 & 0 & 0 \\
        0 & 0 & 0 & 0 & 0 & 2 & -2 & 0 & 0 & 0 \\
        -1 & -1 & 1 & 0 & 0 & 0 & 0 & 0 & 0 & 1 \\
    \end{bmatrix}.
\end{equation*}
Problem~\eqref{linear_problem2} is equivalent to 
\begin{equation} \label{linear_problem3}
    \mathop{\rm minimize}\limits_{{\bf g} } ~ f({\bf g}) \triangleq \bar {\bf g}^\top {\bf G} \bar {\bf g} ,
\end{equation}
where ${\bf G}={\bm \Phi}^\top ({\bf A}+{\bf B}{\bf C})^\top ({\bf A}+{\bf B}{\bf C}) {\bm \Phi}/(n+m)$.
To find the global minimum of~\eqref{linear_problem3}, we first obtain all stationary points of $f({\bf g})$, and then select the one that minimizes $f({\bf g})$. Since $f({\bf g})$ is quartic for $\bf g$, finding all its stationary points boils down to solving for all real solutions of the system of cubic equations $\nabla f({\bf g})=0$. By Bézout's theorem~\cite{hartshorne2013algebraic}, this system admits at most $27$ isolated solutions.

Before constructing a consistent estimator, we analyze the global minimizers of~\eqref{linear_problem} and~\eqref{linear_problem3} in the noise-free setting. The resulting characterization provides the theoretical foundation for the proposed estimator and is summarized in the following lemma.
\begin{lemma} \label{lemma:noise-free property}
    Given Assumption~\ref{assump:identifiability}, in the noise-free setting, the ground truth $({\bf g}^o,{\bm \tau}^o)$ is the unique global minimizer of~\eqref{linear_problem}. Moreover, ${\bf g}^o$ is the unique global minimizer of~\eqref{linear_problem3}.
\end{lemma}
The proof of Lemma is presented in Appendix~\ref{noise-free property proof}.

Lemma~\ref{lemma:noise-free property} establishes that, in the noise-free case, the global minimizer of~\eqref{linear_problem3} is the ground truth ${\bf g}^o$, which can be further used to recover ${\bm \tau}^o$ via ${\bf C}^o{\bm \varphi}^o$. However, the noise-free matrices ${\bf G}^o$ and ${\bf C}^o$ are not available in practice. In the following, we construct bias-eliminated matrices ${\bf G}^{\rm BE}$ and ${\bf C}^{\rm BE}$ that converge to their noise-free counterparts ${\bf G}^o$ and ${\bf C}^o$, respectively.
Specifically, because $\bf A$ and $\bf B$ are contaminated by noise, the products ${\bf A}^\top {\bf A}$, ${\bf A}^\top {\bf B}$, and ${\bf B}^\top {\bf B}$ contain quadratic noise terms. These terms generally have non-zero expectations, thereby introducing biases into $\bf G$ and $\bf C$. Hence, the core idea of our bias elimination is to subtract the expectations of corresponding noise terms and let them have zero means.  

Define the following bias terms:
\begin{equation*}
    \Delta {\bf L}_j \triangleq \begin{bmatrix}
        2 \sigma^2_l {\bf I}_2 & -\sigma^2_l ({\bf p}_j+{\bf q}_j) \\
        -\sigma^2_l ({\bf p}_j+{\bf q}_j)^\top & \sigma^2_l (\|{\bf p}_j\|^2+\|{\bf q}_j\|^2) -2 \sigma^4_l
    \end{bmatrix},
\end{equation*}
and 
\begin{align*}
    \Delta_{AA}  \triangleq & 2\sigma^2_p \frac{\sum_{i=1}^{n} \bar{\bf X}_i^\top {\bf e}_3 {\bf e}_3^\top \bar{\bf X}_i}{n+m} \\
    & + \frac{\sum_{j=1}^{m} \bar{\bf P}_j^\top \Delta {\bf L}_j \bar{\bf P}_j+\bar{\bf Q}_j^\top \Delta {\bf L}_j \bar{\bf Q}_j}{n+m}, \\
    \Delta_{AB}  \triangleq & 2\sigma^2_p \frac{\sum_{i=1}^{n} \bar{\bf X}_i^\top {\bf e}_3 {\bf e}_3^\top}{n+m} + \frac{\sum_{j=1}^{m} \bar{\bf P}_j^\top \Delta {\bf L}_j +\bar{\bf Q}_j^\top \Delta {\bf L}_j }{n+m}, \\
    \Delta_{BB}  \triangleq & \frac{2n \sigma^2_p {\bf e}_3 {\bf e}_3^\top}{n+m} + \frac{\sum_{j=1}^{m} 2 \Delta {\bf L}_j }{n+m}.
\end{align*}
\begin{lemma} \label{lemma:bias elimination}
    The following bias-eliminated matrices satisfy 
    \begin{align*}
         {\bf M}_{AA} & \triangleq \frac{{\bf A}^\top {\bf A}}{n+m}-\Delta_{AA}=\frac{{\bf A}^{o \top} {\bf A}^o}{n+m} + O_p(1/\sqrt{n+m}), \\
         {\bf M}_{AB} & \triangleq \frac{{\bf A}^\top {\bf B}}{n+m}-\Delta_{AB}=\frac{{\bf A}^{o \top} {\bf B}^o}{n+m} + O_p(1/\sqrt{n+m}), \\
         {\bf M}_{BB} & \triangleq \frac{{\bf B}^\top {\bf B}}{n+m}-\Delta_{BB}=\frac{{\bf B}^{o \top} {\bf B}^o}{n+m} + O_p(1/\sqrt{n+m}).       
    \end{align*}
\end{lemma}
The proof of Lemma~\ref{lemma:bias elimination} is presented in Appendix~\ref{bias-elimination proof}.

Based on Lemma~\ref{lemma:bias elimination}, we construct the bias-eliminated matrices ${\bf C}^{\rm BE}$ and ${\bf G}^{\rm BE}$ as follows:
\begin{equation}
\label{eq:G_BE_construction}
\begin{alignedat}{2}
\mathbf{C}^{\mathrm{BE}}
&\triangleq
-\mathbf{M}_{BB}^{-1}\mathbf{M}_{AB}^{\top},
\\
\mathbf{G}^{\mathrm{BE}}
&\triangleq
\boldsymbol{\Phi}^{\top}
\Bigl(
\mathbf{M}_{AA}
+\mathbf{M}_{AB}\mathbf{C}^{\mathrm{BE}}
&&+\mathbf{C}^{\mathrm{BE}\top}
  \mathbf{M}_{AB}^{\top}
\\
&&&+\mathbf{C}^{\mathrm{BE}\top}
  \mathbf{M}_{BB}\mathbf{C}^{\mathrm{BE}}
\Bigr)\boldsymbol{\Phi}.
\end{alignedat}
\end{equation}
Then, we solve the problem
\begin{equation} \label{linear_problem4}
    \mathop{\rm minimize}\limits_{{\bf g} } ~ \bar {\bf g}^\top \mathbf{G}^{\mathrm{BE}} \bar {\bf g} ,
\end{equation}
and obtain $\hat {\bf g}^{\rm BE}$, which further gives $\hat {\bm \varphi}^{\rm BE}$ and $\hat {\bm \tau}^{\rm BE}=\mathbf{C}^{\mathrm{BE}} \hat {\bm \varphi}^{\rm BE}$. Finally, we can recover the bias-eliminated pose estimate $(\hat {\bf R}^{\rm BE},\hat {\bf t}^{\rm BE})$ from $\hat {\bf g}^{\rm BE}$ and $\hat {\bm \tau}^{\rm BE}$.

\begin{theorem} \label{theorem:consistent_estimator}
    The bias-eliminated pose $(\hat {\bf R}^{\rm BE},\hat {\bf t}^{\rm BE})$ is a $\sqrt{n+m}$-consistent estimate of $({\bf R}^o,{\bf t}^o)$, i.e., 
    $$
    \hat {\bf R}^{\rm BE}={\bf R}^o+O_p(1/\sqrt{n+m}),~\hat {\bf t}^{\rm BE}={\bf t}^o+O_p(1/\sqrt{n+m}).
    $$
\end{theorem}
The proof of Theorem~\ref{theorem:consistent_estimator} is presented in Appendix~\ref{proof:consistent_estimator}.

Theorem~\ref{theorem:consistent_estimator} has two implications: first, $(\hat {\bf R}^{\rm BE},\hat {\bf t}^{\rm BE})$ converges to the true value $({\bf R}^o,{\bf t}^o)$ as the feature number $n+m$ increases; second, the convergence speed is $1/\sqrt{n+m}$, i.e., it has an estimation error of order $1/\sqrt{n+m}$. Consequently, its root mean square error (RMSE) decreases proportionally to $1/\sqrt{n+m}$.  
As will be illustrated in the next section, it serves as a good initial value for the subsequent local iterations. A straightforward corollary from Theorem~\ref{theorem:consistent_estimator}, obtained using the algebraic properties of $O_p$, is stated below.
\begin{corollary}
    Suppose that $\hat \sigma_p^2$ and $\hat \sigma_l^2$ are $\sqrt{n}$- and $\sqrt{m}$-consistent estimates of $\sigma_p^2$ and $\sigma_l^2$, respectively. Replacing $\sigma_p^2$ and $\sigma_l^2$ with these estimates also yields a $\sqrt{n+m}$-consistent pose estimate $(\hat {\bf R}^{\rm BE},\hat {\bf t}^{\rm BE})$.
\end{corollary}

Next, we will end this section by detailing how to obtain the consistent estimates of noise variances. 
A $\sqrt{n}$-consistent estimate of point noise variance $\sigma^2_p$ has been given in~\cite{zeng2023cpnp}, hence we only derive a $\sqrt{m}$-consistent estimate of line noise variance $\sigma^2_l$ here. 
Define ${\bm \theta}^o \triangleq {\rm vec}([{\bf R}^o\; {\bf t}^{o \wedge}{\bf R}^o])$. By combining~\eqref{modified_line_measurement} and $\bar {\bf l}_j^o = \left[{\bf R}^o~~{\bf t}^{o\wedge} {\bf R}^o \right] {\bf L}_j$, we construct a matrix-form equation for $m$ line correspondences as follows:
\begin{equation} \label{line_geometry_matrix_form}
    0={\bf D} {\bm \theta}^o-\bm{\varepsilon}_{l},
\end{equation}
where 
\begin{equation*}
    {\bf D} = \begin{bmatrix}
        [{\bf p}_1^h~{\bf q}_1^h]^\top \overbar {\bf L}_1 \\
        \vdots \\
        [{\bf p}_m^h~{\bf q}_m^h]^\top \overbar {\bf L}_m
    \end{bmatrix},~~\bm{\varepsilon}_{l}=\begin{bmatrix}
        [\bm{\epsilon}_{{\bf p}_1}^{h}~\bm{\epsilon}_{{\bf q}_1}^{h}]^\top \bar {\bf l}_1^{o} \\
        \vdots \\
        [\bm{\epsilon}_{{\bf p}_m}^{h}~\bm{\epsilon}_{{\bf q}_m}^{h}]^\top \bar {\bf l}_m^{o}
    \end{bmatrix},
\end{equation*}
and $\overbar {\bf L}_j={\bf L}_j^\top \otimes {\bf I}_3$. 
Define 
\begin{equation} \label{eqn:H_definition}
    {\bf H} \triangleq {\bf D}^{\top} {\bf D}/m
\end{equation}
and its noise-free counterpart ${\bf H}^o \triangleq {\bf D}^{o \top} {\bf D}^o/m$. 
Let ${\bf L} \triangleq [{\bf L}_1~\cdots~{\bf L}_m]$, $\tilde{\bf D}_{1} \triangleq {\bf L}^{\top} \otimes [{\bf e}_1~~{\bf e}_1]^\top$, $\tilde{\bf D}_{2} \triangleq {\bf L}^{\top} \otimes [{\bf e}_2~~{\bf e}_2]^\top$, and 
\begin{equation} \label{eqn:H1_H2_definition}
    \tilde{\bf H}_{1} \triangleq \tilde{\bf D}_{1}^{\top} \tilde{\bf D}_{1}/m,~\tilde{\bf H}_{2} \triangleq \tilde{\bf D}_{2}^{\top} \tilde{\bf D}_{2}/m.
\end{equation}
\begin{theorem} \label{theorem:line_consistent_noise_estimation}
    Let $\hat \sigma^2_l \triangleq 1/\lambda_{\rm max}({\bf H}^{-1}(\tilde{\bf H}_{1}+\tilde{\bf H}_{2}))$. Then, $\hat \sigma^2_l$ is a $\sqrt{m}$-consistent estimate of $\sigma^2_l$.
\end{theorem}
The proof of Theorem~\ref{theorem:line_consistent_noise_estimation} is presented in Appendix~\ref{proof:line noise est}.
\begin{remark} \label{remark:noise_variance_estimation}
    When the numbers of point and line correspondences are small, the noise variance estimates may be unreliable, and bias elimination may consequently degrade pose estimation accuracy. Therefore, in practice, we apply bias elimination and replace $\mathbf{G}$ with $\mathbf{G}^{\mathrm{BE}}$ only when the number of point correspondences or the number of line correspondences exceeds $20$. Otherwise, we retain $\mathbf{G}$ without bias elimination. In this small-sample regime, we set $\hat{\sigma}_p^2=\hat{\sigma}_l^2=1$ when constructing the weighting matrices for the subsequent GN iterations described in Section~\ref{asymptotic_property_analysis}.
\end{remark}

\begin{remark} \label{remark:metric_noise_est}
    Unlike some existing methods that estimate image-feature uncertainty only up to an unknown scale factor~\cite{tirado2024dac}, our method consistently estimates the metric noise variance (e.g., in squared pixels), which is essential for eliminating estimation bias. Notably, accurate metric uncertainty quantification is particularly important for localization and mapping in safety-critical robotic applications and has attracted increasing attention in recent studies~\cite{gao2024closure,zeng2026provably}. Our consistent metric variance estimation method can provide a principled foundation for such uncertainty quantification.
\end{remark}

\section{Asymptotically efficient solution} \label{asymptotic_property_analysis}
In the previous section, we have proposed a consistent estimate $(\hat {\bf R}^{\rm BE},\hat {\bf t}^{\rm BE})$ that converges to the true pose. It is noteworthy that although $(\hat {\bf R}^{\rm BE},\hat {\bf t}^{\rm BE})$ is consistent, it is not asymptotically efficient in the sense that its covariance cannot asymptotically reach the theoretical lower bound CRB. This is because till now, we have just solved a problem relaxed from the ML problem~\eqref{ML_problem}, which is not equivalent to~\eqref{ML_problem}. In this section, we take $(\hat {\bf R}^{\rm BE},\hat {\bf t}^{\rm BE})$ as the initial value and perform the GN algorithm to refine it. After that, we can obtain an asymptotically efficient solution.

Recall that the ML estimate $(\hat {\bf R}^{\rm ML},\hat {\bf t}^{\rm ML})$ is also consistent. Hence, $(\hat {\bf R}^{\rm BE},\hat {\bf t}^{\rm BE})$ converges to $(\hat {\bf R}^{\rm ML},\hat {\bf t}^{\rm ML})$ with the increase of $n+m$. As a result, when $n+m$ is large enough, $(\hat {\bf R}^{\rm BE},\hat {\bf t}^{\rm BE})$ will fall into the attraction neighborhood of $(\hat {\bf R}^{\rm ML},\hat {\bf t}^{\rm ML})$, and local iteration methods, e.g., the GN algorithm can find $(\hat {\bf R}^{\rm ML},\hat {\bf t}^{\rm ML})$ rather than a local minimum. 
In addition, the GN algorithm can be viewed as an approximation of Newton's method. In the large sample case where $(\hat {\bf R}^{\rm BE},\hat {\bf t}^{\rm BE})$ is sufficiently close to $(\hat {\bf R}^{\rm ML},\hat {\bf t}^{\rm ML})$, its convergence rate can approach quadratic as the Newton's method~\cite{nocedal1999numerical}.
Thanks to the quadratic convergence rate, a single GN iteration is enough to attain the same asymptotic property as the ML estimate, see Theorem~\ref{asymptotic_efficiency_theorem}.
Before giving the theorem, we illustrate the execution of the one-step GN iteration. 
There are two issues that need to be elaborately treated: First, the weighting matrices ${\bm \Sigma}_{lj}$'s in the ML problem~\eqref{ML_problem} depend on the unknown true pose; Second, the rotation matrix needs to meet the ${\rm SO}(3)$ constraint. 

As illustrated in Section~\ref{GN_manifold}, to satisfy the ${\rm SO}(3)$ constraint, we can use the \emph{lifting} technique in~\eqref{general_so3_optimization2}. Specifically, the residual for the $i$-th point correspondence is reparameterized as 
\begin{equation*}
    {\bf r}_{pi}({\bf s},{\bf t})={\bf x}_i-\frac{{\bf E}(\hat {\bf R}^{\rm BE}{\rm exp}({\bf s}^\wedge) {\bf X}_i+{\bf t})}{{\bf e}_3^\top(\hat {\bf R}^{\rm BE}{\rm exp}({\bf s}^\wedge) {\bf X}_i+{\bf t})},
\end{equation*}
where ${\bf s} \in \mathbb R^3$. Similarly, the residual for the $j$-th line correspondence is reparameterized as
\begin{equation*}
    {\bf r}_{lj}({\bf s},{\bf t})=\begin{bmatrix}
        {\bf p}_j^{h \top} \\
        {\bf q}_j^{h \top}
    \end{bmatrix} \left[\hat {\bf R}^{\rm BE}{\rm exp}({\bf s}^\wedge)~~{\bf t}^{\wedge} \hat {\bf R}^{\rm BE}{\rm exp}({\bf s}^\wedge) \right] {\bf L}_j.
\end{equation*}

For the weighting matrices, we use $\hat {\bm \Sigma}_p=\hat \sigma_p^2 {\bf I}_2$ and $\hat {\bm \Sigma}_{lj}=\hat \sigma_{lj}^2 ({\bf s},{\bf t}) {\bf I}_2$, where $\hat \sigma_{lj}^2 ({\bf s},{\bf t})=[\hat {\bar {{\bf l}}}_j]_{1:2}^\top \hat {\bm \Sigma}_l [\hat {\bar {{\bf l}}}_j]_{1:2}$, $\hat {\bm \Sigma}_l=\hat \sigma_l^2 {\bf I}_2$, and $\hat {\bar {{\bf l}}}_j=[\hat {\bf R}^{\rm BE}{\rm exp}({\bf s}^\wedge)~~{{\bf t}}^\wedge \hat {\bf R}^{\rm BE} {\rm exp}({\bf s}^\wedge)] {\bf L}_j$.
After the above reparameterization, the standard GN algorithm over Euclidean space can be applied. The initial values are set as ${\bf s}=0$ and ${\bf t}=\hat {\bf t}^{\rm BE}$. Thanks to the $\sqrt{n+m}$-consistency of the initial values, only a single GN iteration suffices to reach asymptotic efficiency. Denote the refined values via a single GN iteration as $\hat {\bf t}^{\rm GN}$ and $\hat {\bf s}^{\rm GN}$.
The GN-refined rotation matrix estimate can be obtained via the \emph{retraction} map:
\begin{equation} \label{GN_retraction}
    \hat{{\bf R}}^{\rm GN}=\hat{{\bf R}}^{\rm BE}{\rm exp}\left(\left( {{}\hat{{\bf s}}^{\rm GN}} \right)^{\wedge} \right).
\end{equation}

The asymptotic property of the GN-refined estimate is given in the following theorem.
\begin{theorem}[{\citealp[Theorem~4.3]{lehmann2006theory}},~{\citealp[Theorem~2]{mu2017globally}}] \label{asymptotic_efficiency_theorem}
    Given the $\sqrt{n+m}$-consistency of the initial estimate $(\hat {\bf R}^{\rm BE}~\hat {\bf t}^{\rm BE})$, the GN-refined estimate $(\hat {\bf R}^{\rm GN}~\hat {\bf t}^{\rm GN})$ follows 
    \begin{equation*}
        \hat {\bf R}^{\rm GN}-\hat {\bf R}^{\rm ML}=o_p(\frac{1}{\sqrt{n+m}}),\hat {\bf t}^{\rm GN}-\hat {\bf t}^{\rm ML}=o_p(\frac{1}{\sqrt{n+m}}).
    \end{equation*}
\end{theorem}
The $o_p$ relationship establishes the asymptotic equivalence of the one-step GN and ML estimates. Consequently, the GN estimate inherits the consistency and asymptotic efficiency of the ML estimate~\cite{mu2017globally,zeng2022global,zeng2023consistent}. In particular, $(\hat{\bf R}^{\rm GN}~\hat{\bf t}^{\rm GN})$ asymptotically attains the CRB, as does $(\hat{\bf R}^{\rm ML}~\hat{\bf t}^{\rm ML})$.

We conclude this section by analyzing the computational complexity of the proposed algorithm, as computational efficiency is essential for real-time robotic applications. Preprocessing, the construction of $\mathbf{G}^{\mathrm{BE}}$ in~\eqref{eq:G_BE_construction}, and the one-step GN update each require $O(n+m)$ time. Estimating $\hat{\sigma}_p^2$ requires $O(n)$ time~\cite{zeng2023cpnp}, while constructing $\mathbf{H}$ in~\eqref{eqn:H_definition} and $\tilde{\mathbf{H}}_1$ and $\tilde{\mathbf{H}}_2$ in~\eqref{eqn:H1_H2_definition} requires $O(m)$ time. All remaining procedures---including estimating $\hat{\sigma}_l^2$ using Theorem~\ref{theorem:line_consistent_noise_estimation}, solving Problem~\eqref{linear_problem4}, and recovering $(\hat{\mathbf{R}}^{\mathrm{BE}},\hat{\mathbf{t}}^{\mathrm{BE}})$ from $\hat{\mathbf{g}}^{\mathrm{BE}}$---have constant time complexity, $O(1)$. Therefore, the overall time complexity is $O(n+m)$, meaning that the computational cost scales linearly with the total number of point and line correspondences.

\section{Experiments} \label{experiments}
In this section, we compare our algorithms, referred to as \texttt{FUSE-PnP} (for only point correspondences), \texttt{FUSE-PnL} (for only line correspondences), and \texttt{FUSE-PnPL} (for combined point and line correspondences), with some well-known and state-of-the-art camera pose solvers, including PnP algorithms (\texttt{EPnP}~\cite{lepetit2009epnp}, \texttt{MLPnP}~\cite{urban2016mlpnp}, \texttt{DLS}~\cite{hesch2011direct}, \texttt{SQPnP}~\cite{terzakis2020consistently}, \texttt{CPnP}~\cite{zeng2023cpnp}), PnL algorithms (\texttt{ASPnL}~\cite{xu2016pose}, \texttt{SRPnL}~\cite{wang2019camera}, \texttt{Ro\_PnL}~\cite{liu2020globally}, \texttt{AlgLS}~\cite{mirzaei2011globally}), and PnPL algorithms (\texttt{EPnPL}~\cite{vakhitov2016accurate}, \texttt{OPnPL}~\cite{vakhitov2016accurate}, \texttt{EPnPLU}~\cite{vakhitov2021uncertainty}, \texttt{DLSLU}~\cite{vakhitov2021uncertainty}). 
The algorithms included in the comparison are implemented using publicly available open-source code.

\begin{figure}[t]
	\centering
	\subfloat[{$\sigma=5$ pixels}\label{noise_mse_5px}]{%
\begin{minipage}[b]{0.24\textwidth}
		\centering
		\includegraphics[width=\textwidth]{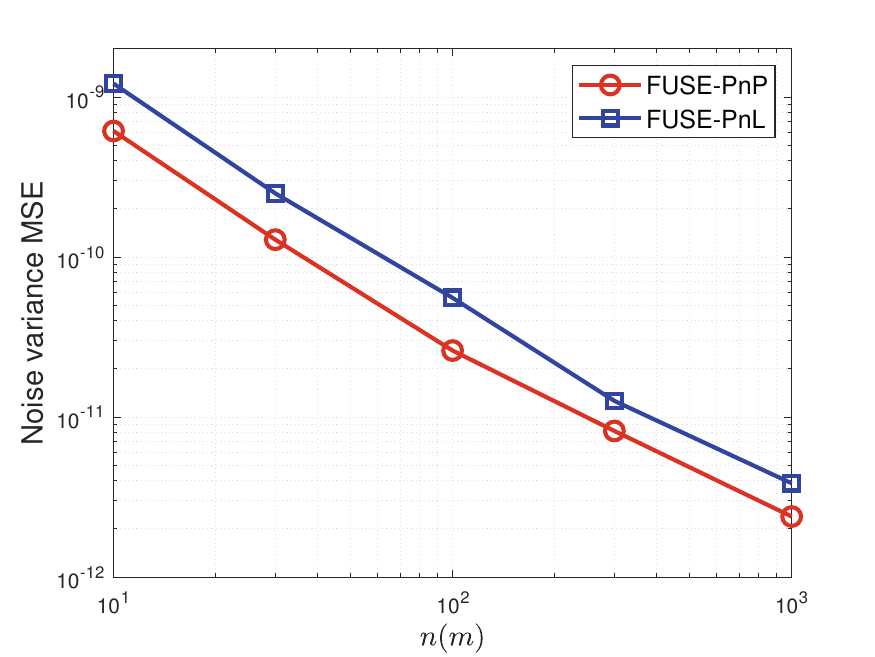}

	\end{minipage}}
	\subfloat[{$\sigma=10$ pixels}\label{noise_mse_10px}]{%
\begin{minipage}[b]{0.24\textwidth}
		\centering
		\includegraphics[width=\textwidth]{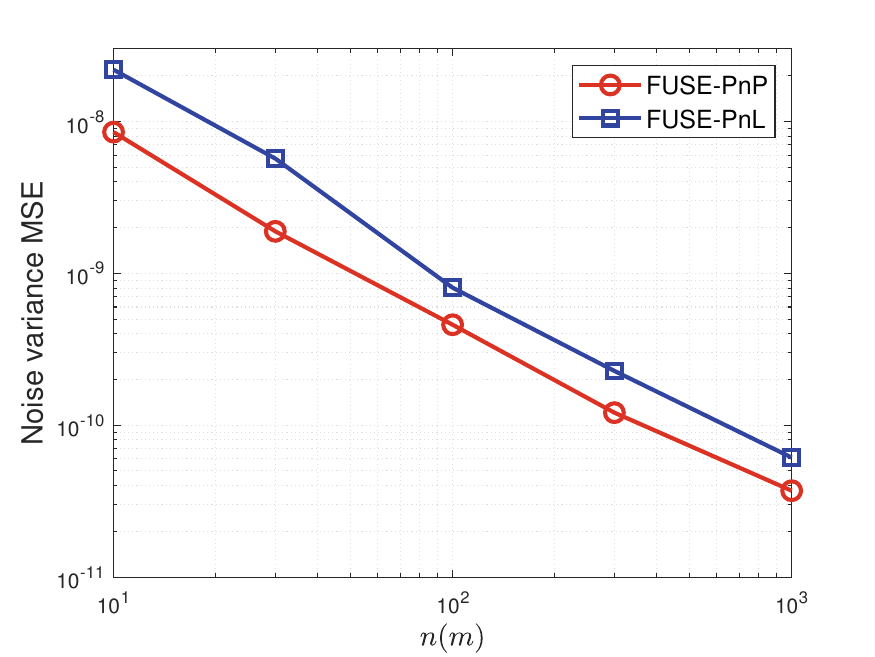}

	\end{minipage}}
	\caption{MSEs of noise variance estimates.}
	\label{noise_mse}
\end{figure}

\begin{figure*}[!htbp]
	\centering
	\subfloat[{$\sigma=5$ pixels ($\bf R$)}\label{pnp_R_mse_5px}]{%
\begin{minipage}[b]{0.24\textwidth}
		\centering
		\includegraphics[width=\textwidth]{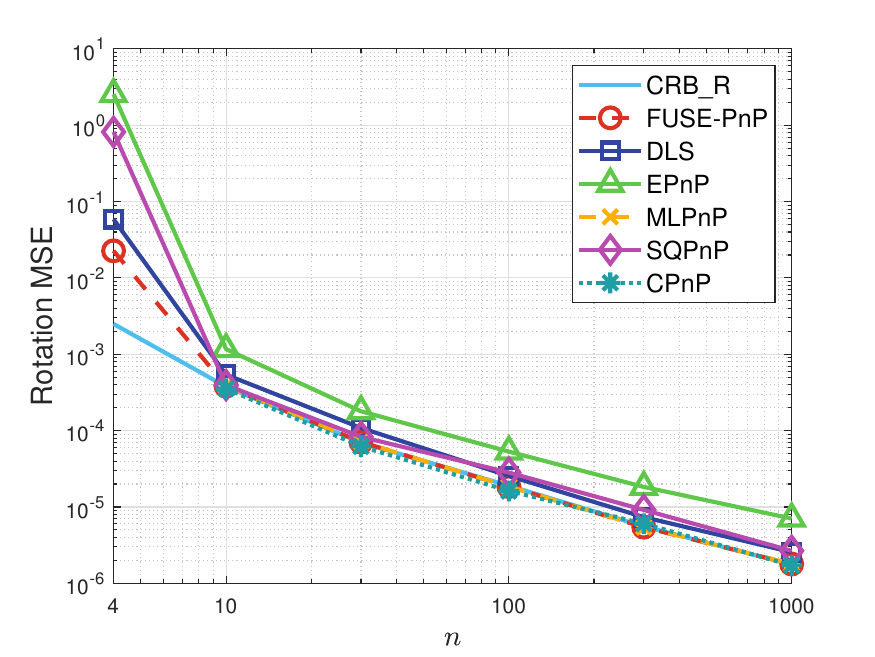}

	\end{minipage}}
	\subfloat[{$\sigma=5$ pixels ($\bf t$)}\label{pnp_t_mse_5px}]{%
\begin{minipage}[b]{0.24\textwidth}
		\centering
		\includegraphics[width=\textwidth]{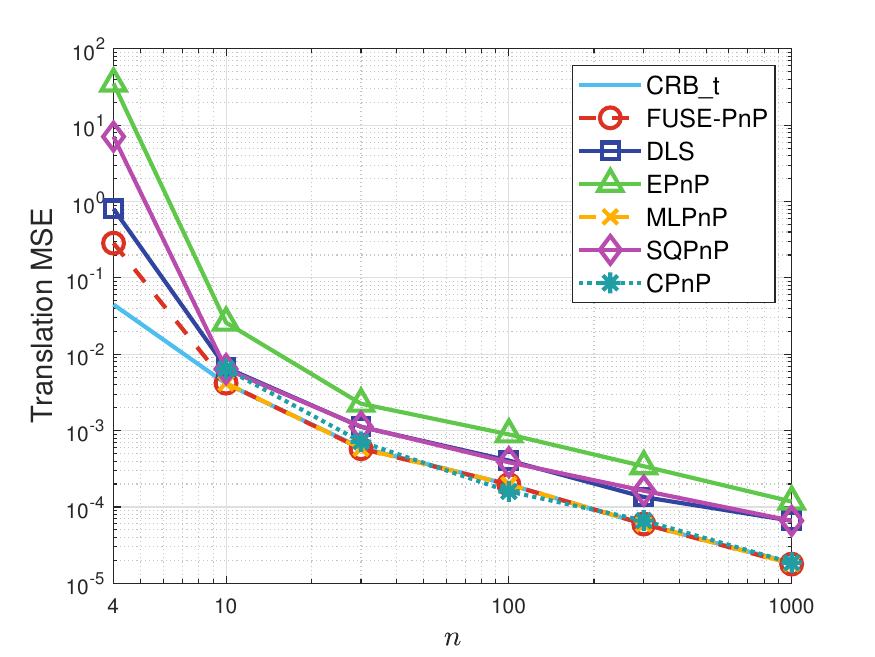}

	\end{minipage}}
 	\subfloat[{$\sigma=10$ pixels ($\bf R$)}\label{pnp_R_mse_10px}]{%
\begin{minipage}[b]{0.24\textwidth}
		\centering
		\includegraphics[width=\textwidth]{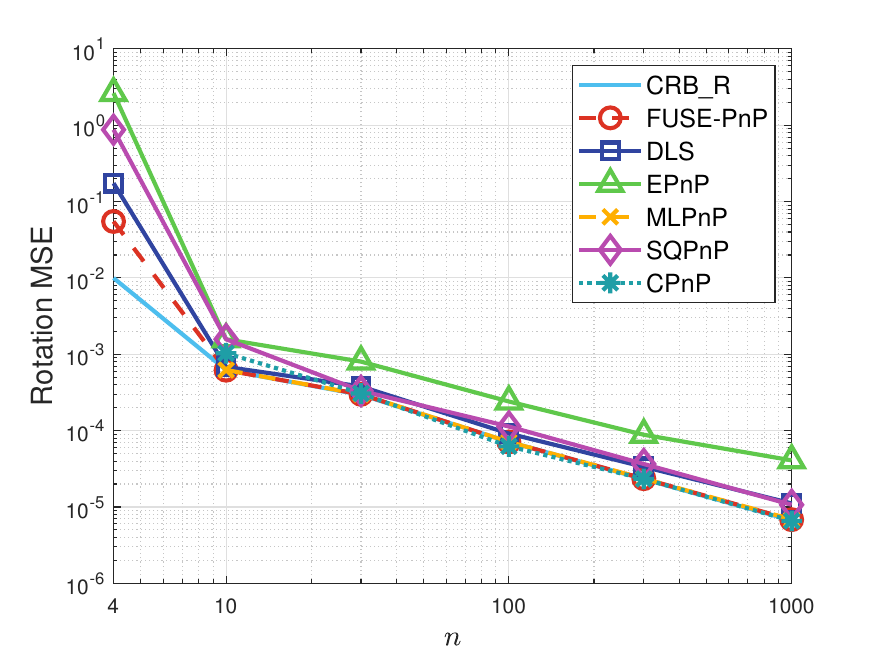}

	\end{minipage}}
	\subfloat[{$\sigma=10$ pixels ($\bf t$)}\label{pnp_t_mse_10px}]{%
\begin{minipage}[b]{0.24\textwidth}
		\centering
		\includegraphics[width=\textwidth]{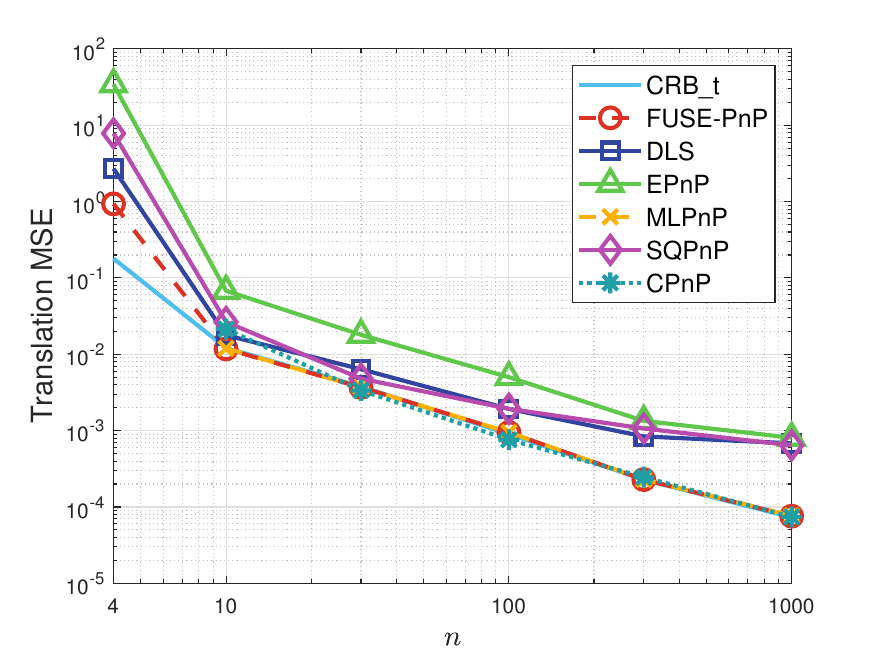}

	\end{minipage}}
	\caption{MSEs of PnP pose estimates.}
	\label{pnp_mse}
\end{figure*}

We conduct three kinds of experiments: In Subsection~\ref{simulations}, we conduct simulations with synthetic data to validate our theoretical claims and evaluate the proposed estimator under challenging and limiting configurations. In Subsection~\ref{static_real_image_tests}, we conduct static real-image experiments using the ETH3D~\cite{schops2017multi} and VGG~\cite{werner2002new} datasets. These experiments evaluate absolute pose estimation for robots operating in known environments. In Subsection~\ref{visual_odometry}, we conduct visual odometry experiments using the EuRoC MAV dataset~\cite{euroc} and our self-collected underwater dataset. These experiments assess the applicability of the proposed method to robot localization and navigation in unknown environments.

\subsection{Simulations with synthetic data} \label{simulations}
Throughout the simulations, the focal length is set as $f=50$mm ($800$ pixels), and the size of the image plane is $640 \times 480$ pixels. The intrinsic matrix is given by
\begin{equation*}
{\bf K} = \begin{bmatrix}
800 & 0 & 320 \\
0 & 800 & 240 \\
0 & 0 & 1
\end{bmatrix}~{\rm pixels}.
\end{equation*}
For data generation, we first randomly generate $n$ points and $2m$ endpoints of $m$ lines in the image plane. Then, these points are endowed with a random depth between $[2,10]$m, yielding their 3D coordinates in the camera frame. Further, the coordinates in the camera frame are transformed into the world frame using the randomly generated true pose value, by which $n$ 3D points and $m$ 3D lines have been generated. Finally, $n$ 2D points and $2m$ 2D endpoints in the image plane are corrupted by additive Gaussian noises.

In terms of estimation accuracy, we compare the mean squared error (MSE) of each estimator. The MSE is approximated as
\begin{align*}
    {\rm MSE}(\hat {\bf R}) & = \frac{1}{K}\sum\limits_{k=1}^{K} {\|\hat {\bf R}(\omega_k)-{\bf R}_k^o\|}_{\rm F}^2, \\
    {\rm MSE}(\hat {\bf t}) & = \frac{1}{K}\sum\limits_{k=1}^{K} {\|\hat {\bf t}(\omega_k)-{\bf t}_k^o\|}^2,
\end{align*}
where $K$ is the number of Monte Carlo tests, and $(\hat {\bf R}(\omega_k),\hat {\bf t}(\omega_k))$ is the estimated pose in the $k$-th test. Throughout the simulations, we set $K=1000$.

\begin{figure}[!t]
	\centering
	\begin{minipage}[b]{0.22\textwidth}
		\centering
		\includegraphics[width=\textwidth]{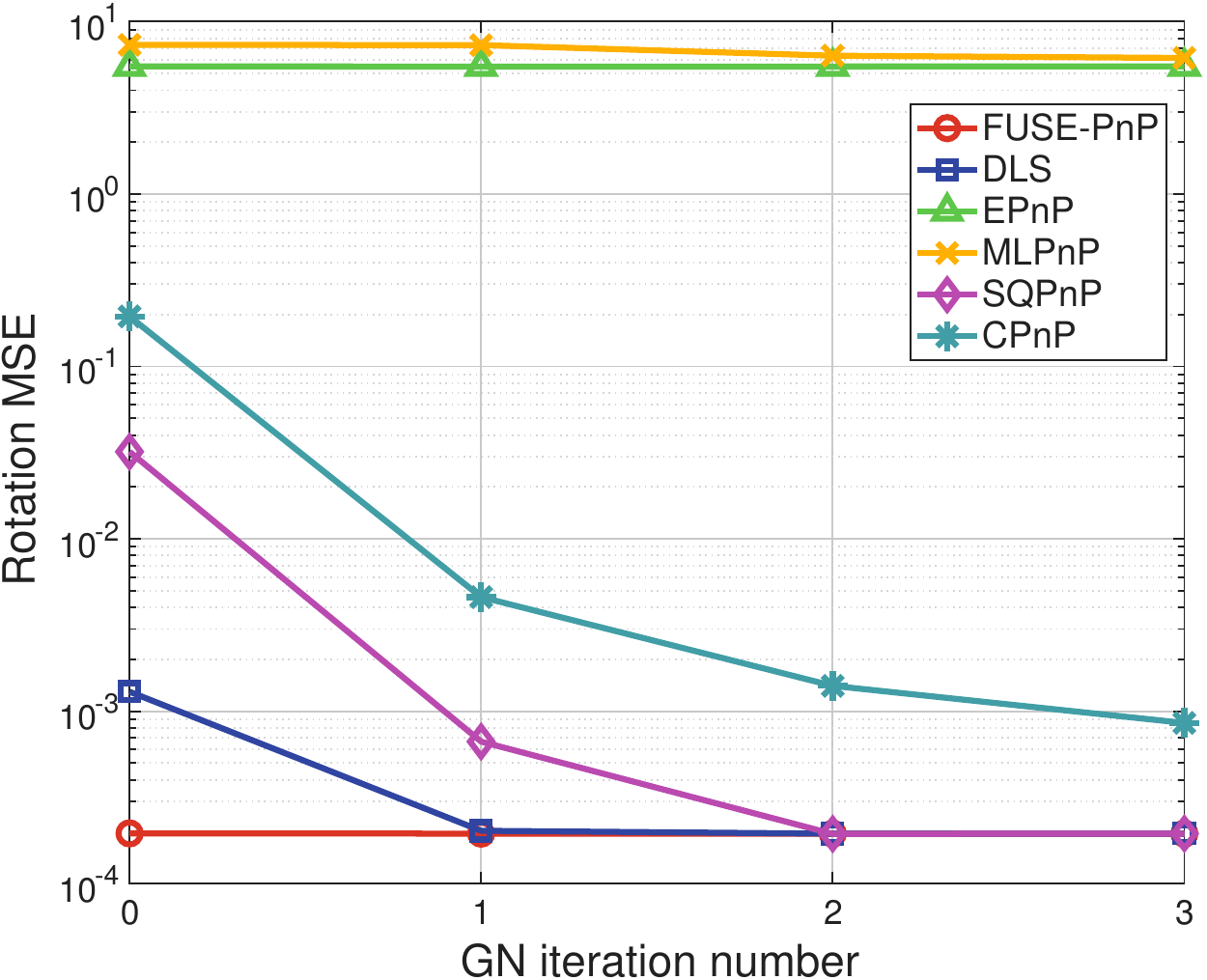}
		\label{pnp_GN_number_rotation}
	\end{minipage}
    \hspace{2mm}
	\begin{minipage}[b]{0.22\textwidth}
		\centering
		\includegraphics[width=\textwidth]{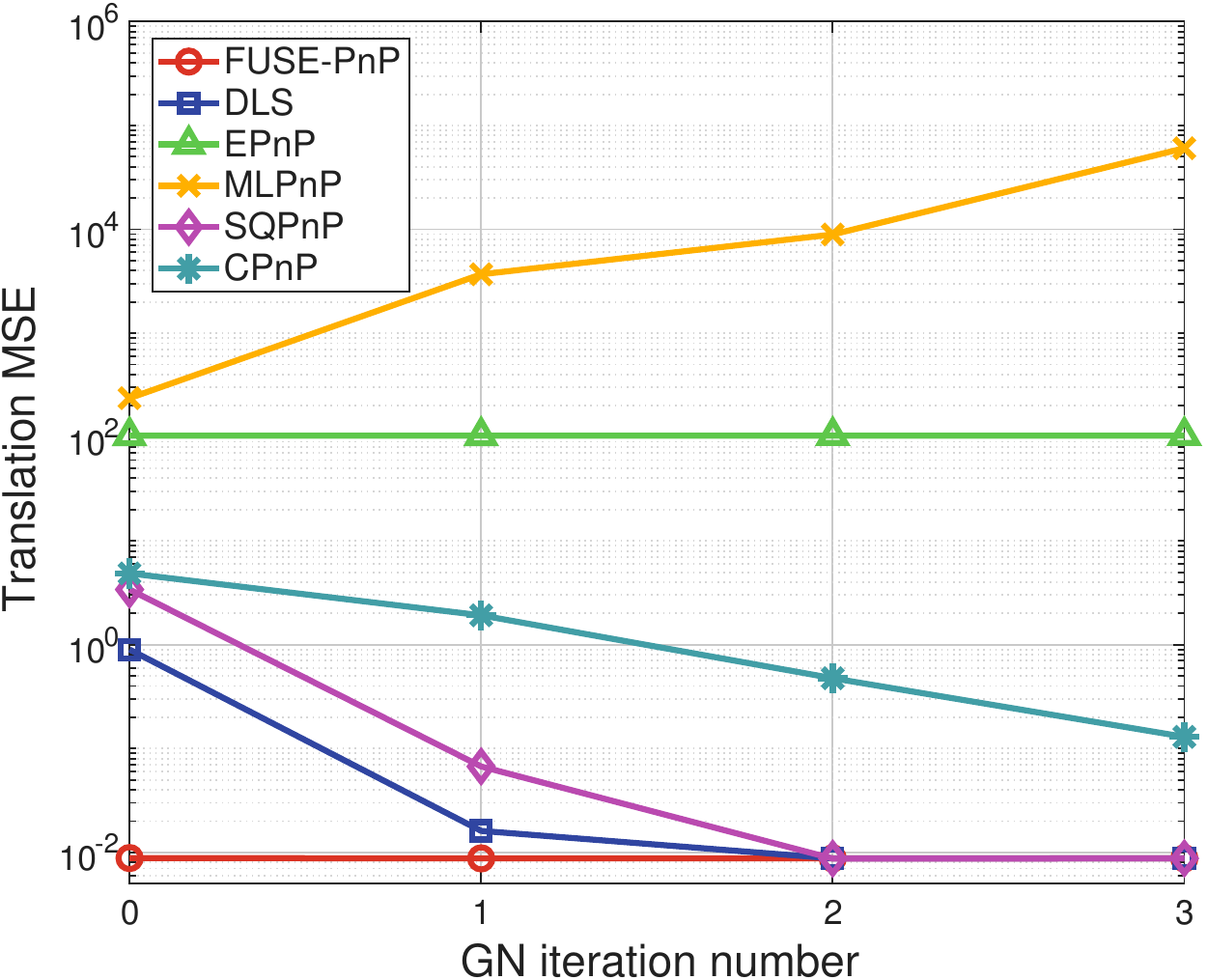}
		\label{pnp_GN_number_translation}
	\end{minipage}
    \vspace{-4mm}
	\caption{MSEs versus GN iteration number under a challenging configuration.}
	\label{pnp_GN_number}
\end{figure}

\begin{figure*}[!htbp]
	\centering
	\subfloat[{$\sigma=5$ pixels ($\bf R$)}\label{pnl_R_mse_5px}]{%
\begin{minipage}[b]{0.24\textwidth}
		\centering
		\includegraphics[width=\textwidth]{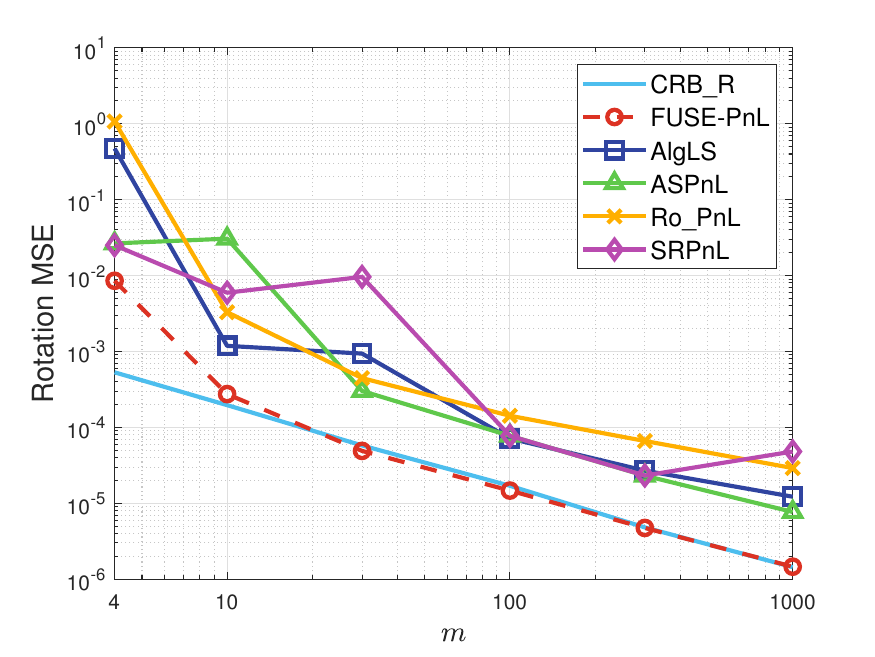}

	\end{minipage}}
	\subfloat[{$\sigma=5$ pixels ($\bf t$)}\label{pnl_t_mse_5px}]{%
\begin{minipage}[b]{0.24\textwidth}
		\centering
		\includegraphics[width=\textwidth]{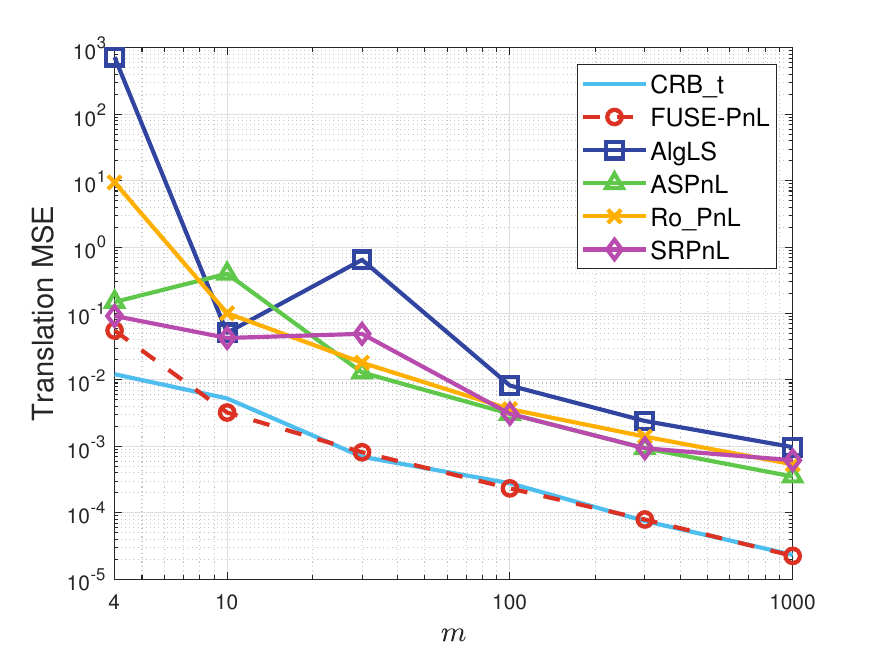}

	\end{minipage}}
 	\subfloat[{$\sigma=10$ pixels ($\bf R$)}\label{pnl_R_mse_10px}]{%
\begin{minipage}[b]{0.24\textwidth}
		\centering
		\includegraphics[width=\textwidth]{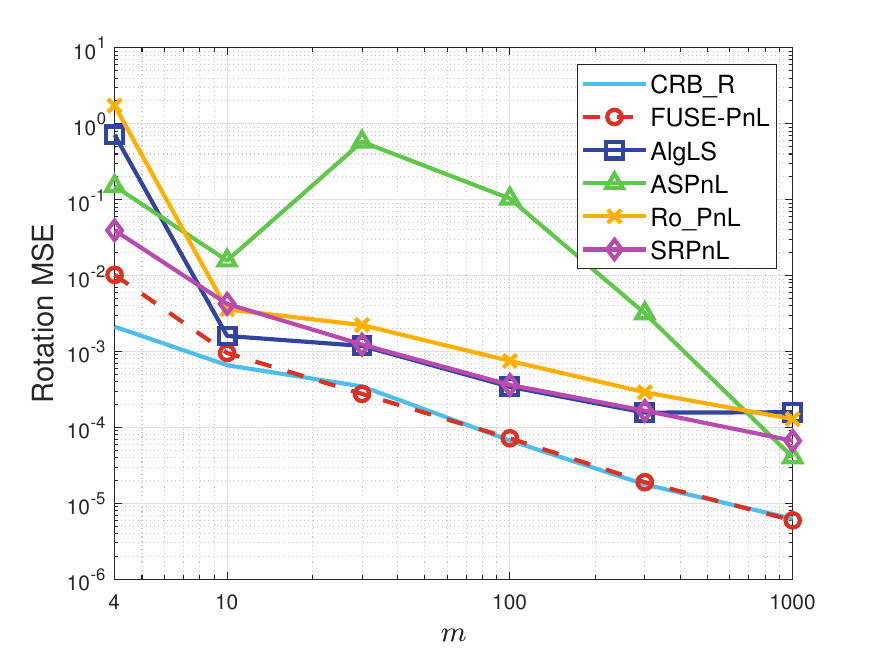}

	\end{minipage}}
	\subfloat[{$\sigma=10$ pixels ($\bf t$)}\label{pnl_t_mse_10px}]{%
\begin{minipage}[b]{0.24\textwidth}
		\centering
		\includegraphics[width=\textwidth]{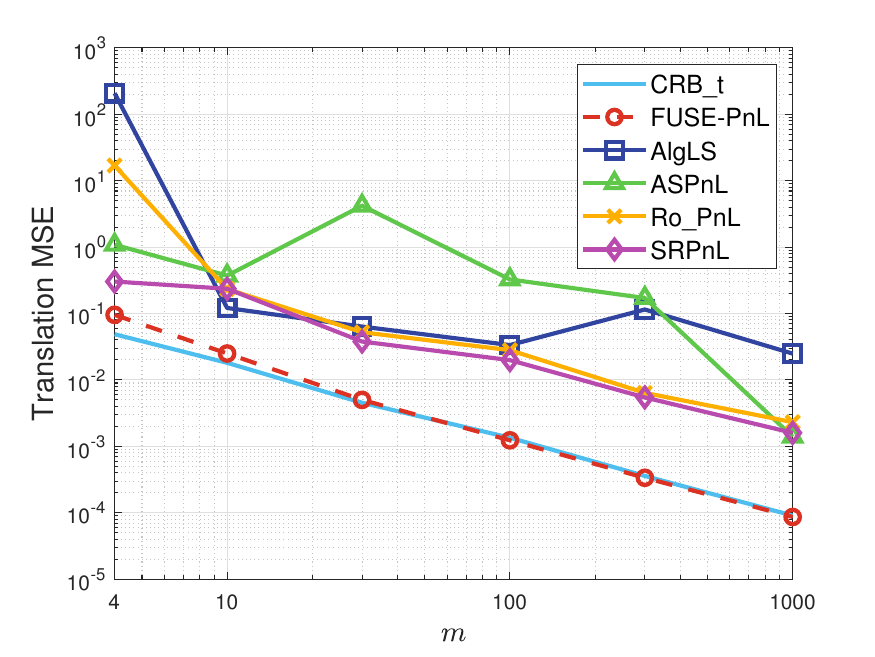}

	\end{minipage}}
	\caption{MSEs of PnL pose estimates.}
	\label{pnl_mse}
\end{figure*}

\emph{(1) Consistency of noise variance estimators.} 
To devise consistent pose estimators, we first need consistent estimators for measurement noise variances. 
We let $n=m=10,30,100,300,1000$ and $\sigma=5,10$ pixels.
Figure~\ref{noise_mse} presents the MSEs of noise variance estimates, where double logarithmic axes have been used. We see that the MSEs decline linearly w.r.t. the measurement number, which validates the $\sqrt{n}$- and $\sqrt{m}$-consistency of the point and line noise variance estimates, respectively. The point-based estimator yields lower MSEs than the line-based estimator, which can be attributed to its simpler parameterization involving fewer parameters~\cite{zeng2023cpnp}.

\begin{figure*}[!htbp]
	\centering
	\subfloat[{$\sigma=5$ pixels ($\bf R$)}\label{pnpl_R_mse_5px}]{%
\begin{minipage}[b]{0.24\textwidth}
		\centering
		\includegraphics[width=\textwidth]{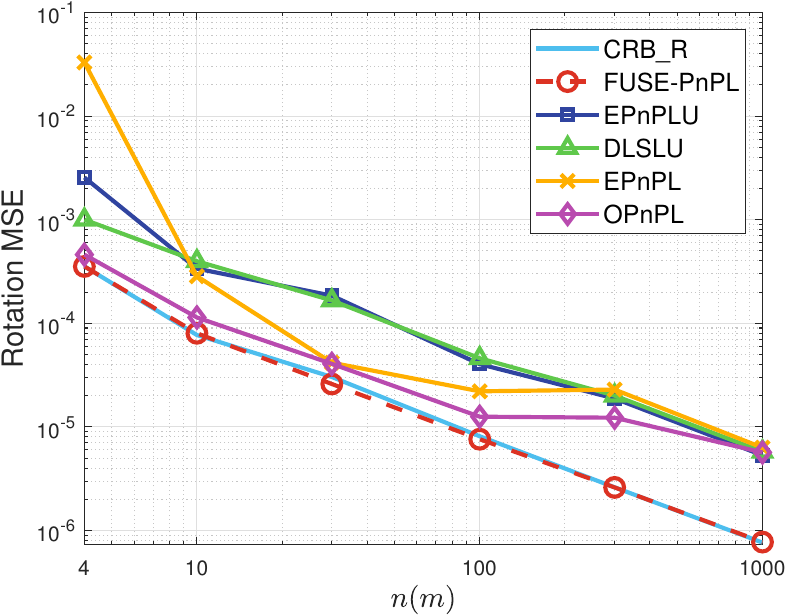}

	\end{minipage}}
	\subfloat[{$\sigma=5$ pixels ($\bf t$)}\label{pnpl_t_mse_5px}]{%
\begin{minipage}[b]{0.24\textwidth}
		\centering
		\includegraphics[width=\textwidth]{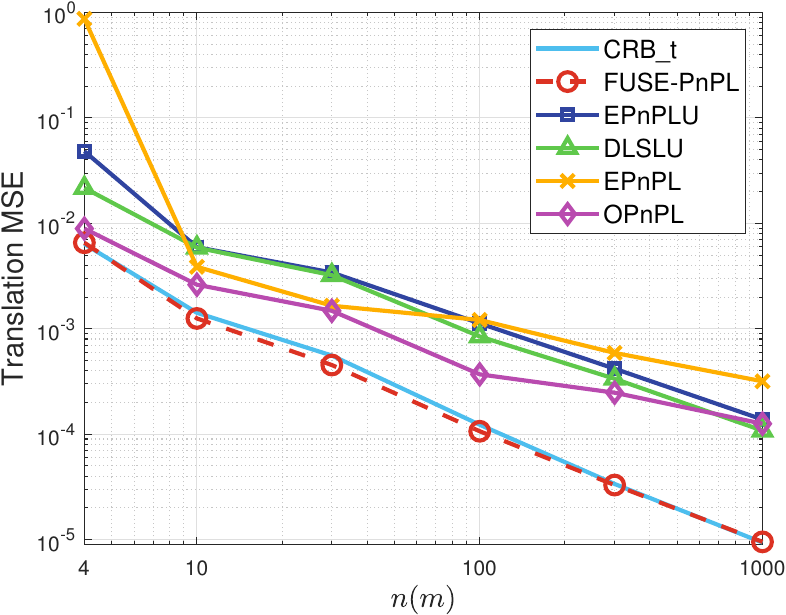}

	\end{minipage}}
 	\subfloat[{$\sigma=10$ pixels ($\bf R$)}\label{pnpl_R_mse_10px}]{%
\begin{minipage}[b]{0.24\textwidth}
		\centering
		\includegraphics[width=\textwidth]{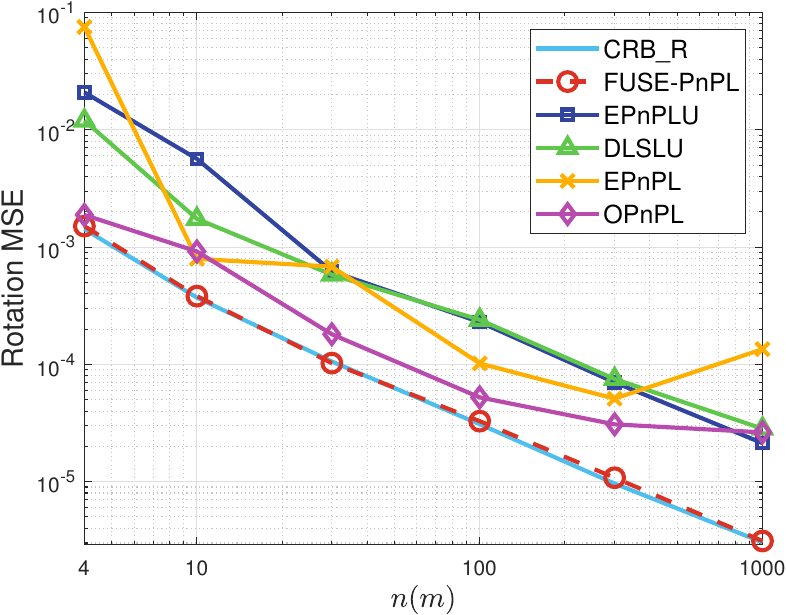}

	\end{minipage}}
	\subfloat[{$\sigma=10$ pixels ($\bf t$)}\label{pnpl_t_mse_10px}]{%
\begin{minipage}[b]{0.24\textwidth}
		\centering
		\includegraphics[width=\textwidth]{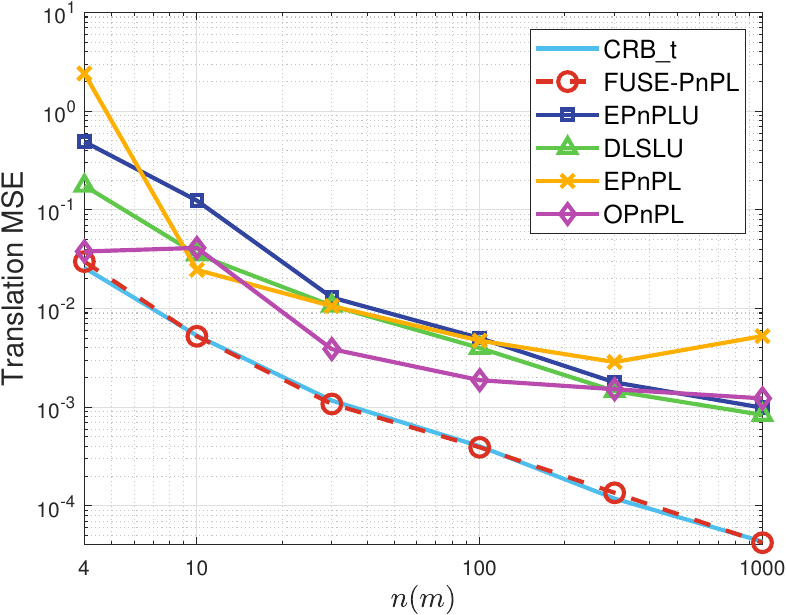}

	\end{minipage}}
	\caption{MSEs of PnPL pose estimates.}
	\label{pnpl_mse}
\end{figure*}

\begin{figure*}[!htbp]
	\centering
	\subfloat[{$\sigma=5$ pixels ($\bf R$)}\label{pnp_pnl_pnpl_R_mse_5px}]{%
\begin{minipage}[b]{0.24\textwidth}
		\centering
		\includegraphics[width=\textwidth]{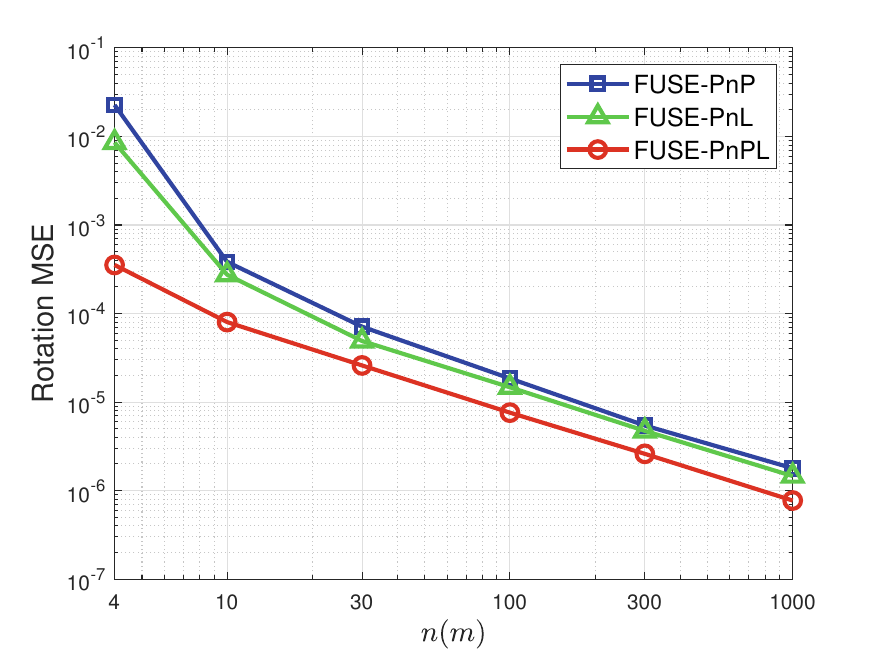}

	\end{minipage}}
	\subfloat[{$\sigma=5$ pixels ($\bf t$)}\label{pnp_pnl_pnpl_t_mse_5px}]{%
\begin{minipage}[b]{0.24\textwidth}
		\centering
		\includegraphics[width=\textwidth]{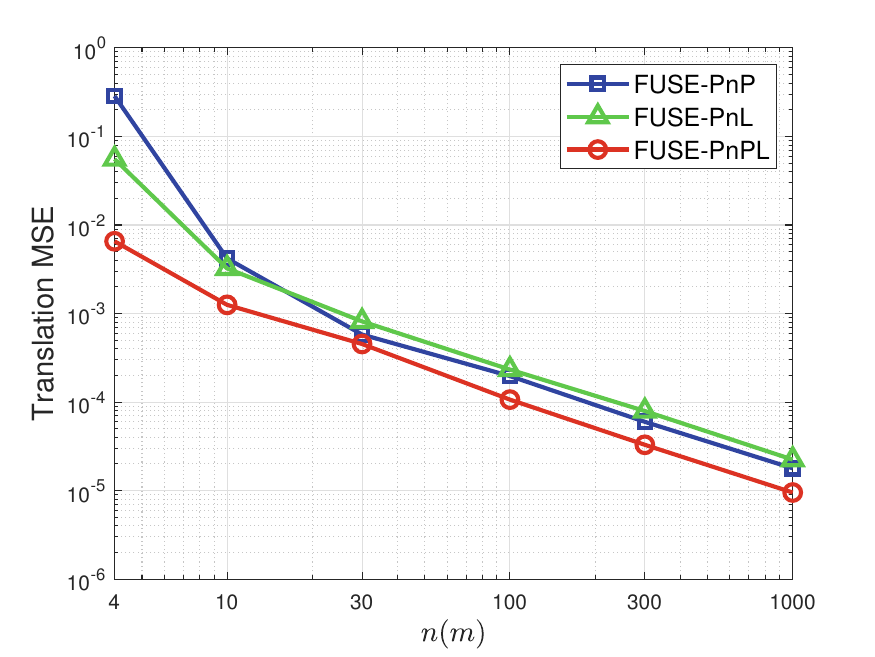}

	\end{minipage}}
 	\subfloat[{$\sigma=10$ pixels ($\bf R$)}\label{pnp_pnl_pnpl_R_mse_10px}]{%
\begin{minipage}[b]{0.24\textwidth}
		\centering
		\includegraphics[width=\textwidth]{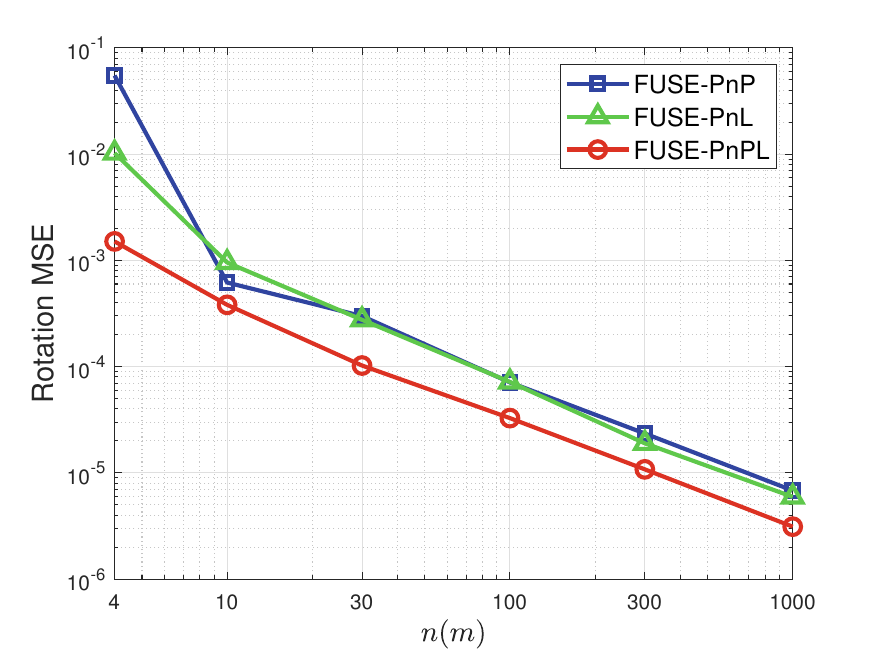}

	\end{minipage}}
	\subfloat[{$\sigma=10$ pixels ($\bf t$)}\label{pnp_pnl_pnpl_t_mse_10px}]{%
\begin{minipage}[b]{0.24\textwidth}
		\centering
		\includegraphics[width=\textwidth]{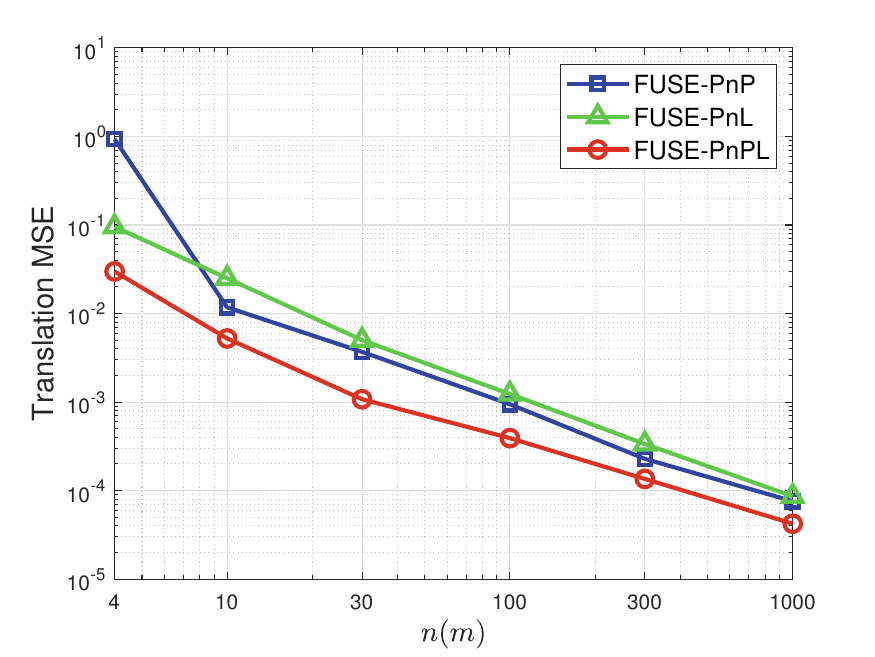}

	\end{minipage}}
	\caption{MSEs of different kinds of measurements.}
	\label{pnp_pnl_pnpl_mse}
\end{figure*}

\begin{figure*}[!htbp]
	\centering
	\subfloat[{PnP estimators}\label{pnp_CPU_time}]{%
\begin{minipage}[b]{0.24\textwidth}
		\centering
		\includegraphics[width=\textwidth]{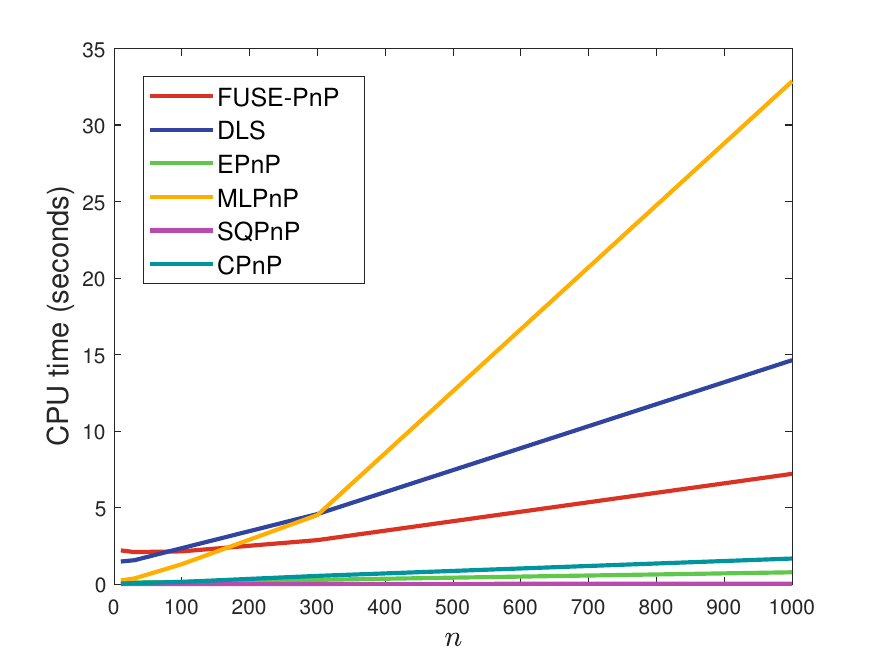}

	\end{minipage}}
 \subfloat[{PnL estimators}\label{pnl_CPU_time}]{%
\begin{minipage}[b]{0.24\textwidth}
		\centering
		\includegraphics[width=\textwidth]{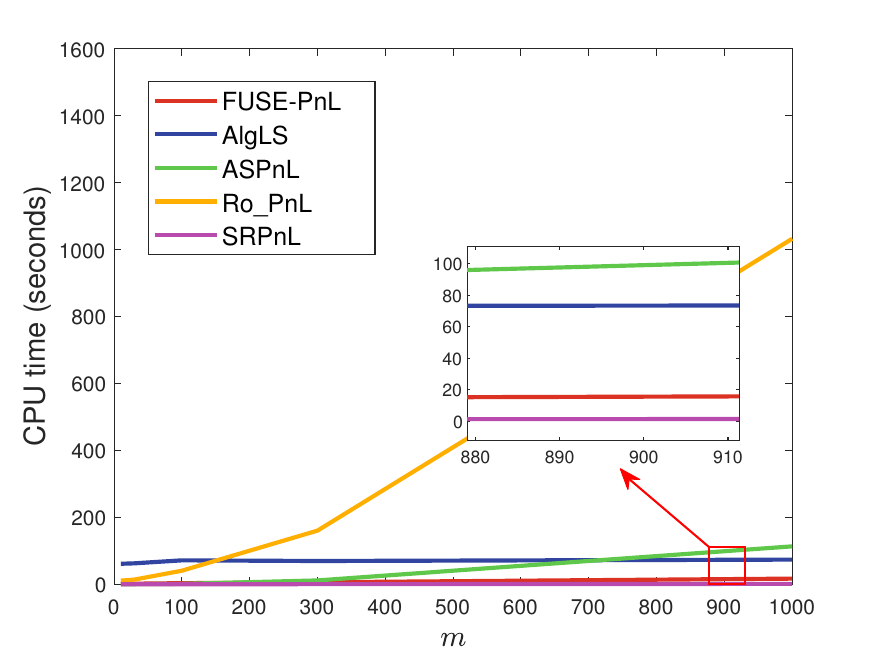}

	\end{minipage}}
	\subfloat[{PnPL estimators}\label{pnpl_CPU_time}]{%
\begin{minipage}[b]{0.24\textwidth}
		\centering
		\includegraphics[width=\textwidth]{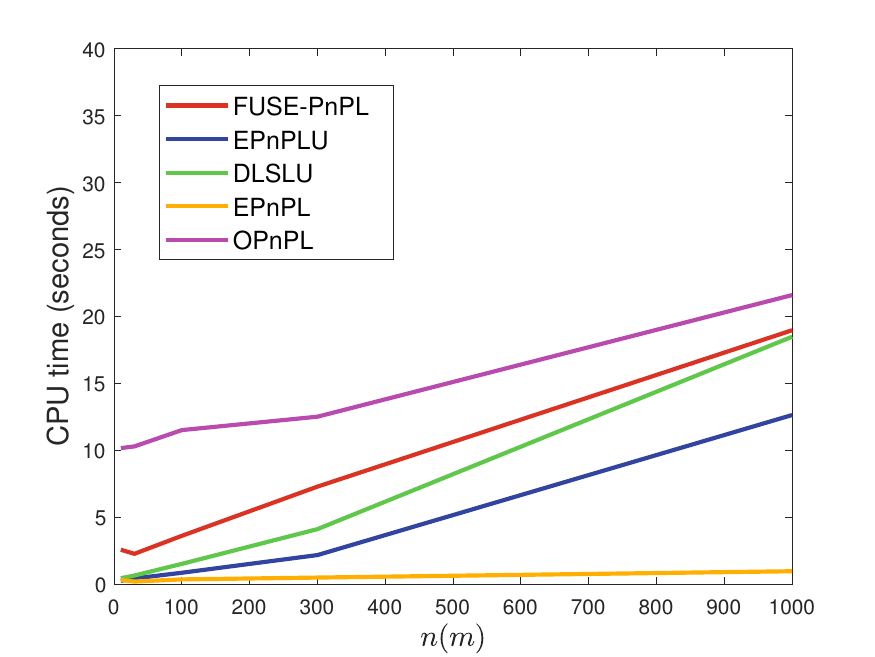}

	\end{minipage}}
  \subfloat[{Different measurements}\label{pnp_pnl_pnpl_CPU_time}]{%
\begin{minipage}[b]{0.24\textwidth}
		\centering
		\includegraphics[width=\textwidth]{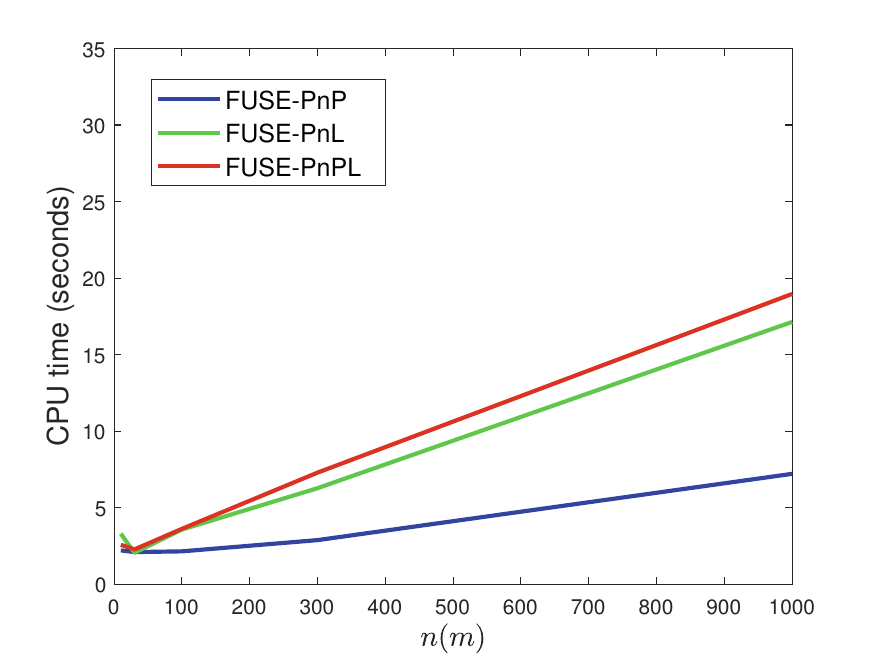}

	\end{minipage}}
	\caption{Accumulated CPU time comparison over $1000$ Monte Carlo tests.}
	\label{CPU_time_comparison}
\end{figure*}

\emph{(2) Consistency and asymptotic efficiency of pose estimators.}
We have proved that the proposed estimator is $\sqrt{n+m}$-consistent in its first step, based on which a single GN iteration suffices to reach the CRB, i.e., the proposed two-step estimator is asymptotically efficient. 
We let $n=m=4,10,30,100,300,1000$ and $\sigma=5,10$ pixels. The MSE comparison among different PnP pose estimators is presented in Figure~\ref{pnp_mse}. It is obvious from Figure~\ref{pnp_t_mse_10px} that the MSEs of the \texttt{DLS}, \texttt{EPnP}, and \texttt{SQPnP} estimators do not converge to $0$. This is because they are not asymptotically unbiased, and their asymptotic MSEs are dominated by asymptotic biases.  
Our estimator \texttt{FUSE-PnP} is consistent and asymptotically attains the CRB, coinciding with our theoretical claims. 
Although \texttt{CPnP} is also consistent and asymptotically efficient, its initial estimation step does not enforce the rotation constraint $\mathbf R\in{\rm SO}(3)$, resulting in larger MSEs than the proposed estimator at $n=10$. This relaxation also necessitates at least six point correspondences~\cite{zeng2023cpnp}.
Notably, the \texttt{MLPnP} estimator also appears to be asymptotically efficient in this scenario. This behavior can be attributed to its use of multiple GN iterations and the relatively large basin of attraction of the global minimum of the ML problem~\eqref{ML_problem}. Although its initial estimator is inconsistent, the resulting initial estimate can still lie within this basin, allowing subsequent GN iterations to converge to the global minimum. Finally, both \texttt{CPnP} and \texttt{MLPnP} cannot handle the case $n=4$.

To further demonstrate the advantages of statistical consistency and minimal pose parameterization over existing PnP solvers, we consider a challenging configuration in which feature points are confined to only 1\% of the image area and the depths of 3D points are restricted to $[8,9]$ m. The number of point correspondences is set to $n=1000$, with a noise level of $\sigma=5$ pixels. Figure~\ref{pnp_GN_number} reports the MSEs versus the number of GN iterations. Under this configuration, the initial estimate of \texttt{MLPnP} exhibits substantial bias and may lie outside the basin of attraction of the global ML solution. Consequently, GN refinement may even increase its MSE. Although \texttt{CPnP} is consistent, its initial estimation step does not enforce the rotation constraint $\mathbf R\in\mathrm{SO}(3)$, resulting in a relatively inaccurate initial estimate that cannot be adequately corrected within the considered number of GN iterations. In contrast, our estimator employs a minimal pose parameterization and achieves a low MSE even without GN refinement.

Figure~\ref{pnl_mse} compares the MSEs of different PnL pose estimators. The proposed \texttt{FUSE-PnL} estimator asymptotically approaches the CRB as the number of line correspondences increases. In small-sample settings, its MSE may even fall below the CRB. This does not contradict the bound, since the CRB is a lower bound for unbiased estimators, while the proposed estimator is not necessarily unbiased in small-sample settings. Moreover, \texttt{FUSE-PnL} substantially outperforms the competing estimators. This improvement can be attributed to our accurate characterization of measurement noise propagation and the resulting ML formulation in~\eqref{ML_problem} for line correspondences, whereas the competing methods solve least-squares problems without statistically optimal weighting matrices. The MSEs of \texttt{ASPnL}, \texttt{SRPnL}, and \texttt{AlgLS} also exhibit fluctuations, suggesting less stable performance in these experiments.

\begin{figure}[!t]
	\centering
	\begin{minipage}[b]{0.23\textwidth}
		\centering
		\includegraphics[width=\textwidth]{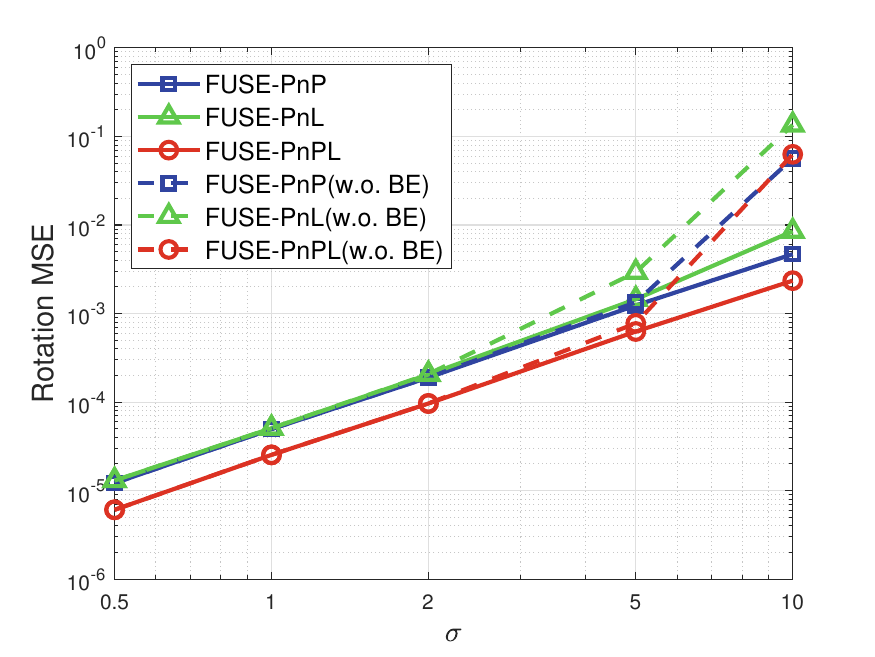}
		\label{be_ablation_rotation}
	\end{minipage}
    \hspace{2mm}
	\begin{minipage}[b]{0.23\textwidth}
		\centering
		\includegraphics[width=\textwidth]{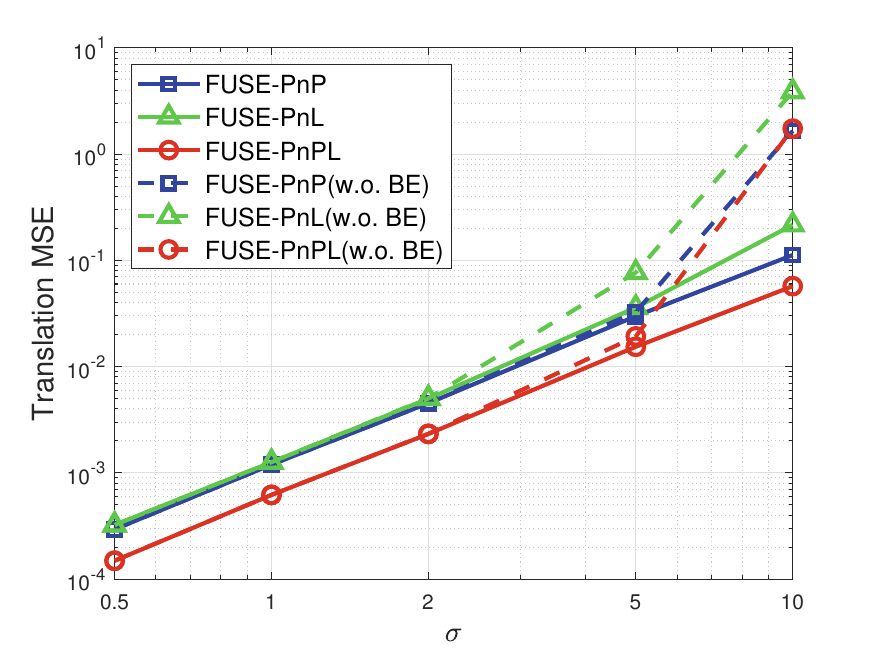}
		\label{be_ablation_translation}
	\end{minipage}
    \vspace{-5mm}
	\caption{Ablation study of bias elimination under a challenging configuration.}
	\label{be_ablation}
\end{figure}

Finally, we compare PnPL estimators that use both point and line correspondences. As shown in Figure~\ref{pnpl_mse}, the proposed \texttt{FUSE-PnPL} estimator approaches the CRB and outperforms the competing estimators. In particular, when the number of correspondences is large, its MSEs are more than an order of magnitude lower than those of the competing methods. 
Figure~\ref{pnp_pnl_pnpl_mse} compares the MSEs of the proposed estimators using different types of measurements. Point-only and line-only estimation yield comparable MSEs in these experiments. This similarity can be attributed to the identical noise levels applied to points and line endpoints, together with the fact that each point or line correspondence provides two independent equations. By jointly exploiting point and line correspondences, \texttt{FUSE-PnPL} achieves an MSE smaller than that of either \texttt{FUSE-PnP} or \texttt{FUSE-PnL}.

\begin{figure*}[!htbp]
	\centering
	\subfloat[{Rotation errors}\label{ETH3D_bar_plot_R}]{%
\begin{minipage}[b]{0.49\textwidth}
		\centering
		\includegraphics[width=\textwidth]{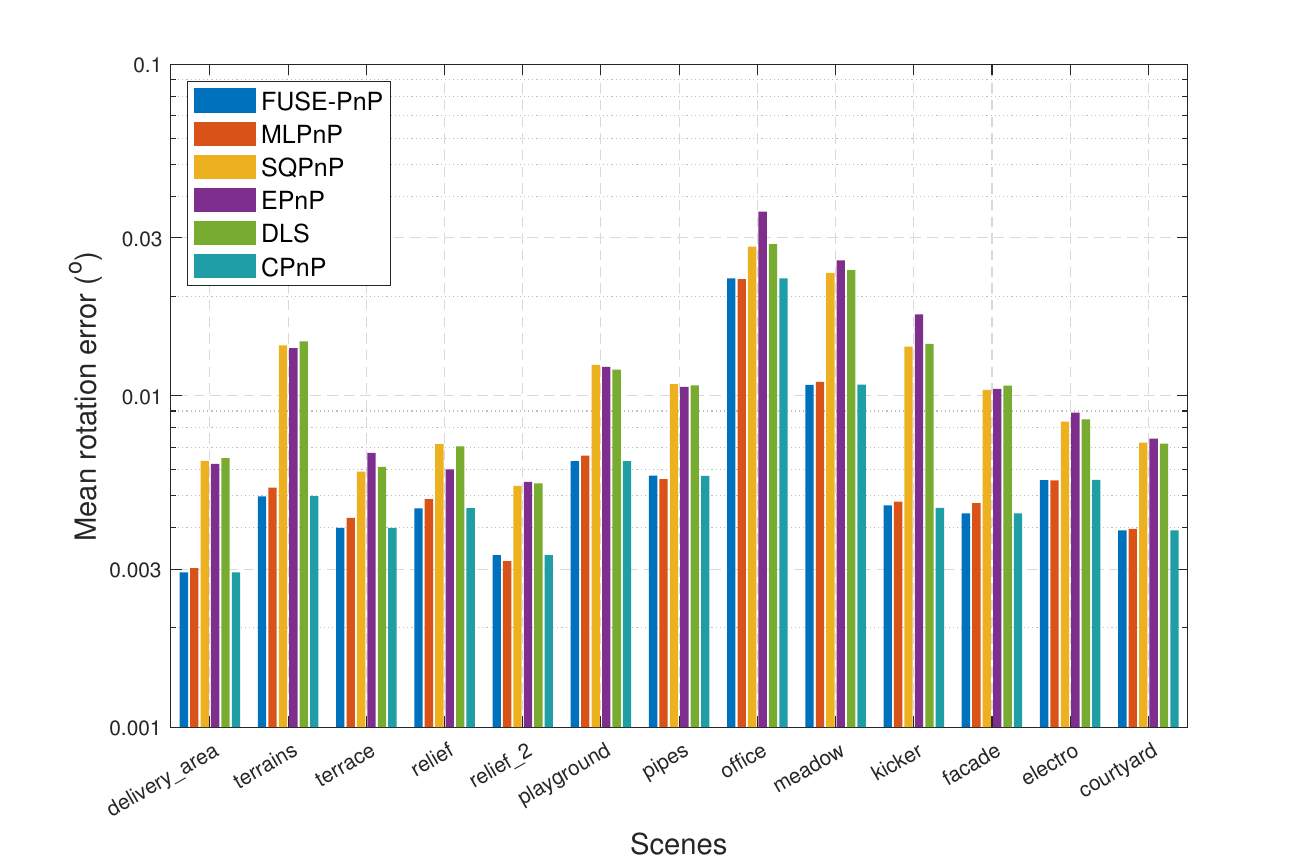}

	\end{minipage}}
  \subfloat[{Translation errors}\label{ETH3D_bar_plot_t}]{%
\begin{minipage}[b]{0.49\textwidth}
		\centering
		\includegraphics[width=\textwidth]{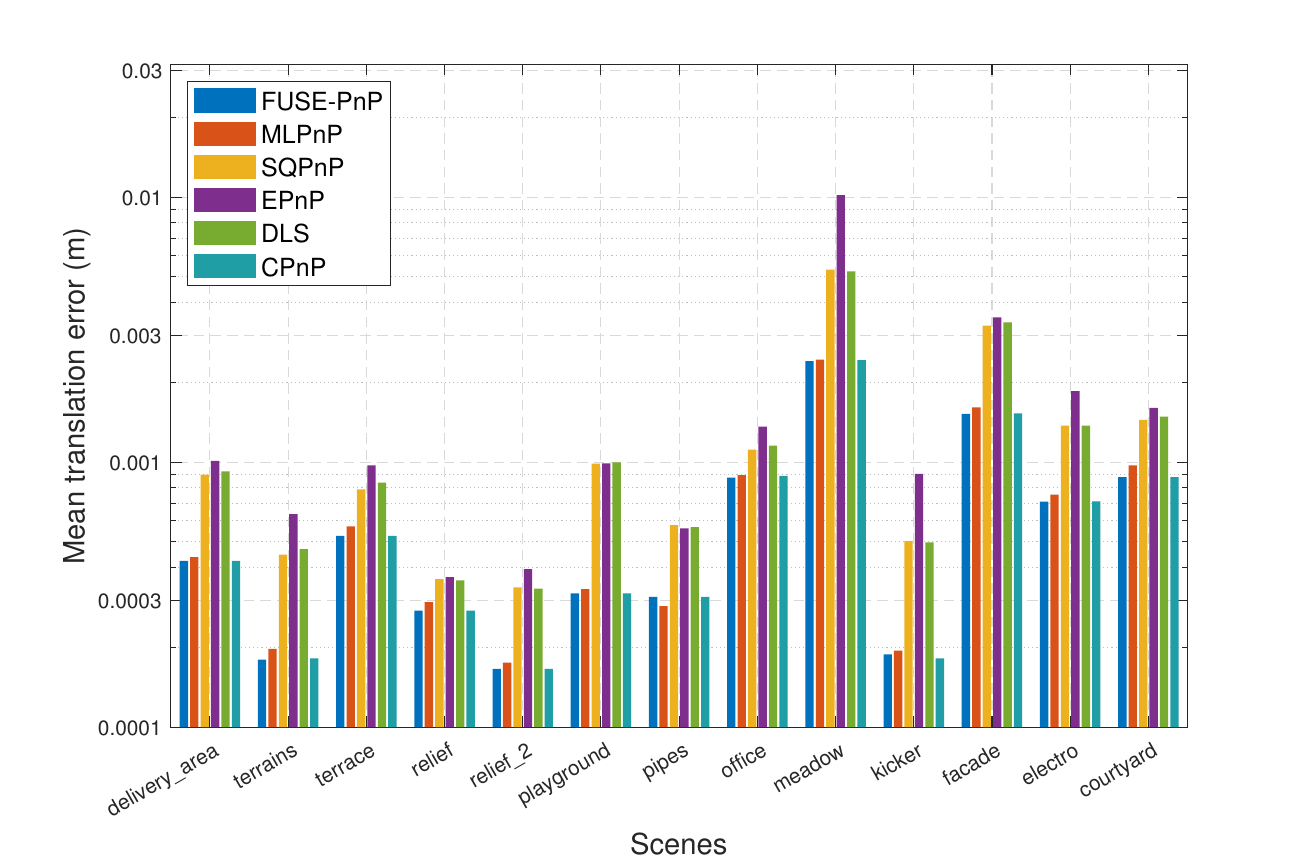}

	\end{minipage}}
	\caption{Mean PnP pose estimation error comparison in $13$ scenes of the ETH3D dataset.}
	\label{ETH3D_bar_plot}
\end{figure*}

\begin{figure*}[!htbp]
	\centering
	\subfloat[{Rotation errors}\label{VGG_bar_plot_R}]{%
\begin{minipage}[b]{0.49\textwidth}
		\centering
		\includegraphics[width=\textwidth]{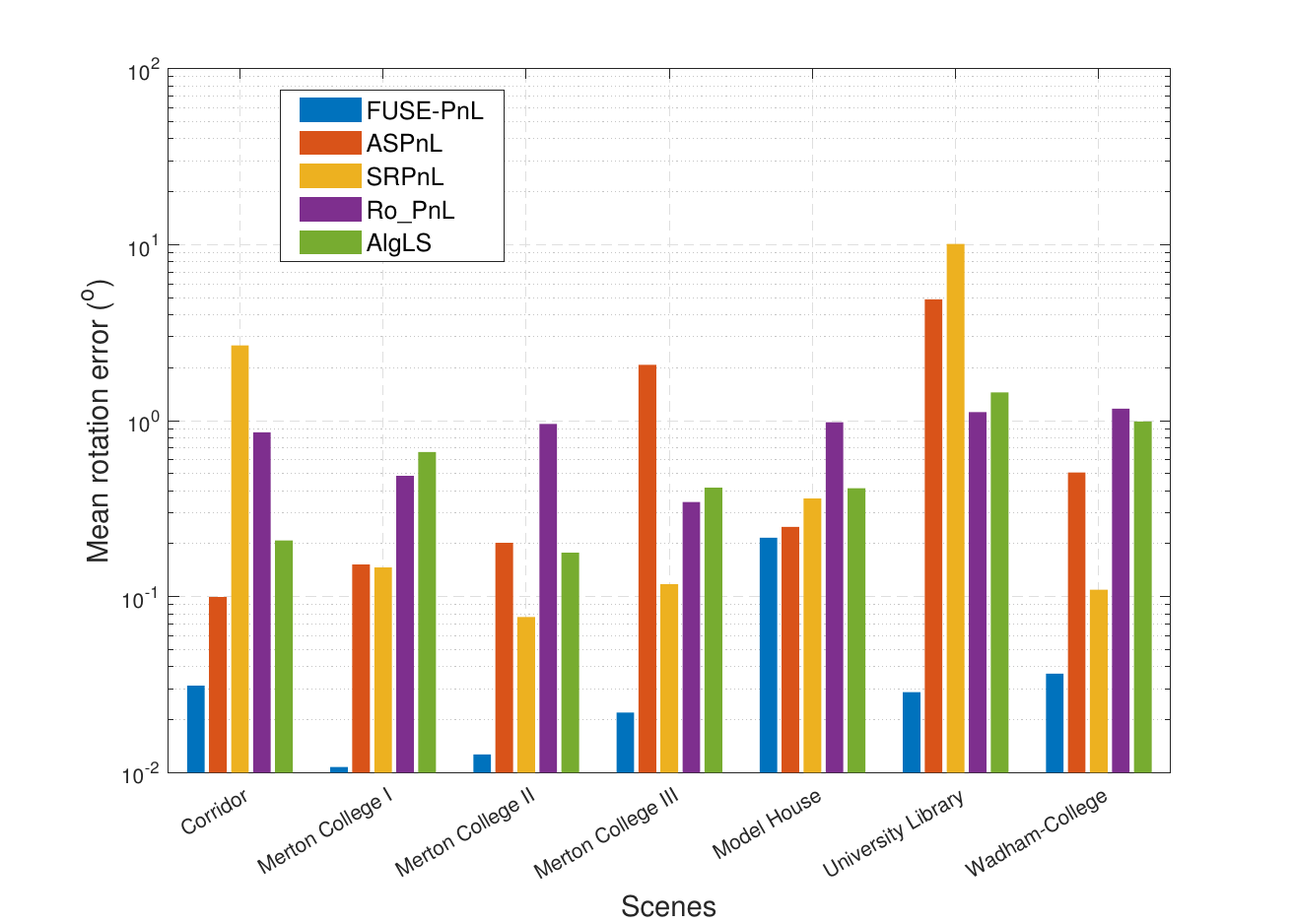}

	\end{minipage}}
  \subfloat[{Translation errors}\label{VGG_bar_plot_t}]{%
\begin{minipage}[b]{0.49\textwidth}
		\centering
		\includegraphics[width=\textwidth]{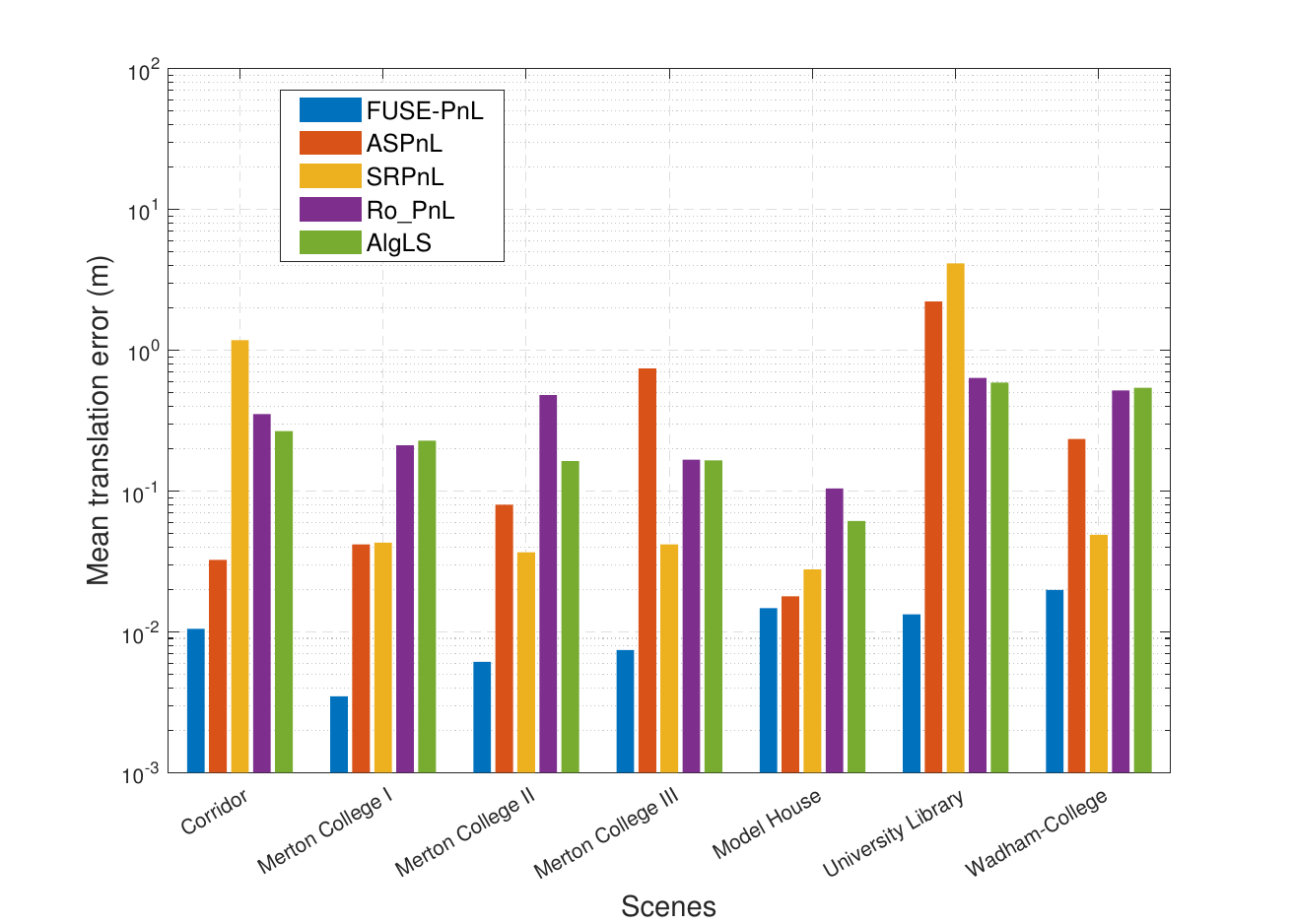}

	\end{minipage}}
	\caption{Mean PnL pose estimation error comparison in $7$ scenes of the VGG dataset.}
	\label{VGG_bar_plot}
\end{figure*}

\begin{figure*}[!htbp]
	\centering
	\subfloat[{Rotation errors}\label{VGG_bar_plot_R_PnPL}]{%
\begin{minipage}[b]{0.49\textwidth}
		\centering
		\includegraphics[width=\textwidth]{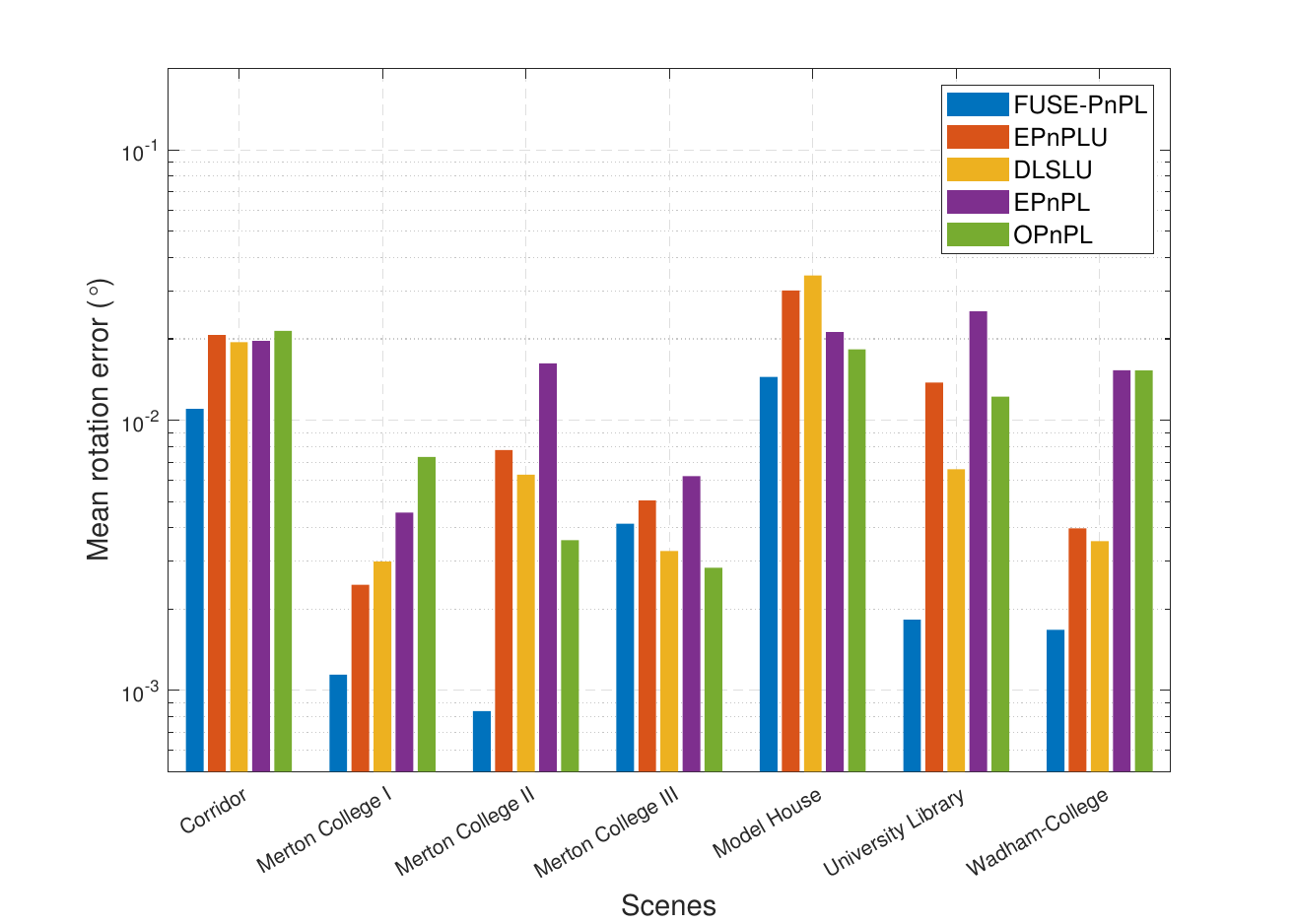}

	\end{minipage}}
  \subfloat[{Translation errors}\label{VGG_bar_plot_t_PnPL}]{%
\begin{minipage}[b]{0.49\textwidth}
		\centering
		\includegraphics[width=\textwidth]{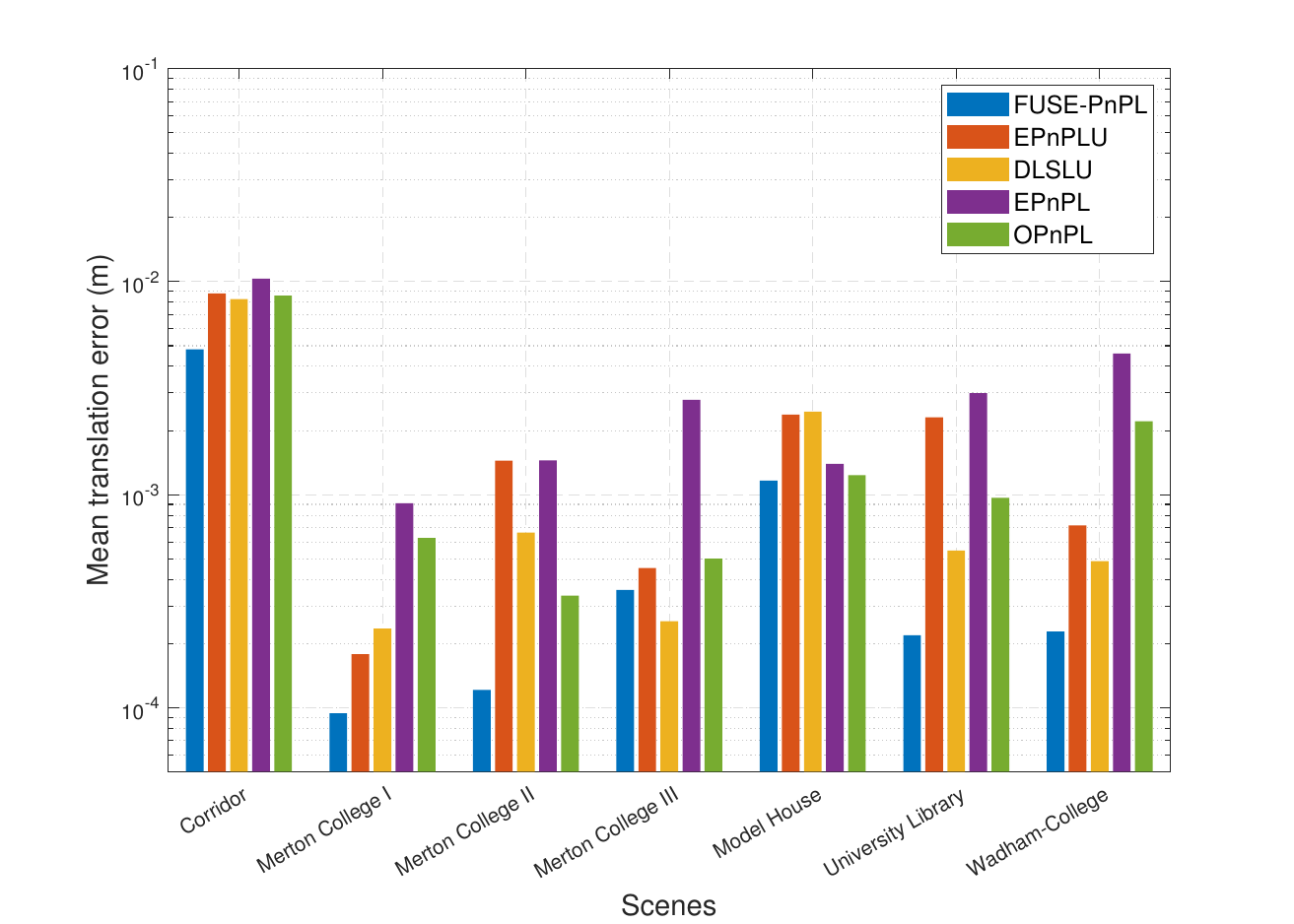}

	\end{minipage}}
	\caption{Mean PnPL pose estimation error comparison in $7$ scenes of the VGG dataset.}
	\label{VGG_bar_plot_PnPL}
\end{figure*}

\begin{figure*}[!htbp]
	\centering
	\begin{minipage}[b]{\dimexpr0.23\textwidth+25pt\relax}
		\centering
  \makebox[25pt]{\raisebox{40pt}{\rotatebox[origin=c]{90}{FUSE-PnL}}}%
		\includegraphics[width=\dimexpr\linewidth-25pt\relax]{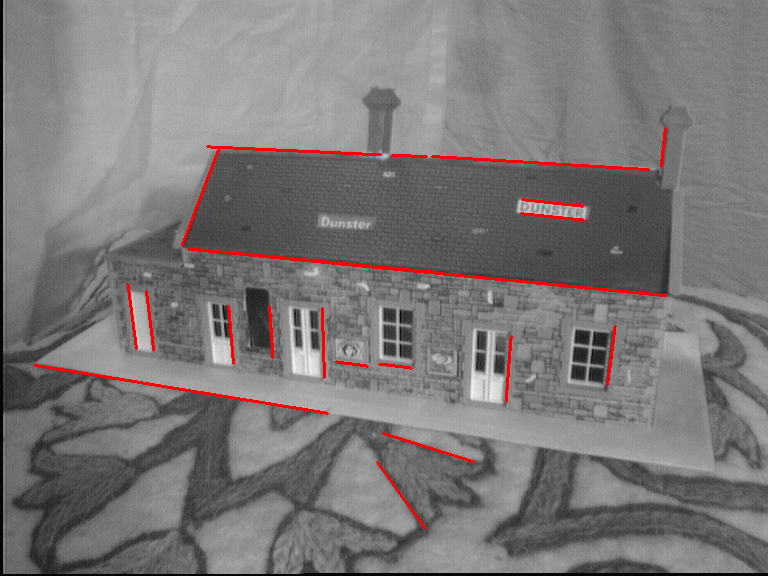}
  \makebox[25pt]{\raisebox{40pt}{\rotatebox[origin=c]{90}{Ro\_PnL}}}%
  \includegraphics[width=\dimexpr\linewidth-25pt\relax]{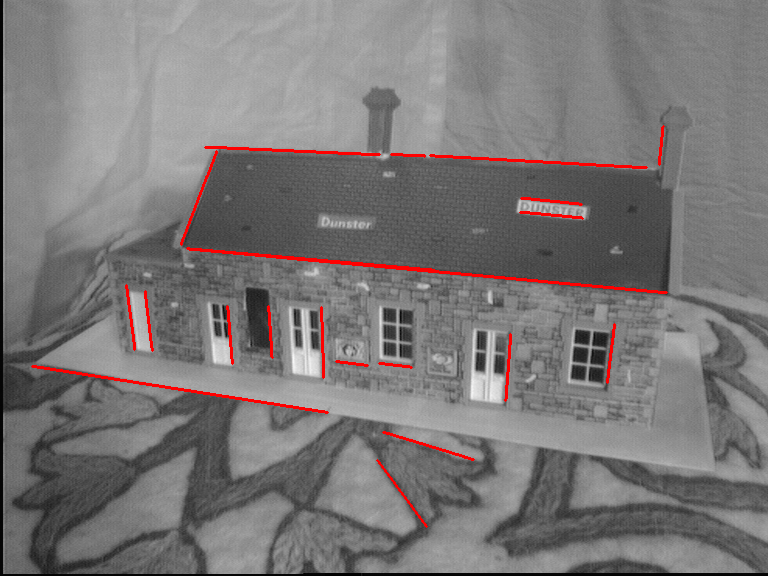}
\makebox[25pt]{\raisebox{40pt}{\rotatebox[origin=c]{90}{ASPnL}}}%
\includegraphics[width=\dimexpr\linewidth-25pt\relax]{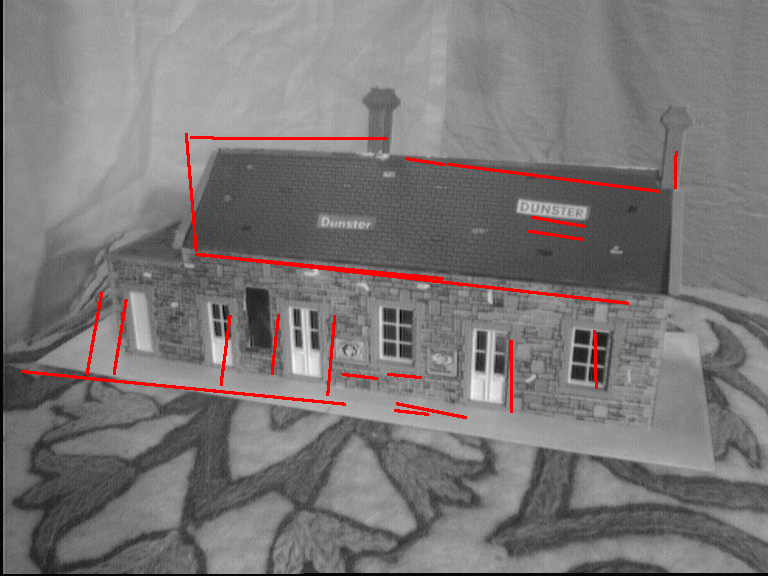}
\makebox[25pt]{\raisebox{40pt}{\rotatebox[origin=c]{90}{SRPnL}}}%
\includegraphics[width=\dimexpr\linewidth-25pt\relax]{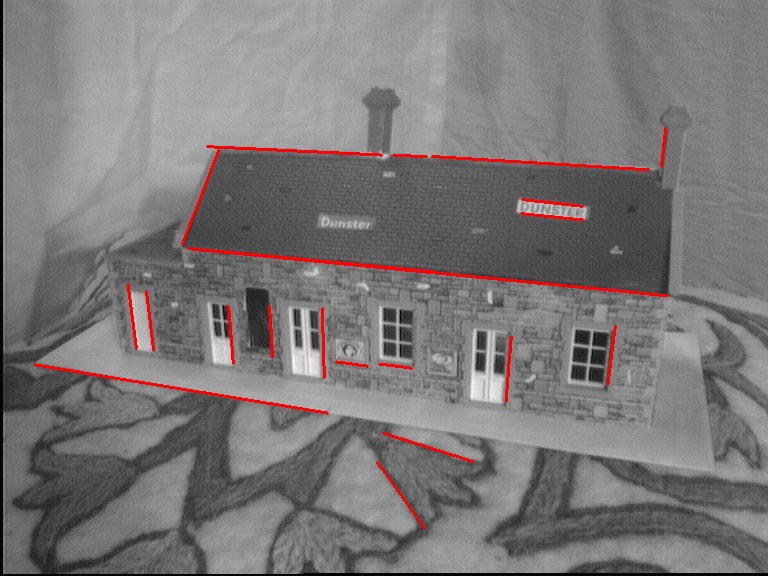}
\makebox[25pt]{\raisebox{40pt}{\rotatebox[origin=c]{90}{AlgLS}}}%
\includegraphics[width=\dimexpr\linewidth-25pt\relax]{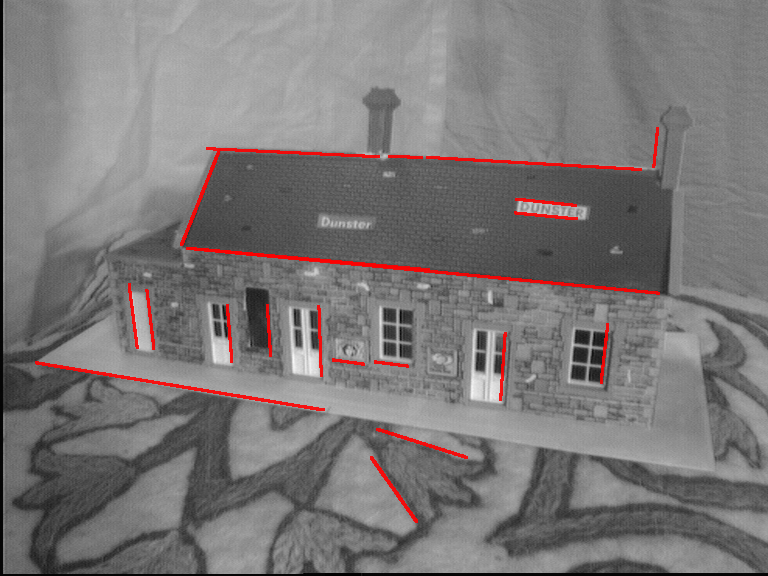}
\par\smallskip{\footnotesize Model-House 005\par}
	\end{minipage}\hfill
  \begin{minipage}[b]{0.1726\textwidth}
		\centering
 \includegraphics[width=\textwidth]{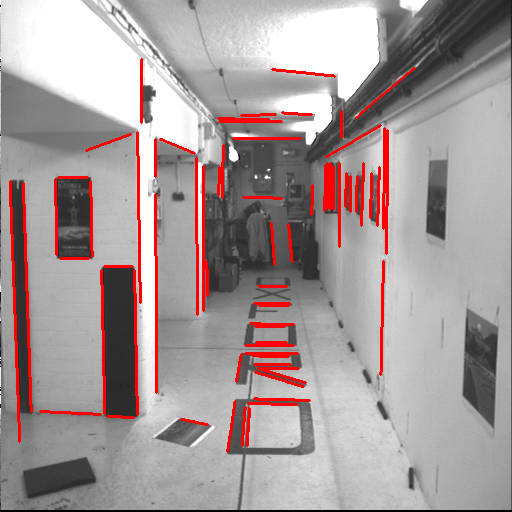}	
  \includegraphics[width=\textwidth]{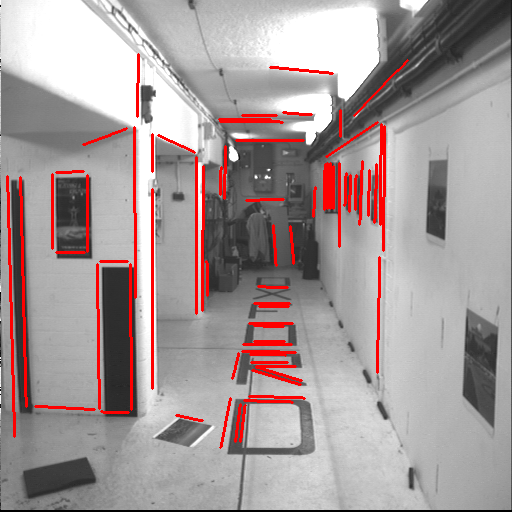}
\includegraphics[width=\textwidth]{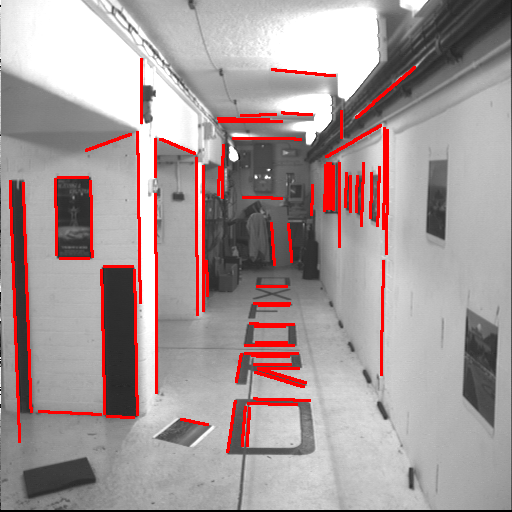}
\includegraphics[width=\textwidth]{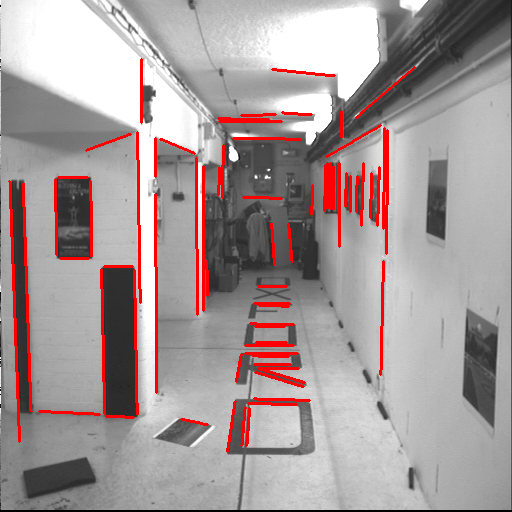}
\includegraphics[width=\textwidth]{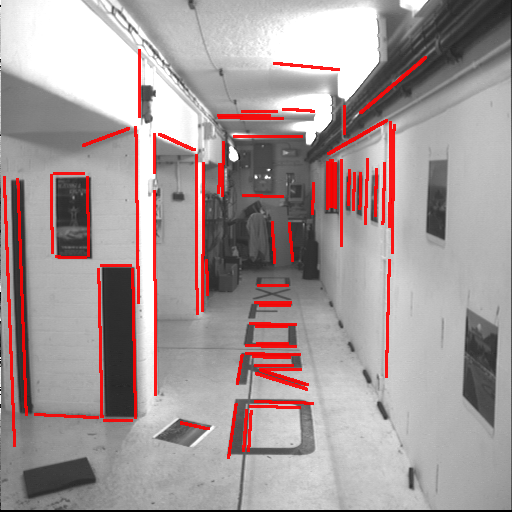}
\par\smallskip{\footnotesize Corridor 001\par}
	\end{minipage}\hfill
 \begin{minipage}[b]{0.23\textwidth}
		\centering
\includegraphics[width=\textwidth]{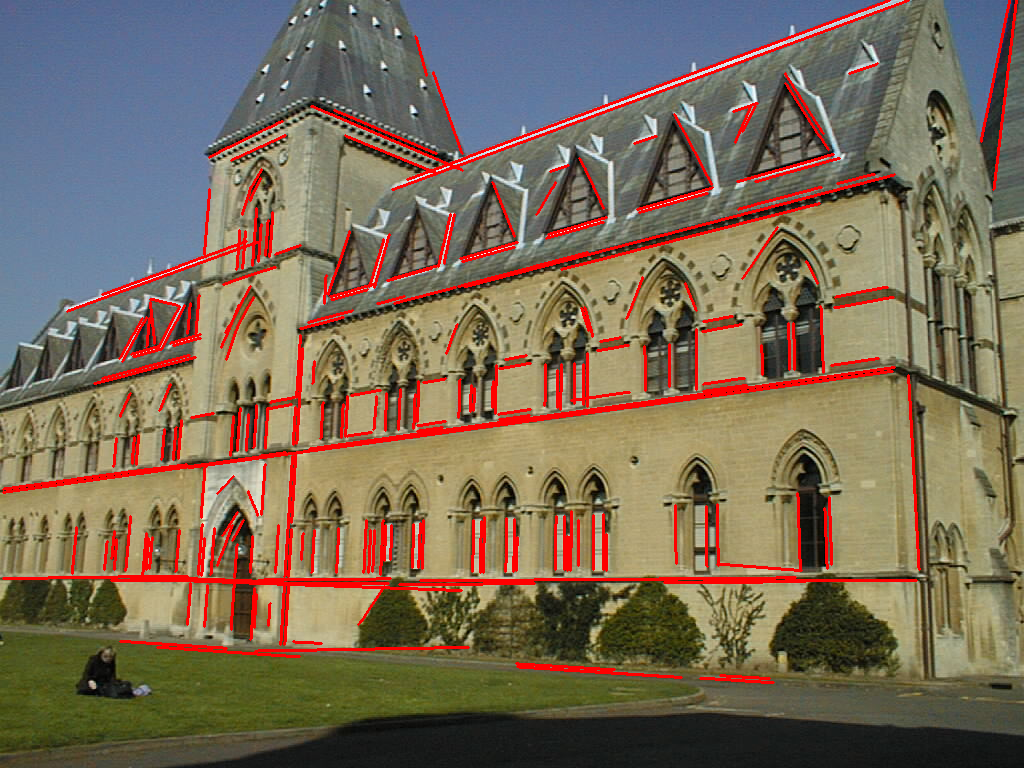}		
\includegraphics[width=\textwidth]{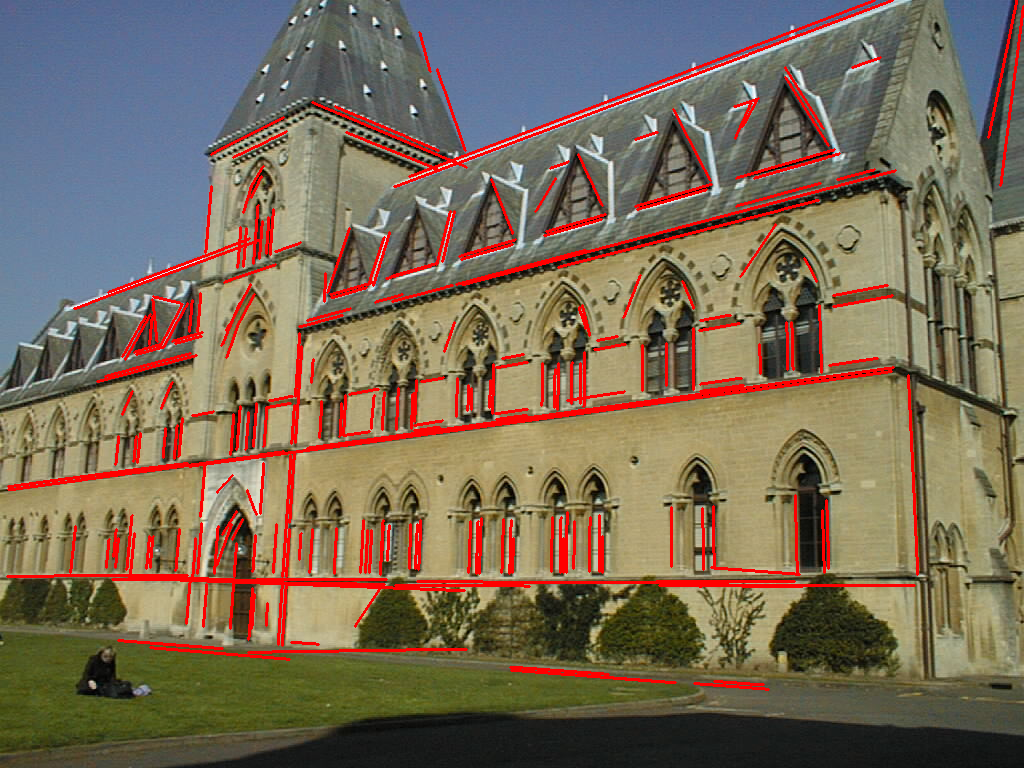}
\includegraphics[width=\textwidth]{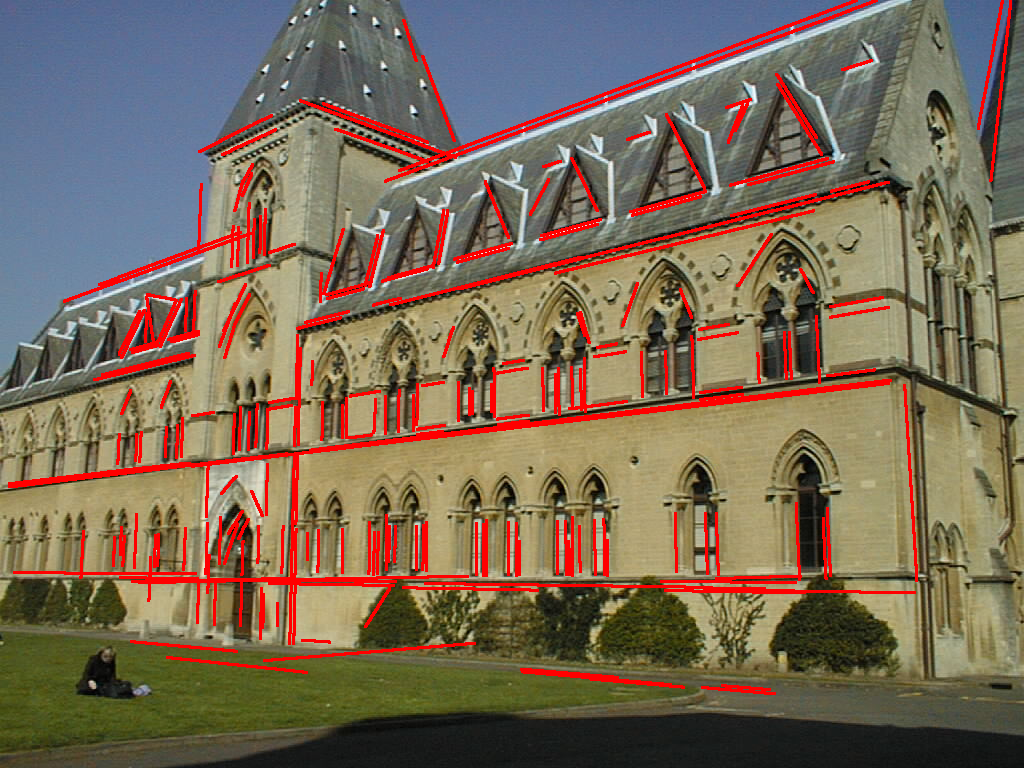}
\includegraphics[width=\textwidth]{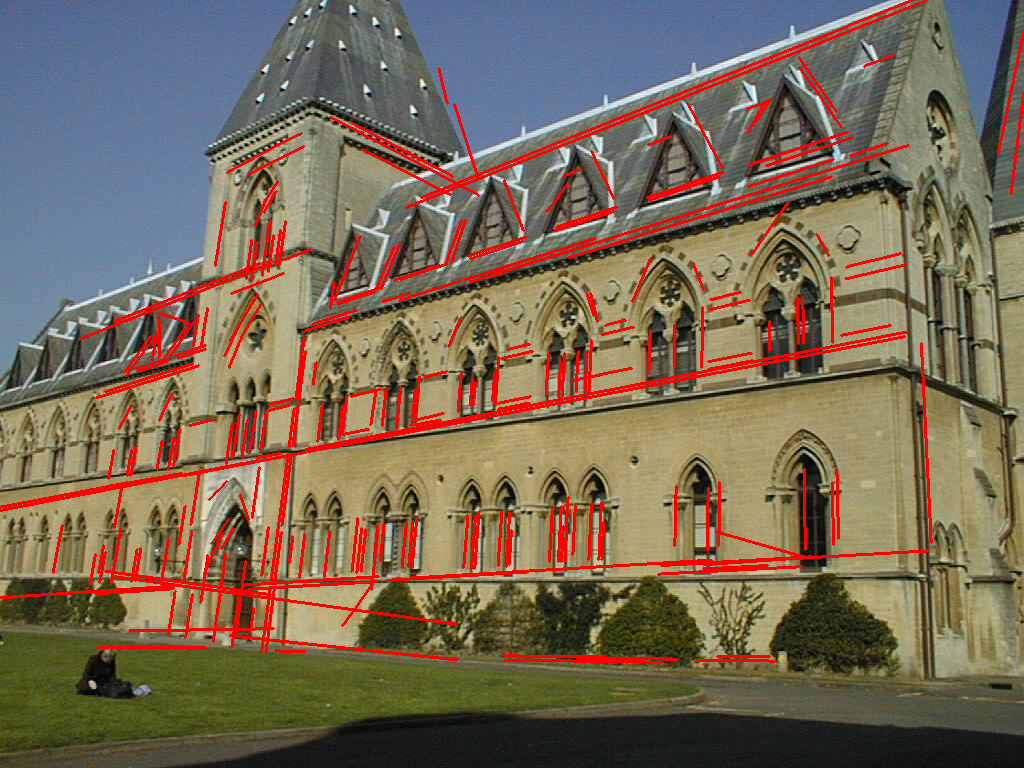}
\includegraphics[width=\textwidth]{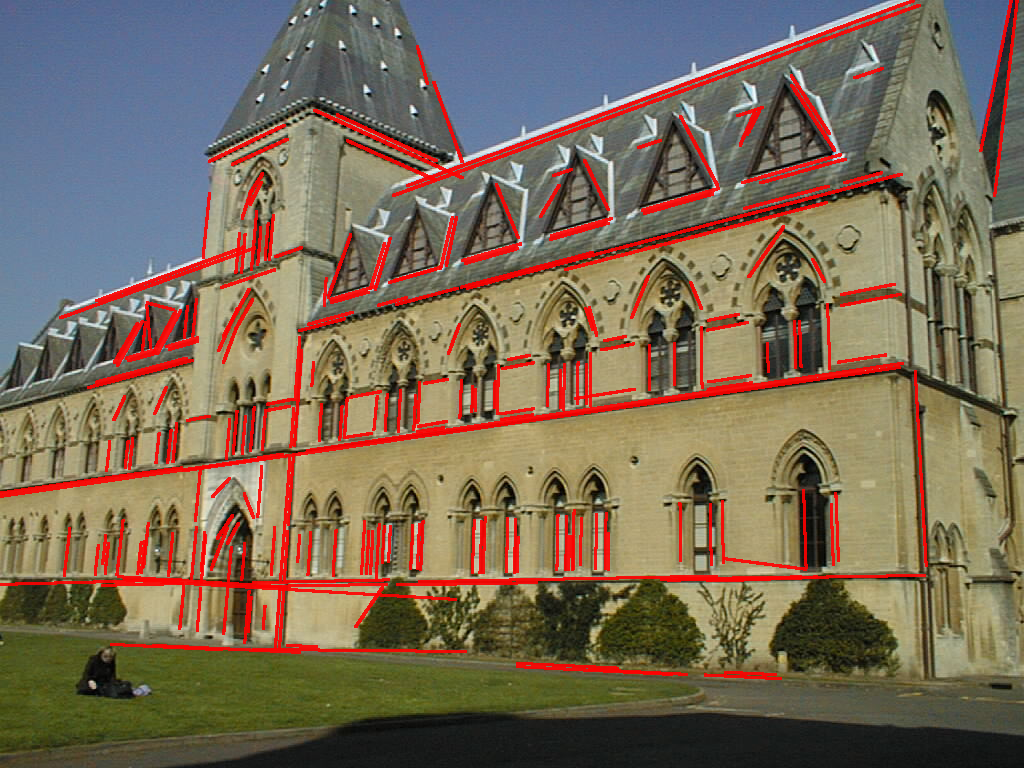}
	\par\smallskip{\footnotesize University-Library 002\par}
 \end{minipage}\hfill
  \begin{minipage}[b]{0.23\textwidth}
		\centering
\includegraphics[width=\textwidth]{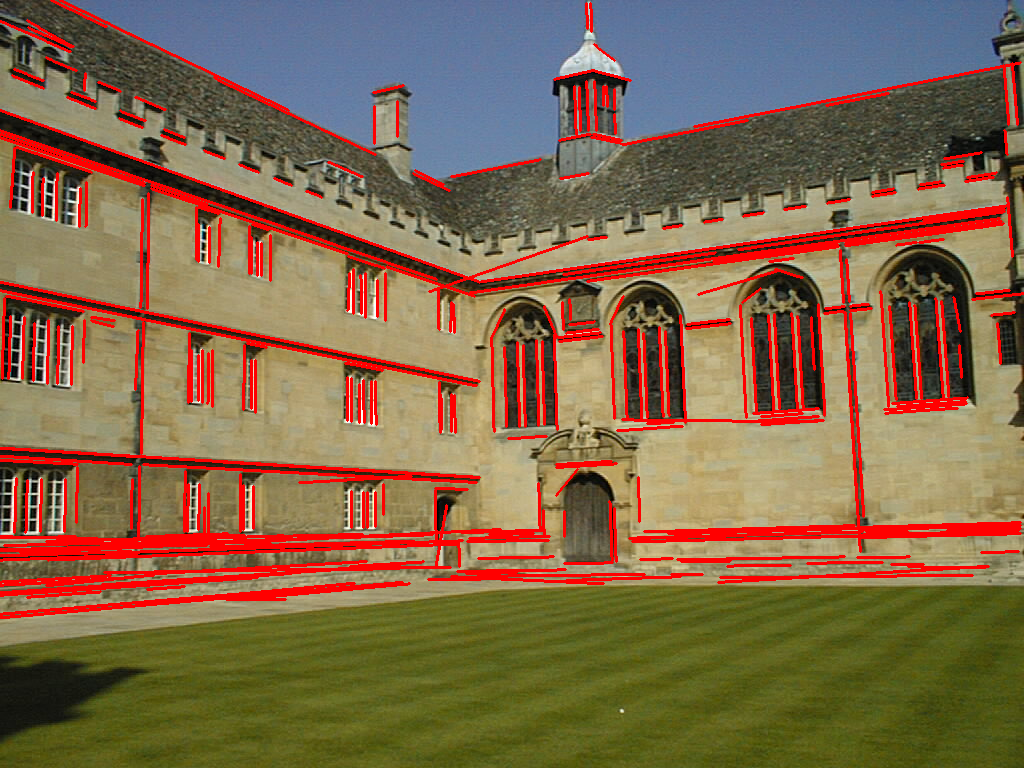}		
\includegraphics[width=\textwidth]{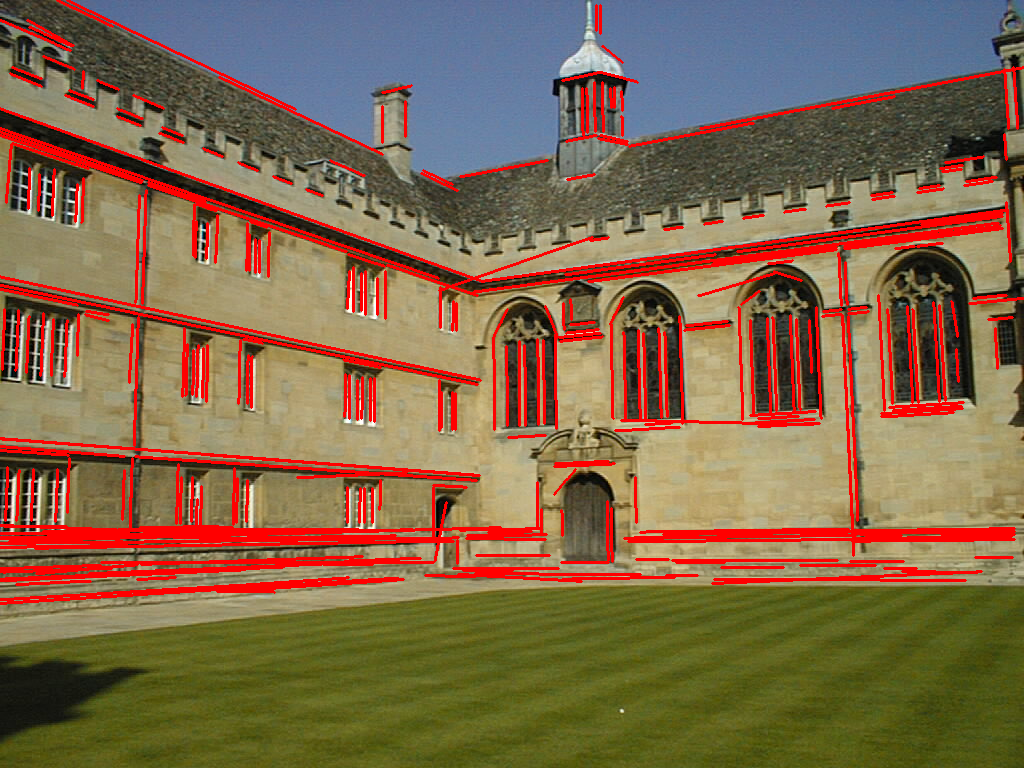}
\includegraphics[width=\textwidth]{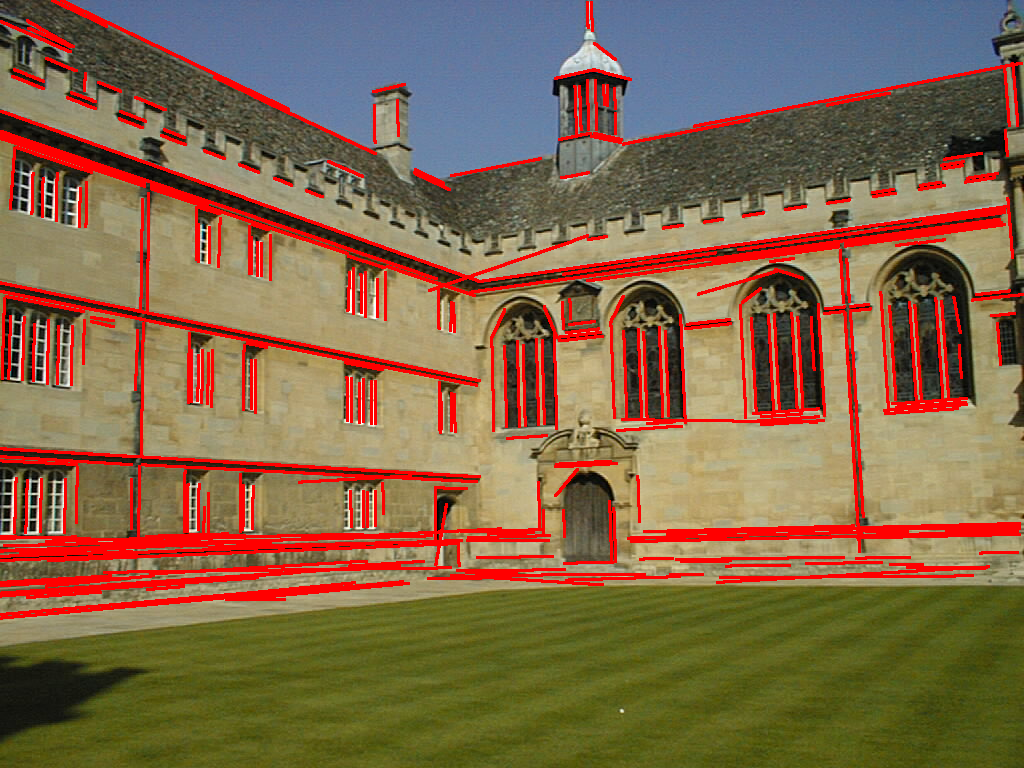}
\includegraphics[width=\textwidth]{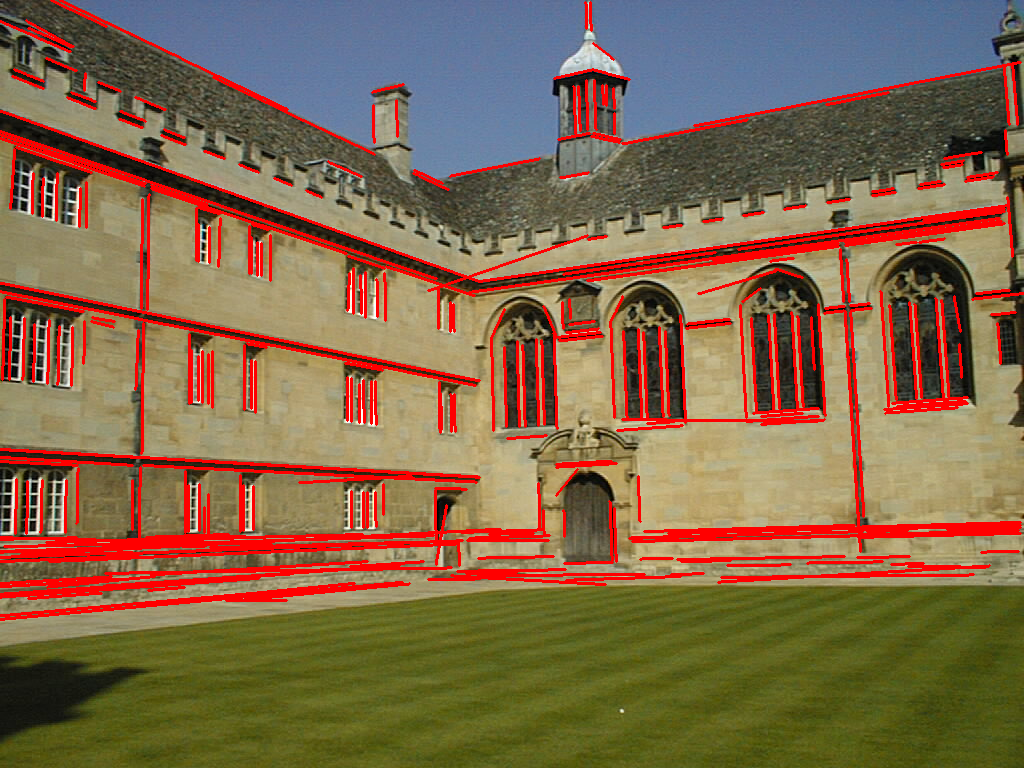}
\includegraphics[width=\textwidth]{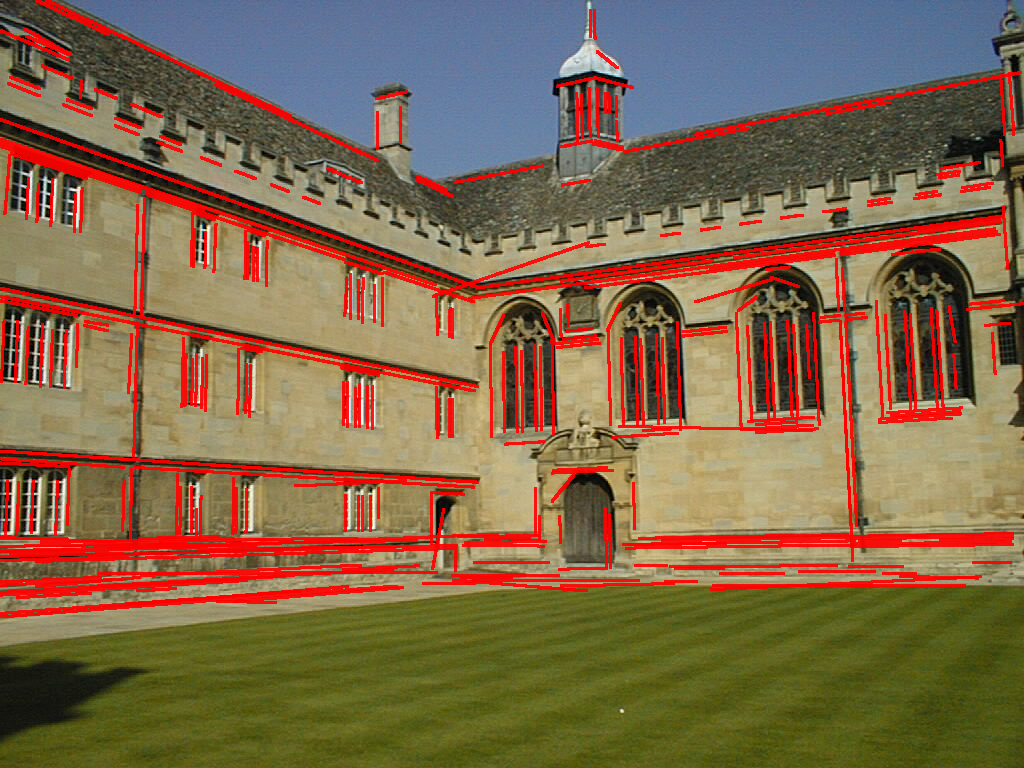}
	\par\smallskip{\footnotesize Wadham-College 004\par}
 \end{minipage}\hfill
	\caption{Line reprojection onto images with estimated poses.}
	\label{VGG_line_reprojection}
\end{figure*}

\emph{(3) Linear time complexity of pose estimators.}
As analyzed in Section~\ref{asymptotic_property_analysis}, the proposed algorithm has an overall time complexity of $O(n+m)$, indicating that its computational cost scales linearly with the total number of measurements. Figure~\ref{CPU_time_comparison} reports the cumulative CPU time over $1000$ Monte Carlo trials for all evaluated algorithms. All algorithms are implemented in MATLAB, except for \texttt{SQPnP}, which uses a C++ implementation. The results confirm that the CPU time of the proposed algorithm increases linearly with the number of measurements. Its computational efficiency is intermediate among the evaluated methods. Notably, \texttt{FUSE-PnPL} remains fast enough for real-time operation even with thousands of features. In contrast, some estimators, such as \texttt{Ro\_PnL} and \texttt{AlgLS}, may become too computationally expensive for real-time use when the number of measurements is large.

\emph{(4) Minimal-configuration tests.} 
To evaluate the performance of the proposed estimator in minimal cases, we consider four configurations: P3P, P3L, P2P1L, and P1P2L. For each configuration, we investigate two scenarios: a noise-free scenario to assess whether the estimator can recover the ground-truth pose, and a noisy scenario with $\sigma=1$px to evaluate its robustness to measurement noise.
We conduct 1,000 Monte Carlo trials for each scenario. Since minimal configurations may admit multiple poses consistent with the given correspondences, we evaluate the noise-free scenario using the coverage rate (CR), defined as the proportion of trials in which the candidate pose set contains the ground-truth pose. For the noisy scenario, we select the candidate pose closest to the ground truth and compute the corresponding MSEs.
The results are reported in Table~\ref{tab:minimal_cases}, where methods that cannot handle a particular minimal configuration are omitted. The proposed estimator achieves a 100\% coverage rate across all four configurations in the noise-free scenario, whereas \texttt{SQPnP}, \texttt{Ro\_PnL}, \texttt{DLSLU}, and \texttt{OPnPL} fail to achieve full coverage. Moreover, the proposed estimator consistently achieves the lowest MSEs in the noisy scenario, demonstrating its superior robustness to measurement noise.

\emph{(5) Effect of bias elimination.}
We conduct an ablation study to evaluate the effectiveness of bias elimination. In typical configurations, the basin of attraction of the ML solution may be sufficiently large for GN refinement to compensate for inaccuracies in the initial estimate, even without bias elimination. To better reveal its contribution, we consider a challenging configuration in which 2D features are confined to the central 25\% of the image area, while the depths of 3D points and line endpoints are restricted to $[8,9]$m. The numbers of point and line correspondences are set to $n=m=100$.
We vary the measurement noise level and report the ablation results in Figure~\ref{be_ablation}. As the noise level reaches $\sigma=5$px, the variants without bias elimination (denoted as w.o. BE) begin to exhibit noticeably higher MSEs than their full counterparts. This performance gap becomes more pronounced as the noise level increases, suggesting that the uncorrected initial estimates may fall outside the attraction basin of the ML solution. These results demonstrate the importance of bias elimination in providing reliable initialization, particularly under challenging geometric configurations and high measurement noise.

\begin{table}[!t]
    \centering
    \caption{Minimal-case tests. CR denotes the coverage rate of the ground-truth pose by the candidate set.}
    \label{tab:minimal_cases}
    \small
    \begin{tabular}{llccc}
        \toprule
        \multirow{2}{*}{Case} &
        \multirow{2}{*}{Method} &
        \multicolumn{1}{c}{Noise-free} &
        \multicolumn{2}{c}{1 px} \\
        \cmidrule(lr){3-3}
        \cmidrule(lr){4-5}
        & & CR & MSE ($\mathbf{R}$) & MSE ($\mathbf{t}$) \\
        \midrule
        \multirow{3}{*}{P3P}
        & FUSE-PnP & \textbf{100\%} & \textbf{0.9816} & \textbf{16.2495} \\
        & DLS      & \textbf{100\%} & 2.3140 & 33.5204 \\
        & SQPnP    & 96\% & 3.0056 & 44.9912 \\
        \midrule
        \multirow{4}{*}{P3L}
        & FUSE-PnL & \textbf{100\%} & \textbf{0.8724} & \textbf{11.0975} \\
        & ASPnL    & \textbf{100\%} & 3.4731 & 33.9313 \\
        & Ro\_PnL  & 18.5\% & 7.0050 & 114.3973 \\
        & SRPnL    & \textbf{100\%} & 2.1899 & 26.3256 \\
        \midrule
        \multirow{3}{*}{P2P1L}
        & FUSE-PnPL & \textbf{100\%} & \textbf{1.4510} & \textbf{15.7349} \\
        & DLSLU     & 83.4\% & 2.2337 & 33.2606 \\
        & OPnPL     & \textbf{100\%} & 1.9215 & 30.4614 \\
        \midrule
        \multirow{3}{*}{P1P2L}
        & FUSE-PnPL & \textbf{100\%} & \textbf{1.3508} & \textbf{67.8232} \\
        & DLSLU     & 97.6\% & 3.2147 & 119.4975 \\
        & OPnPL     & 99.6\% & 2.5794 & 210.2679 \\
        \bottomrule
    \end{tabular}
\end{table}

\subsection{Static real image tests} \label{static_real_image_tests}
In this subsection, we conduct static real image tests, which correspond to the applications of absolute robot pose estimation in a pre-reconstructed map. 
First, we compare PnP estimators with the ETH3D dataset~\cite{schops2017multi}. The dataset comprises $13$ scenes, each containing dozens of images. Figure~\ref{ETH3D_bar_plot} presents the mean estimation errors of the evaluated PnP estimators for each scene. The proposed \texttt{FUSE-PnP} estimator achieves the lowest errors in most scenes. Consistent with the simulation results, \texttt{MLPnP} and \texttt{CPnP} achieves accuracy comparable to that of \texttt{FUSE-PnP}, whereas the remaining estimators exhibit substantially larger errors.

Second, we compare PnL estimators with the VGG dataset~\cite{werner2002new}. 
Figure~\ref{VGG_bar_plot} presents the mean estimation errors of the evaluated PnL estimators across the tested scenes. The proposed \texttt{FUSE-PnL} estimator achieves the lowest errors in every scene, with errors up to nearly three orders of magnitude lower than those of competing methods. Its advantage is less pronounced in the Model House scene because two images contain few line correspondences, resulting in relatively large estimation errors for \texttt{FUSE-PnL}. 
To qualitatively evaluate pose estimation accuracy, we reproject the 3D lines onto the images using the estimated poses. The lines reprojected by our method consistently align well with the actual image lines, whereas competing methods occasionally exhibit large reprojection errors. Figure~\ref{VGG_line_reprojection} illustrates this comparison using four examples from two indoor and two outdoor scenes. Our \texttt{FUSE-PnL} estimator achieves close alignment between the reprojected lines and the corresponding 2D lines in all four examples. In contrast, obvious misalignments are observed for \texttt{Ro\_PnL} in the Corridor and Wadham-College scenes, \texttt{ASPnL} in the Model-House and University-Library scenes, \texttt{SRPnL} in the University-Library scene, and \texttt{AlgLS} in the Model-House, Corridor, and University-Library scenes. We attribute this performance advantage to our noise uncertainty propagation analysis and the resulting ML formulation of the PnL problem.

Finally, we compare the PnPL estimators on the VGG dataset, which provides both point and line correspondences. Figure~\ref{VGG_bar_plot_PnPL} reports the mean estimation errors across seven scenes. Unlike the other competing methods, \texttt{EPnPLU} and \texttt{DLSLU} require prior information about both the 2D measurement uncertainties and the average scene depth. In our experiments, we provide them with the ground-truth average scene depth the noise variance estimated by our method.
As shown in Figure~\ref{VGG_bar_plot_PnPL}, the proposed \texttt{FUSE-PnPL} estimator achieves the lowest rotation and translation errors in all scenes except Merton College III. The improvements are particularly pronounced in the Merton College I, Merton College II, and University Library scenes. Notably, despite \texttt{EPnPLU} and \texttt{DLSLU} being provided with the additional information required by their formulations, \texttt{FUSE-PnPL} still achieves higher estimation accuracy without relying on such prior information. 
These results further demonstrate the effectiveness of the proposed uncertainty-aware formulation across diverse real-world scenes.

\begin{table*}[t]
\centering
\caption{AirSLAM comparison on the EuRoC MAV dataset. ATE and translation RPE are in meters, while rotation RPE is in degrees. Boldface indicates lower errors, and parentheses show percentage changes relative to the original AirSLAM.}
\label{tab:euroc_all_points_airslam_paired}
\resizebox{0.88\textwidth}{!}{%
\begin{tabular}{cccc@{\hspace{10mm}}ccc}
\toprule
\multirow{2}{*}{Sequence} & \multicolumn{3}{c}{AirSLAM} & \multicolumn{3}{c}{AirSLAM(ours)} \\
\cmidrule(lr){2-4}\cmidrule(lr){5-7}
 & ATE & RPE ($\bf t$) & RPE ($\bf R$) & ATE & RPE ($\bf t$) & RPE ($\bf R$) \\
\midrule
MH\_01\_easy & \textbf{0.0879} & 0.0228 & 0.235 & 0.1002 (+13.99\%) & \textbf{0.0213 (-6.69\%)} & \textbf{0.224 (-4.53\%)} \\
MH\_02\_easy & 0.0941 & 0.0169 & 0.189 & \textbf{0.0908 (-3.49\%)} & \textbf{0.0159 (-5.75\%)} & \textbf{0.183 (-3.25\%)} \\
MH\_03\_medium & 0.1439 & 0.0326 & 0.253 & \textbf{0.1135 (-21.11\%)} & \textbf{0.0317 (-2.95\%)} & \textbf{0.235 (-7.06\%)} \\
MH\_04\_difficult & 0.3794 & 0.0404 & 0.232 & \textbf{0.2847 (-24.96\%)} & \textbf{0.0391 (-3.32\%)} & \textbf{0.227 (-2.09\%)} \\
MH\_05\_difficult & 0.2416 & 0.0343 & 0.224 & \textbf{0.2115 (-12.45\%)} & \textbf{0.0332 (-3.04\%)} & \textbf{0.221 (-1.10\%)} \\
V1\_01\_easy & 0.0877 & 0.0850 & 1.344 & \textbf{0.0711 (-18.94\%)} & \textbf{0.0843 (-0.84\%)} & \textbf{1.310 (-2.48\%)} \\
V1\_02\_medium & \textbf{0.1216} & \textbf{0.0459} & \textbf{0.569} & 0.1252 (+2.90\%) & 0.0468 (+2.02\%) & 0.589 (+3.43\%) \\
V1\_03\_difficult & \textbf{0.1562} & \textbf{0.0634} & 0.937 & 0.1969 (+26.01\%) & 0.0648 (+2.23\%) & \textbf{0.922 (-1.61\%)} \\
V2\_01\_easy & 0.0946 & 0.0311 & 0.904 & \textbf{0.0910 (-3.76\%)} & \textbf{0.0298 (-4.35\%)} & \textbf{0.861 (-4.82\%)} \\
V2\_02\_medium & \textbf{0.1577} & 0.0458 & \textbf{0.756} & 0.1608 (+1.99\%) & \textbf{0.0456 (-0.39\%)} & 0.757 (+0.08\%) \\
V2\_03\_difficult & 1.1954 & 0.1203 & \textbf{2.807} & \textbf{1.1509 (-3.72\%)} & \textbf{0.1144 (-4.89\%)} & 2.833 (+0.93\%) \\
\midrule
Mean & 0.2509 & 0.0490 & 0.768 & \textbf{0.2361 (-5.92\%)} & \textbf{0.0479 (-2.16\%)} & \textbf{0.760 (-1.04\%)} \\
\bottomrule
\end{tabular}%
}
\end{table*}

\begin{table*}[t]
\centering
\caption{PL-SLAM comparison on the EuRoC MAV dataset.}
\label{tab:euroc_all_points_pl_slam_paired}
\resizebox{0.88\textwidth}{!}{%
\begin{tabular}{cccc@{\hspace{10mm}}ccc}
\toprule
\multirow{2}{*}{Sequence} & \multicolumn{3}{c}{PL-SLAM} & \multicolumn{3}{c}{PL-SLAM(ours)} \\
\cmidrule(lr){2-4}\cmidrule(lr){5-7}
 & ATE & RPE ($\bf t$) & RPE ($\bf R$) & ATE & RPE ($\bf t$) & RPE ($\bf R$) \\
\midrule
MH\_01\_easy & \textbf{0.9424} & 0.0440 & \textbf{0.313} & 0.9439 (+0.16\%) & \textbf{0.0435 (-1.19\%)} & 0.314 (+0.33\%) \\
MH\_02\_easy & 0.7817 & 0.0339 & 0.360 & \textbf{0.7807 (-0.13\%)} & \textbf{0.0328 (-3.35\%)} & \textbf{0.345 (-4.35\%)} \\
MH\_03\_medium & \textbf{1.3044} & \textbf{0.0533} & \textbf{0.480} & 1.3058 (+0.11\%) & 0.0535 (+0.36\%) & 0.485 (+1.04\%) \\
MH\_04\_difficult & 2.2395 & 0.1665 & 0.911 & \textbf{2.2190 (-0.91\%)} & \textbf{0.1363 (-18.11\%)} & \textbf{0.802 (-12.00\%)} \\
MH\_05\_difficult & 1.8837 & 0.1525 & 0.890 & \textbf{1.8648 (-1.00\%)} & \textbf{0.1293 (-15.22\%)} & \textbf{0.792 (-10.99\%)} \\
V1\_01\_easy & \textbf{0.9082} & 0.1289 & 2.265 & 0.9084 (+0.02\%) & \textbf{0.1286 (-0.28\%)} & \textbf{2.265 (-0.02\%)} \\
V1\_02\_medium & \textbf{0.5066} & \textbf{0.0779} & \textbf{1.388} & 0.5079 (+0.25\%) & 0.0781 (+0.33\%) & 1.397 (+0.59\%) \\
V1\_03\_difficult & \textbf{0.6620} & 0.1416 & 2.569 & 0.7004 (+5.79\%) & \textbf{0.1308 (-7.60\%)} & \textbf{2.519 (-1.95\%)} \\
V2\_01\_easy & 1.0858 & 0.2301 & 9.534 & \textbf{0.8364 (-22.97\%)} & \textbf{0.1402 (-39.07\%)} & \textbf{4.044 (-57.58\%)} \\
V2\_02\_medium & \textbf{0.9599} & 0.1323 & 2.229 & 0.9765 (+1.73\%) & \textbf{0.1269 (-4.10\%)} & \textbf{2.186 (-1.90\%)} \\
V2\_03\_difficult & 2.7327 & 0.3284 & 6.595 & \textbf{2.6150 (-4.31\%)} & \textbf{0.2698 (-17.85\%)} & \textbf{5.503 (-16.56\%)} \\
\midrule
Mean & 1.2734 & 0.1354 & 2.503 & \textbf{1.2417 (-2.49\%)} & \textbf{0.1154 (-14.75\%)} & \textbf{1.877 (-25.00\%)} \\
\bottomrule
\end{tabular}%
}
\end{table*}

\begin{figure}[!htbp]
	\centering
	\includegraphics[width=0.48\textwidth]{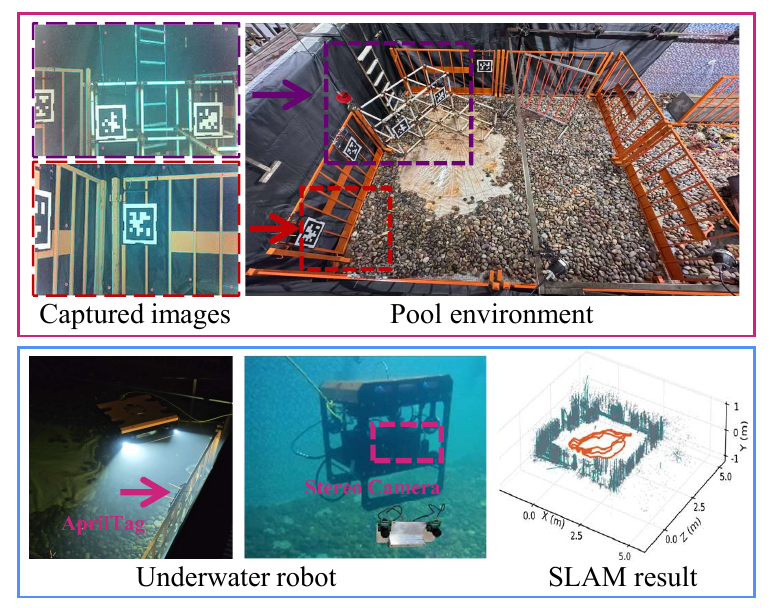}
	\caption{Experimental setup and SLAM result. Top: Pool environment and captured scene images. Bottom: Underwater robot, stereo camera, and a SLAM result example.}
	\label{experiment_setup}
\end{figure}

\subsection{Visual odometry with real image sequences} \label{visual_odometry}
In the preceding subsection, we evaluated camera pose estimation using a known 3D map, corresponding to robot localization in a mapped environment. In practice, however, robots often operate in unknown environments, requiring continuous pose estimation through SLAM or visual odometry. In this subsection, we integrate our proposed estimator into existing SLAM pipelines to evaluate its impact on odometry accuracy.
Specifically, we consider two SLAM systems, AirSLAM~\cite{xu2025airslam} and PL-SLAM~\cite{gomez2019pl}. In each system, we replace the original PnPL solver with our \texttt{FUSE-PnPL} estimator and adopt the line-residual weighting scheme derived from the ML formulation in~\eqref{ML_problem}. We then evaluate the resulting improvements in odometry accuracy using the following two datasets.

1) \textbf{EuRoC MAV}~\cite{euroc}: This dataset consists of 11 sequences recorded by a micro aerial vehicle in two indoor environments: an industrial machine hall and Vicon-equipped rooms. Accurate ground-truth trajectories are provided for quantitative evaluation.

2) \textbf{PoolStereo}: This is our self-collected underwater dataset acquired in an indoor laboratory pool. 
In underwater environments, light scattering, nonuniform illumination, and texture degradation often hinder reliable point-feature detection and matching. In contrast, line features, particularly those associated with prominent structural boundaries, can provide more stable geometric cues and facilitate robust detection and matching under challenging imaging conditions. Therefore, incorporating line features is particularly valuable for achieving robust and accurate underwater pose estimation. To emulate real-world subsea infrastructure inspection scenarios, as illustrated in Figure~\ref{experiment_setup}, we covered the pool walls with black fabric to suppress background visual features and installed several metal frames to mimic underwater infrastructure. 
Multiple AprilTag boards are mounted on the metal frames, while a stereo camera mounted on the robot captures image sequences during underwater navigation. Specifically, the robot's trajectory is carefully planned to ensure that successive AprilTag boards remain observable, enabling the acquisition of reference poses for trajectory evaluation.

\begin{table*}[t]
\centering
\caption{AirSLAM comparison on our self-collected PoolStereo dataset.}
\label{tab:poolstereo_k12_airslam_paired}
\resizebox{0.88\textwidth}{!}{%
\begin{tabular}{cccc@{\hspace{10mm}}ccc}
\toprule
\multirow{2}{*}{Sequence} & \multicolumn{3}{c}{AirSLAM} & \multicolumn{3}{c}{AirSLAM(ours)} \\
\cmidrule(lr){2-4}\cmidrule(lr){5-7}
 & ATE & RPE ($\bf t$) & RPE ($\bf R$) & ATE & RPE ($\bf t$) & RPE ($\bf R$) \\
\midrule
PS-E01 & 0.1033 & 0.1010 & 4.182 & \textbf{0.0991 (-4.05\%)} & \textbf{0.0999 (-1.05\%)} & \textbf{4.141 (-1.00\%)} \\
PS-E02 & 0.0660 & 0.0748 & 2.790 & \textbf{0.0627 (-4.97\%)} & \textbf{0.0734 (-1.85\%)} & \textbf{2.716 (-2.64\%)} \\
PS-E03 & 0.0671 & 0.0847 & 2.366 & \textbf{0.0559 (-16.69\%)} & \textbf{0.0797 (-5.92\%)} & \textbf{2.284 (-3.46\%)} \\
PS-M01 & 0.1221 & 0.0715 & 2.632 & \textbf{0.0718 (-41.13\%)} & \textbf{0.0714 (-0.14\%)} & \textbf{2.629 (-0.12\%)} \\
PS-M02 & 0.1097 & 0.1086 & 5.527 & \textbf{0.0862 (-21.38\%)} & \textbf{0.1047 (-3.62\%)} & \textbf{5.251 (-4.99\%)} \\
PS-H01 & 0.1426 & 0.0712 & 2.384 & \textbf{0.1344 (-5.72\%)} & \textbf{0.0688 (-3.41\%)} & \textbf{2.318 (-2.75\%)} \\
PS-H02 & 0.1654 & \textbf{0.0995} & \textbf{4.661} & \textbf{0.0836 (-49.42\%)} & 0.1007 (+1.17\%) & 4.736 (+1.60\%) \\
\midrule
Mean & 0.1109 & 0.0873 & 3.506 & \textbf{0.0848 (-23.48\%)} & \textbf{0.0855 (-2.09\%)} & \textbf{3.439 (-1.90\%)} \\
\bottomrule
\end{tabular}%
}
\end{table*}

\begin{table*}[t]
\centering
\caption{PL-SLAM comparison on our self-collected PoolStereo dataset.}
\label{tab:poolstereo_k12_pl_slam_paired}
\resizebox{0.88\textwidth}{!}{%
\begin{tabular}{cccc@{\hspace{10mm}}ccc}
\toprule
\multirow{2}{*}{Sequence} & \multicolumn{3}{c}{PL-SLAM} & \multicolumn{3}{c}{PL-SLAM(ours)} \\
\cmidrule(lr){2-4}\cmidrule(lr){5-7}
 & ATE & RPE ($\bf t$) & RPE ($\bf R$) & ATE & RPE ($\bf t$) & RPE ($\bf R$) \\
\midrule
PS-E01 & 0.2147 & 0.1078 & 4.480 & \textbf{0.1131 (-47.33\%)} & \textbf{0.1042 (-3.35\%)} & \textbf{4.297 (-4.07\%)} \\
PS-E02 & 0.2677 & 0.0724 & \textbf{2.451} & \textbf{0.1169 (-56.33\%)} & \textbf{0.0721 (-0.40\%)} & 2.533 (+3.36\%) \\
PS-E03 & 0.1574 & 0.0898 & 2.732 & \textbf{0.1436 (-8.77\%)} & \textbf{0.0833 (-7.24\%)} & \textbf{2.579 (-5.59\%)} \\
PS-M01 & 0.1761 & 0.0746 & 2.844 & \textbf{0.1457 (-17.22\%)} & \textbf{0.0733 (-1.80\%)} & \textbf{2.783 (-2.15\%)} \\
PS-M02 & 0.1939 & \textbf{0.1128} & \textbf{5.655} & \textbf{0.1629 (-16.00\%)} & 0.1132 (+0.37\%) & 5.712 (+1.01\%) \\
PS-H01 & 0.6559 & 0.0777 & \textbf{2.678} & \textbf{0.5624 (-14.27\%)} & \textbf{0.0775 (-0.21\%)} & 2.680 (+0.10\%) \\
PS-H02 & 0.5003 & \textbf{0.1110} & 5.349 & \textbf{0.4205 (-15.96\%)} & 0.1120 (+0.89\%) & \textbf{5.310 (-0.73\%)} \\
\midrule
Mean & 0.3094 & 0.0923 & 3.741 & \textbf{0.2379 (-23.13\%)} & \textbf{0.0908 (-1.63\%)} & \textbf{3.699 (-1.12\%)} \\
\bottomrule
\end{tabular}%
}
\end{table*}

The two datasets exhibit distinct feature distributions. EuRoC MAV contains abundant point features that substantially outnumber line features, whereas PoolStereo has relatively few point features, with line features outnumbering point features in many images. 
The quantitative results are reported in Tables~\ref{tab:euroc_all_points_airslam_paired}--\ref{tab:poolstereo_k12_pl_slam_paired}, with both ATE and RPE expressed as RMSEs. Overall, integrating our \texttt{FUSE-PnPL} estimator and line-residual weighting scheme improves odometry accuracy. The improvement in ATE is less pronounced on EuRoC MAV than on PoolStereo. This is primarily because point features dominate in EuRoC MAV, reducing the contribution of line features. Since the point-based sliding-window bundle adjustment (BA) remains unchanged, the differences between the initial pose estimates produced by different PnPL solvers are largely diminished during refinement.
In contrast, PoolStereo contains relatively few point features, with line features often being equally or more abundant. Consequently, the improved pose initialization provided by \texttt{FUSE-PnPL}, together with the proposed line-residual weighting scheme in sliding-window BA, significantly improves global trajectory accuracy, reducing the mean ATE of AirSLAM by $23.48$\%.
Notably, integrating our method into PL-SLAM substantially reduces RPE on EuRoC MAV, achieving a $25$\% reduction in mean rotation RPE. This suggests that the original PnPL estimator in PL-SLAM is less stable, whereas our approach provides more reliable pose estimation, further demonstrating its robustness.
For qualitative evaluation, we visualize the estimated trajectories for the first sequence of each dataset in Figure~\ref{odometry_result}. Compared with the original pipelines, integrating our method yields visually comparable trajectories on EuRoC MAV but noticeably improves trajectory accuracy on PoolStereo. In both examples, AirSLAM achieves better trajectory accuracy than PL-SLAM.
Regarding computational efficiency, we evaluated AirSLAM integrated with our method on the PS-E01 sequence of PoolStereo and the MH\_01\_easy sequence of EuRoC MAV. These sequences have durations of 159.5s and 184.05s, respectively, while the corresponding odometry processing times are only 68.85s and 158.19s. In both cases, the processing time is shorter than the sequence duration, demonstrating the computational efficiency of our method and its potential for real-time operation.

\begin{figure*}[!htbp]
	\centering
	\includegraphics[width=\textwidth]{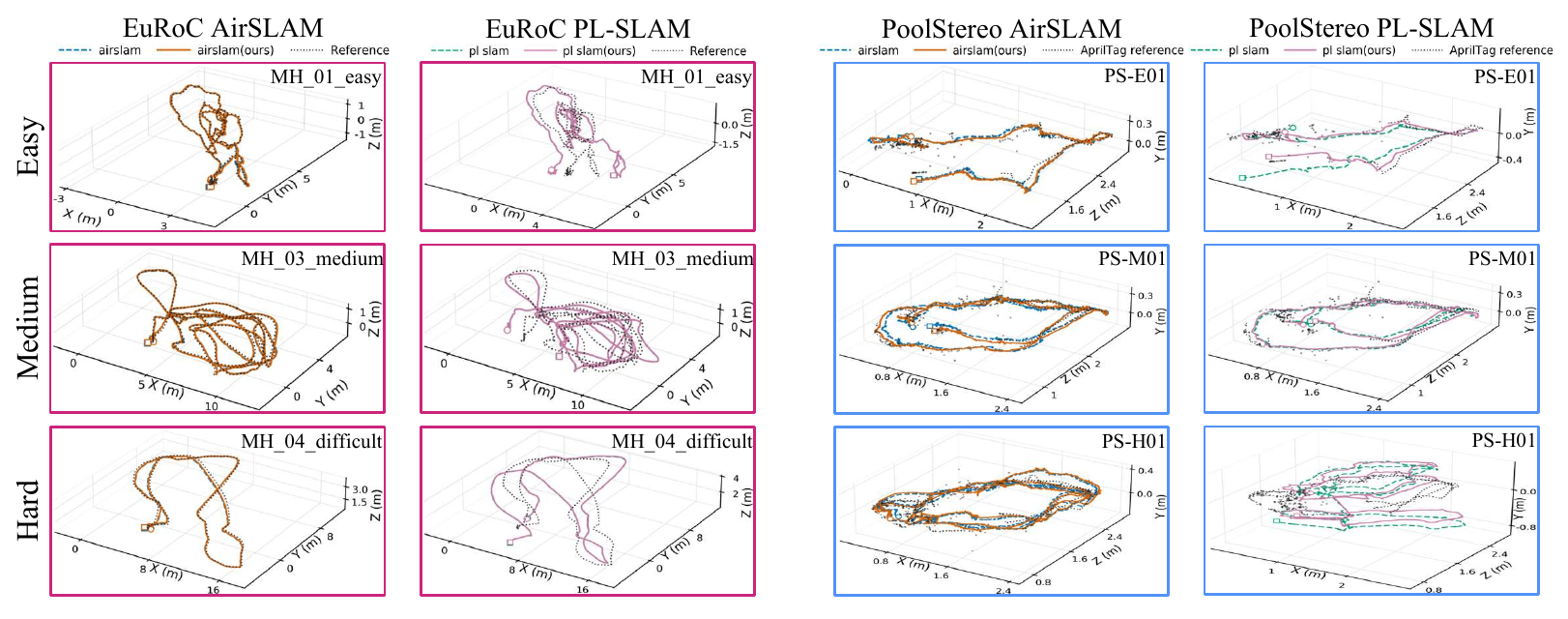}
	\caption{Odometry results on the EuRoC MAV and PoolStereo datasets. From left to right: trajectory comparisons for AirSLAM and PL-SLAM on the EuRoC MAV dataset and the PoolStereo dataset, respectively. From top to bottom: easy, medium, and difficult sequences. Black dotted lines denote reference trajectories.}
	\label{odometry_result}
\end{figure*}

\begin{table}[!t]
\centering
\caption{Ablation study of bias elimination in \texttt{FUSE-PnPL}. Relative changes in sequence-averaged RMSEs ($\bar E$) after disabling bias elimination: $\Delta=(\bar E_{\mathrm{off}}/\bar E_{\mathrm{on}}-1)\times100\%$.}
\label{tab:bias_elimination_relative}
\small
\setlength{\tabcolsep}{4pt}
\renewcommand{\arraystretch}{1.12}
\begin{tabular}{ccccc}
\toprule
Dataset & Pipeline & $\Delta$ATE & $\Delta$RPE ($\mathbf{t}$) & $\Delta$RPE ($\mathbf{R}$) \\
\midrule
\multirow{2}{*}{EuRoC}
 & AirSLAM & $+1.56$ & $+0.45$ & $+0.96$ \\
 & PL-SLAM & $-0.22$ & $-0.01$ & $-0.02$ \\
\midrule
\multirow{2}{*}{PoolStereo}
 & AirSLAM & $+8.33$ & $+0.31$ & $+0.24$ \\
 & PL-SLAM & $+0.54$ & $+0.01$ & $+0.13$ \\
\bottomrule
\end{tabular}
\end{table}

Finally, we conduct an ablation study to evaluate the effect of bias elimination in our \texttt{FUSE-PnPL} estimator. We report the relative changes in sequence-averaged RMSEs after disabling bias elimination, using the enabled variant as the reference. All other components, including noise estimation, residual weighting, GN refinement, and BA, remain unchanged.
The results are presented in Table~\ref{tab:bias_elimination_relative}. Disabling bias elimination increases most error metrics, demonstrating its positive contribution to estimation accuracy. In AirSLAM, the increase in ATE is more pronounced on PoolStereo than on EuRoC. This may be attributed to the higher measurement noise in underwater environments, which can induce greater algebraic bias, consistent with our simulation results. Specifically, the average estimated noise standard deviation is approximately $1.5$px on PoolStereo, compared with $0.8$px on EuRoC. In contrast, bias elimination has a negligible effect on the PL-SLAM pipeline. This may be because errors arising from feature association and landmark reconstruction dominate the overall estimation error in PL-SLAM, thereby limiting the benefit of bias correction.

\section{Conclusion} \label{conclusion}
In this paper, we proposed \texttt{FUSE-PnPL}, a unified camera pose estimator that accommodates arbitrary combinations of point and line correspondences, requiring as few as three features. The proposed estimator is statistically consistent and asymptotically efficient, with linear computational complexity in the number of correspondences. Its ML formulation enables accurate pose estimation across a wide range of feature counts, while the benefits of its minimal pose parameterization and bias-eliminated initialization are particularly evident in challenging feature configurations. Furthermore, the statistically principled weighting of line residuals substantially improves estimation accuracy over competing methods in line-dominant scenarios. Extensive experiments demonstrate the effectiveness of \texttt{FUSE-PnPL} in both map-based global localization and SLAM in unknown environments.

\appendices 

\section{Proof of Lemma~\ref{lemma:noise-free property}} \label{noise-free property proof}
First, since the ground truth naturally satisfies the underlying geometric constraints, we have $\|{\bf A}^o {\bm \varphi}^o+{\bf B}^o {\bm \tau}^o\|^2=0$ and $\bar {\bf g}^{o \top} {\bf G}^o \bar {\bf g}^o=0$, where ${\bm \varphi}^o$ and $\bar {\bf g}^o$ are constructed from ${\bf g}^o$. This implies $({\bf g}^o,{\bm \tau}^o)$ and ${\bf g}^o$ are global minimizers of~\eqref{linear_problem} and~\eqref{linear_problem3}, respectively. To prove the uniqueness, we suppose there exists a pair $({\bf g}',{\bm \tau}') \neq ({\bf g}^o,{\bm \tau}^o)$ such that $\|{\bf A}^o {\bm \varphi}'+{\bf B}^o {\bm \tau}'\|^2=0$.
    If ${\bf g}' \neq {\bf g}^o$, since the CGR representation is an injective mapping, the resulting rotation matrices hold that ${\bf R}' \neq {\bf R}^o$, which means there is another pose that conforms to all point and line correspondences, contradicting Assumption~\ref{assump:identifiability}. If ${\bf g}' = {\bf g}^o$ while ${\bm \tau}' \neq {\bm \tau}^o$, we have ${\bf t}' \neq {\bf t}^o$, which also contradicts Assumption~\ref{assump:identifiability}. Therefore, $({\bf g}^o,{\bm \tau}^o)$ is the unique global minimizer of~\eqref{linear_problem}. 
    Now we suppose there exists a ${\bf g}' \neq {\bf g}^o$ such that $\bar {\bf g}'^{\top} {\bf G}^o \bar {\bf g}'=0$. Let ${\bm \tau}'={\bf C}^o{\bm \varphi}'$. It can be verified that $\|{\bf A}^o {\bm \varphi}'+{\bf B}^o {\bm \tau}'\|^2=0$, which contradicts the established uniqueness of the global minimizer $({\bf g}^o,{\bm \tau}^o)$. Hence, ${\bf g}^o$ is the unique global minimizer of~\eqref{linear_problem3}, which completes the proof.

\section{Proof of Lemma~\ref{lemma:bias elimination}} \label{bias-elimination proof}
 The supporting lemma is 
      \begin{lemma}[{\cite[Lemma 4]{zeng2022global}}] \label{lemma_noise_aver}
	Let $\{X_j\}$ be a sequence of independent random variables with $\mathbb E[X_j]=0$ and $\mathbb E\left[X_j^2 \right]  \leq \varphi <\infty$ for all $j$. Then, there holds $\sum_{j=1}^{m}X_j/m=O_p(1/\sqrt{m})$.
\end{lemma}
Taking $ {\bf M}_{AA}$ as an example, for the $i$-th point correspondence, we have ${\bf A}_{pi}
={\bf A}_{pi}^o
+{\bm\epsilon}_{x_i}{\bf e}_3^\top\bar{\bf X}_i$.
Since $\mathbb{E}[{\bm\epsilon}_{x_i}]={\bf 0}$ and
$\mathbb{E}[{\bm\epsilon}_{x_i}^\top{\bm\epsilon}_{x_i}]
=2\sigma_p^2$, direct expansion gives
\begin{equation*}
\mathbb{E}\!\left[
{\bf A}_{pi}^\top{\bf A}_{pi}
-2\sigma_p^2\bar{\bf X}_i^\top{\bf e}_3{\bf e}_3^\top
\bar{\bf X}_i
\right]
=({\bf A}_{pi}^o)^\top{\bf A}_{pi}^o.
\end{equation*}

For the $j$-th line correspondence, a direct calculation using the endpoint-noise model yields
\begin{align*}
    &\mathbb{E}[\tilde{\bf l}_j\tilde{\bf l}_j^\top] \\
&=
\tilde{\bf l}_j^o\tilde{\bf l}_j^{o\top}
+
\begin{bmatrix}
2\sigma_l^2{\bf I}_2
&
-\sigma_l^2({\bf p}_j^o+{\bf q}_j^o)
\\
-\sigma_l^2({\bf p}_j^o+{\bf q}_j^o)^\top
&
\sigma_l^2(\|{\bf p}_j^o\|^2+\|{\bf q}_j^o\|^2)
+2\sigma_l^4
\end{bmatrix}.
\end{align*}
Moreover, $\mathbb{E}[{\bf p}_j+{\bf q}_j]
={\bf p}_j^o+{\bf q}_j^o$ and $\mathbb{E}[\|{\bf p}_j\|^2+\|{\bf q}_j\|^2]
=
\|{\bf p}_j^o\|^2+\|{\bf q}_j^o\|^2+4\sigma_l^2$.
It follows from the definition of ${\Delta {\bf L}}_{j}$ that
\[
\mathbb{E}\!\left[
\tilde{\bf l}_j\tilde{\bf l}_j^\top-{\Delta {\bf L}}_{j}
\right]
=
\tilde{\bf l}_j^o\tilde{\bf l}_j^{o\top},
\]
which gives 
\begin{equation*}
\mathbb{E}\!\left[
{\bf A}_{lj}^\top{\bf A}_{lj}
-(\bar{\bf P}_j^\top \Delta {\bf L}_j \bar{\bf P}_j+\bar{\bf Q}_j^\top \Delta {\bf L}_j \bar{\bf Q}_j)
\right]
=({\bf A}_{lj}^o)^\top{\bf A}_{lj}^o.
\end{equation*}

Applying Lemma~\ref{lemma_noise_aver} yields
\begin{equation*}
    {\bf M}_{AA}  =\frac{{\bf A}^{o \top} {\bf A}^o}{n+m} + O_p(1/\sqrt{n+m}).
\end{equation*}
The proofs for ${\bf M}_{AB}$ and ${\bf M}_{BB}$ follow analogously and are therefore omitted.

\section{Proof of Theorem~\ref{theorem:consistent_estimator}} \label{proof:consistent_estimator}
 The proof is straightforward from Lemma~\ref{lemma:bias elimination}. Since equation~\eqref{eq:G_BE_construction} is differentiable at the corresponding true parameter values, it preserves $\sqrt{n+m}$-consistency and we have $\mathbf{C}^{\mathrm{BE}}=\mathbf{C}^o+O_p(1/\sqrt{n+m})$ and $\mathbf{G}^{\mathrm{BE}}=\mathbf{G}^o+O_p(1/\sqrt{n+m})$. Recall from Lemma~\ref{lemma:noise-free property} that ${\bf g}^o$ is the global minimizer of~\eqref{linear_problem3}. Hence, we have $\hat {\bf g}^{\rm BE}={\bf g}^o+O_p(1/\sqrt{n+m})$. Again, the recovery of  $(\hat {\bf R}^{\rm BE},\hat {\bf t}^{\rm BE})$ from $\hat {\bf g}^{\rm BE}$ involves differentiable functions and we complete the proof.

\section{Proof of Theorem~\ref{theorem:line_consistent_noise_estimation}} \label{proof:line noise est}
   The first supporting lemma is Lemma~\ref{lemma_noise_aver}.
Let
\begin{equation*}
    \Delta {\bf D}\triangleq{\bf D}-{\bf D}^o=\begin{bmatrix}
        [{\bm \epsilon}_{{\bf p}_1}^h~{\bm \epsilon}_{{\bf q}_1}^h]^\top \overbar {\bf L}_1 \\
        \vdots \\
        [{\bm \epsilon}_{{\bf p}_m}^h~{\bm \epsilon}_{{\bf q}_m}^h]^\top \overbar {\bf L}_m
    \end{bmatrix}.
\end{equation*}
Then, we have 
\begin{equation} \label{eqn:H_bias_eli}
    \begin{split}
    {\bf H} -{\bf H}^o & = \frac{1}{m}( \Delta {\bf D}^\top {\bf D}^o+{\bf D}^{o \top}\Delta {\bf D}+\Delta {\bf D}^\top \Delta {\bf D}) \\
    & = \frac{1}{m}( \Delta {\bf D}^\top \Delta {\bf D}) + O_p(1/\sqrt{m}) \\
    & = \frac{\sigma^2_l}{m}( \tilde {\bf D}_{1}^\top \tilde {\bf D}_{1}+\tilde {\bf D}_{2}^\top \tilde {\bf D}_{2}) + O_p(1/\sqrt{m}) \\
    & = \hat \sigma^2_l (\tilde{\bf H}_{1}+\tilde{\bf H}_{2}) + O_p(1/\sqrt{m}),
    \end{split}
\end{equation}
where the second and third ``$=$'' are based on Lemma~\ref{lemma_noise_aver}, and the last ``$=$'' attributes to the boundedness of $\tilde{\bf H}_{1}$ and $\tilde{\bf H}_{2}$. 

The second supporting lemma is
    \begin{lemma}[{\cite[Lemma 6]{zeng2023consistent}}] \label{lemma_largest_eig}
	Let ${\bf R}$ and ${\bf S}$ be two real symmetric matrices and ${\bf Q}={\bf R}+{\bf S}$. If ${\bf Q}$ is positive-definite and ${\bf R}$ is positive-semidefinite with $0$ eigenvalues, then $\lambda_{\rm max}( {\bf Q}^{-1} {\bf S})=1$.
\end{lemma}
Recall from~\eqref{eqn:H_bias_eli} that
\begin{equation*}
    {\bf H}={\bf H}^o+\sigma^2_l (\tilde{\bf H}_1+\tilde{\bf H}_2) + O_p(1/\sqrt{m}).
\end{equation*}
This implies that $\lim_{m \rightarrow \infty} {\bf H}={\bf H}^o+\sigma^2_l (\tilde{\bf H}_1+\tilde{\bf H}_2)$. On the one hand, in the noise-free case, ${\bf D}^o {\bm \theta}^o=0$, which implies that ${\bf D}^o$ does not have full column rank. Hence, ${\bf H}^o$ is a positive-semidefinite matrix with $0$ eigenvalues. 
On the other hand, since ${\bf D}$ is corrupted by random noises, given $m \geq 9$ line correspondences, it has full column rank with probability one, i.e., ${\bf H}$ is almost surely positive-definite. Then, according to Lemma~\ref{lemma_largest_eig}, we have $\lim_{m \rightarrow \infty} 1/\lambda_{\rm max}( {\bf H}^{-1}(\tilde{\bf H}_1+\tilde{\bf H}_2))=\sigma^2_l$, i.e., $\hat \sigma^2_l$ converges to $\sigma^2_l$. Moreover, since ${\bf H}-\sigma^2_l (\tilde{\bf H}_1+\tilde{\bf H}_2)$ converges to ${\bf H}^o$ with a rate of $1/\sqrt{m}$, the convergence rate of $\hat \sigma^2_l$ is also $1/\sqrt{m}$, which completes the proof.

\bibliographystyle{IEEEtran}
\bibliography{sj_reference}

\end{document}